\documentclass{article}

\usepackage{microtype}
\usepackage{graphicx}
\usepackage{subfigure}
\usepackage{booktabs}

\usepackage{hyperref}

\usepackage[accepted]{icml2019}

\icmltitlerunning{Nearest Neighbor and Kernel Survival Analysis}

\usepackage{amsmath}
\usepackage{amsthm}
\usepackage{amssymb}
\usepackage{dsfont}
\usepackage{bm}
\usepackage{enumitem}

\usepackage[symbol]{footmisc}

\newtheorem{thm}{Theorem}[section]
\newtheorem{lem}{Lemma}[section]
\newtheorem{cor}{Corollary}[section]
\newtheorem{prop}{Proposition}[section]
\newtheorem{example}{Example}[section]
\newcommand{\ind}{\mathds{1}}

\newcommand{\neighborsKNN}{\mathcal{N}_{k\textsc{-NN}}}
\newcommand{\neighborsNNh}{\mathcal{N}_{\textsc{NN}(h)}}

\newcommand{\setY}{\mathcal{Y}}
\newcommand{\setS}{\mathcal{I}}

\newcommand{\weightVar}{w}
\newcommand{\holderIndex}{\alpha}

\newcommand{\realNumbers}{\mathbb{R}}

\newcommand{\obsVar}{Y}
\newcommand{\survVar}{T}
\newcommand{\censVar}{C}
\newcommand{\eventVar}{\delta}
\newcommand{\probError}{\gamma}
\newcommand{\lips}{\lambda}
\newcommand{\Lips}{\Lambda}
\newcommand{\timeVar}{t}
\newcommand{\timeHorizon}{\tau}
\newcommand{\stanDistThresh}{\phi}

\newcommand{\survSubscript}{\textsc{\tiny{\survVar}}}
\newcommand{\censSubscript}{\textsc{\tiny{\censVar}}}
\newcommand{\obsSubscript}{\textsc{\tiny{\obsVar}}}

\newcommand{\obsEnd}{S_{\obsSubscript}}
\newcommand{\estObsEnd}{\widehat{S}_{\obsSubscript}}
\newcommand{\survEnd}{S}
\newcommand{\censEnd}{S_{\censSubscript}}
\newcommand{\estCensEnd}{\widehat{S}_{\censSubscript}}

\newcommand{\survDensity}{f_{\survSubscript}}
\newcommand{\censDensity}{f_{\censSubscript}}
\newcommand{\badEvent}[2]{\mathcal{E}_{{#2}}^{{#1}}}

\newcommand{\hazard}{h}
\newcommand{\Hazard}{H}

\newcommand{\featureSpace}{\mathcal{X}}

\newcommand{\featureVar}{X}
\newcommand{\fixedFeatureVector}{x}
\newcommand{\featureDist}{\mathbb{P}_{\featureVar}}

\begin{document}

\twocolumn[
\icmltitle{Nearest Neighbor and Kernel Survival Analysis: Nonasymptotic Error Bounds and Strong Consistency Rates}

\icmlsetsymbol{equal}{*}

\begin{icmlauthorlist}
\icmlauthor{George H.~Chen}{cmu}
\end{icmlauthorlist}

\icmlaffiliation{cmu}{Heinz College of Information Systems and Public Policy, Carnegie Mellon University, Pittsburgh, PA, USA}

\icmlcorrespondingauthor{George H.~Chen}{georgechen@cmu.edu}

\icmlkeywords{survival analysis, nearest neighbors, kernel methods, nonparametric methods}

\vskip 0.3in
]

\printAffiliationsAndNotice{}

\begin{abstract}
We establish the first nonasymptotic error bounds for Kaplan-Meier-based nearest neighbor and kernel survival probability estimators where feature vectors reside in metric spaces. Our bounds imply rates of strong consistency for these nonparametric estimators and, up to a log factor, match an existing lower bound for conditional CDF estimation. Our proof strategy also yields nonasymptotic guarantees for nearest neighbor and kernel variants of the Nelson-Aalen cumulative hazards estimator.
We experimentally compare these methods on four datasets.
We find that for the kernel survival estimator, a good choice of kernel is one learned using random survival forests.
\end{abstract}

\vspace{-1.7em}
\section{Introduction}
\vspace{-0.3em}

Survival analysis arises in numerous applications where we want to reason about the amount of time until some critical event happens. For example, in health care, we may be interested in using electronic health records to predict how long a patient with a particular disease will live (e.g., \citealt{botsis2010secondary,ganssauge2016exploring}), or how much time a patient has before a disease relapses (e.g., \citealt{zupan2000machine}). In criminology, we may be interested in predicting the time until a convicted criminal reoffends \citep{chung_1991}.

A fundamental task in survival analysis is estimating the survival probability over time for a specific subject (for ease of exposition, we stick to using standard survival analysis terminology in which the critical event of interest is death). Formally, suppose a subject has feature vector~$\featureVar$ (a random variable that takes on values in a feature space~$\featureSpace$) and survival time $\survVar$ (a nonnegative real-valued random variable). For a given feature vector $\fixedFeatureVector\in\featureSpace$, our goal is to estimate the conditional survival function $\survEnd(\timeVar | \fixedFeatureVector) := \mathbb{P}(\survVar > \timeVar | \featureVar=\fixedFeatureVector)$ for time~$\timeVar\ge0$.

To estimate $\survEnd$, we assume that we have access to $n$ training subjects. For the $i$-th subject, we have the subject's feature vector $\featureVar_i \in \featureSpace$ as well as two observations: $\eventVar_i\in\{0,1\}$ indicates whether we observe the survival time for the $i$-th subject, and $\obsVar_i\in\realNumbers_+$ is the survival time for the $i$-th subject if $\eventVar_i=1$ or the ``censoring time'' if $\eventVar_i=0$. The censoring time gives a lower bound for the $i$-th subject's survival time (e.g., when we stop collecting training data, the $i$-th subject might still be alive, in which case that is when the subject's true survival time is ``censored'' and we only know that the subject survives beyond the time of censoring).

Many approaches have been devised for estimating the conditional survival function $\survEnd$. Most standard approaches impose strong structural assumptions on $\survEnd$ via constraining the hazard function $\hazard(\timeVar|\fixedFeatureVector) := -\frac{\partial}{\partial \timeVar}\log \survEnd(\timeVar|\fixedFeatureVector)$. For example, the Cox proportional hazards model decouples the effects of time $\timeVar\ge0$ and of feature vector $\fixedFeatureVector\in\realNumbers^d$ by assuming the factorization $\hazard(\timeVar|\fixedFeatureVector) = \hazard_0(\timeVar)\exp(\beta^\top \fixedFeatureVector)$, where positive-valued function $\hazard_0$ and vector $\beta\in\realNumbers^d$ are parameters \citep{cox_1972}. After estimating $\hazard_0$ and $\beta$ from training data, then for any feature vector $\fixedFeatureVector$, we can estimate the hazard function~$\hazard(\timeVar|\fixedFeatureVector)$ by plugging in estimates for~$\hazard_0$ and~$\beta$. Integrating the estimate for $\hazard(\timeVar|\fixedFeatureVector)$ thus yields an estimate for $\survEnd(\timeVar|\fixedFeatureVector)=\exp(-\int_0^\timeVar\hazard(s|\fixedFeatureVector)ds)$. Other standard approaches such as the Aalen additive model \citep{aalen_1989} and accelerated failure time models \citep[Chapter~7]{kalbfleisch_2002} also impose structure on hazard function $\hazard(\timeVar|\fixedFeatureVector)$ and are typically used with parametric assumptions. More recent approaches include, for instance, modifying the Cox proportional hazards model by replacing the inner product $\beta^\top x$ with a nonlinear function of $x$ that is encoded as a deep net \citep{katzman2018deepsurv}, or completely specifying $S$ via a hierarchical generative model \citep{ranganath2016deep}.

Rather than making structural assumptions on~$\survEnd$, \citet{beran_1981} takes a nonparametric approach using nearest neighbors and kernels. The idea is simple: there already is a nonparametric estimator for the marginal survival function $\survEnd_{\text{marg}}(\timeVar):=\mathbb{P}(\survVar>\timeVar)$ known as the Kaplan-Meier estimator \citep{kaplan_meier_1958}. This estimator does not use feature vectors. We can incorporate feature vectors in a straightforward manner. For a test subject with feature vector~$\fixedFeatureVector$, we first find training subjects whose feature vectors are sufficiently close to~$\fixedFeatureVector$ (e.g., pick the $k$ closest). We apply the Kaplan-Meier estimator to just these nearby subjects to estimate the conditional survival probability function $\survEnd(\timeVar|\fixedFeatureVector)$ (the kernel variant can weight training subjects differently). \citet{beran_1981} provided consistency results for these $k$-NN and kernel estimates for~$\survEnd$, while \citet{dabrowska_1989}, \citet{van_1996}, and \citet{van_1998} established nonasymptotic error bounds for the kernel variant when feature vectors are Euclidean.

In this paper, we present the first nonasymptotic error bounds for nearest neighbor and kernel estimators for $\survEnd$ where feature vectors reside in the general setting of separable metric spaces (Euclidean space is a special case). Our error bounds lead to rates of strong consistency for both estimators across a wide range of distributions. Furthermore, our bounds are essentially optimal with respect to the number of training data~$n$. In particular, note that $1-\survEnd(\cdot|\fixedFeatureVector)$ is a conditional CDF. If there is no right-censoring, the problem reduces to conditional CDF estimation. Up to a log factor, our error rates match an existing conditional CDF estimation error lower bound by \citet{chagny_2014}.

Our proof strategy also yields nonasymptotic error bounds for Nelson-Aalen-based nearest neighbor and kernel estimates of the conditional cumulative hazard function $-\log\survEnd(\timeVar|\fixedFeatureVector)$. These bounds turn out to be crucial in how we derive generalization guarantees for automatic parameter selection (choosing the number of nearest neighbors or the kernel bandwidth) via a validation set.

Despite our theory handling a wide range of distances and kernels, both of these still have to be pre-specified by the user and, in practice, can lead to large prediction accuracy differences. As a simple heuristic, we propose using random survival forests \citep{ishwaran_2008} to learn a kernel for the kernel survival estimator. We experimentally show that the resulting \textit{adaptive} kernel estimator has prediction accuracy on par with regular random survival forests and is, in particular, typically as good as or better than other methods tested.

\vspace{-.7em}
\section{Model and Nonparametric Estimators}
\label{sec:problem-setup}
\vspace{-.3em}

\begingroup
\setlength\abovedisplayskip{2pt plus 1pt}
\setlength\belowdisplayskip{2pt plus 1pt}

\textbf{Model.}
The training data $(\featureVar_1,\obsVar_1,\eventVar_1),\dots,(\featureVar_n,\obsVar_n,\eventVar_n)$ $\in\featureSpace\times\realNumbers_+\times\{0,1\}$ are assumed to be generated i.i.d.~by the following process, stated for a generic data point $(\featureVar,\obsVar,\eventVar)$:
\begin{enumerate}[leftmargin=1.5em,topsep=-2pt,itemsep=-1ex,partopsep=1ex,parsep=1ex]
\item Sample feature vector $\featureVar\sim\featureDist$.
\item Sample nonnegative survival time $\survVar\sim\mathbb{P}_{\survVar|\featureVar}$.
\item Sample nonnegative censoring time $\censVar\sim\mathbb{P}_{\censVar|\featureVar}$. (Note that $\survVar$ and $\censVar$ are independent given $\featureVar$.)
\item Set $\obsVar=\min\{\survVar,\censVar\}$, and $\eventVar=\ind\{\survVar\le\censVar\}$.
\end{enumerate}
We refer to $\obsVar$ as the \emph{observed time}, and $\eventVar$ as the \emph{censoring indicator} (0 means censoring happened). For test feature vector $\fixedFeatureVector\in\featureSpace$, we aim to estimate the conditional survival function $\survEnd(\timeVar|\fixedFeatureVector) = {\mathbb{P}(\survVar>\timeVar|\featureVar=\fixedFeatureVector)}$ using the training data.

\textbf{Nonparametric survival function estimators.}
All nonparametric estimators for $\survEnd$ in this paper 
are based on the Kaplan-Meier estimator \citep{kaplan_meier_1958}, restricted to a subset of the $n$ training subjects. This estimator works as follows. Let $[n]:=\{1,2,\dots,n\}$ denote the set of all training subjects. For any subset of training subjects $\setS\subseteq[n]$, the Kaplan-Meier estimator first identifies the unique times when death occurred, given by the set $\setY_{\setS}:=\{\obsVar_j : j\in\setS\text{~s.t.~}\eventVar_j=1\}$ (repeated observed times get counted once). Next, we keep track of how many deaths and how many subjects are at risk at any given time~$t\ge0$:
\[
d_{\setS}(\timeVar) := \sum_{j\in\setS}\eventVar_j\ind\{\obsVar_j = \timeVar\},
\quad
n_{\setS}(\timeVar) := \sum_{j\in\setS}\ind\{\obsVar_j \ge \timeVar\}.
\]
Then the Kaplan-Meier estimator restricted to training subjects~$\setS$ is given by
\begin{equation*}
\widehat{\survEnd}^{\text{KM}}(\timeVar|\setS)
:=
\prod_{\timeVar'\in \setY_{\setS}} \Big(1 - \frac{d_{\setS}(\timeVar')}{n_{\setS}(\timeVar')}\Big)^{\ind\{\timeVar' \le \timeVar\}}.
\end{equation*}
This equation has a simple interpretation: if we sort the unique death times $\setY_{\setS}$ as $\timeVar_1 < \timeVar_2 < \cdots < \timeVar_{|\setY_{\setS}|}$, then the terms being multiplied above are estimated probabilities of a subject surviving from time~0 to $\timeVar_1$, from $\timeVar_1$ to $\timeVar_2$, and so forth until reaching time $\timeVar$. The standard Kaplan-Meier estimator has $\setS=[n]$.

We now state four nonparametric estimators for the conditional survival function $\survEnd$. The first two are by \citet{beran_1981} and are the estimators that we provide theoretical analysis for in the next section. Distances between feature vectors are measured via a user-specified metric $\rho:\featureSpace\times\featureSpace\rightarrow\realNumbers_+$.

\emph{$k$-NN survival estimator.} For a test feature vector $\fixedFeatureVector\in\featureSpace$, we first find the~$k$ training subjects with feature vectors closest to~$\fixedFeatureVector$ according to metric $\rho$, breaking ties uniformly at random.
Let $\neighborsKNN(\fixedFeatureVector)\subseteq[n]$ denote these $k$ subjects' indices. Then the $k$-NN estimate for~$\survEnd$ is $\widehat{\survEnd}^{k\textsc{-NN}}(\timeVar|\fixedFeatureVector) := \widehat{\survEnd}^{\text{KM}}(\timeVar|\neighborsKNN(\fixedFeatureVector))$.

\emph{Kernel survival estimator.} For a user-specified kernel function $K:\realNumbers_+\rightarrow\realNumbers_+$ and bandwidth $h>0$, we can measure how similar training subject $j\in[n]$ is to~$\fixedFeatureVector$ by the weight $K(\frac{\rho(\fixedFeatureVector, \featureVar_j)}h)$. We generalize the unique death times, death counts, and survivor counts as follows:
\begin{align*}
\setY_K(\fixedFeatureVector;h) \!&:=\! \Big\{ \obsVar_j \text{~for~}j\in\mathcal[n]\text{~s.t.~}\eventVar_j K\Big(\frac{\rho(\fixedFeatureVector, \featureVar_j)}{h}\Big)\!>\!0 \Big\}, \\
d_K(\timeVar|\fixedFeatureVector;h) \!&:=\! \sum_{j=1}^n K\Big(\frac{\rho(\fixedFeatureVector, \featureVar_j)}{h}\Big) \eventVar_j\ind\{\obsVar_j = \timeVar\}, \\
n_K(\timeVar|\fixedFeatureVector;h) \!&:=\! \sum_{j=1}^n K\Big(\frac{\rho(\fixedFeatureVector, \featureVar_j)}{h}\Big) \ind\{\obsVar_j \ge \timeVar\}.
\end{align*}
Then the kernel estimate for $\survEnd$ is given by
\begin{equation}
\widehat{\survEnd}^K(\timeVar|\fixedFeatureVector;h)
:= \prod_{\timeVar'\in \setY_K(\fixedFeatureVector;h)} \Big(1 - \frac{d_K(\timeVar'|\fixedFeatureVector;h)}{n_K(\timeVar'|\fixedFeatureVector;h)}\Big)^{\ind\{\timeVar' \le \timeVar\}}.
\label{eq:kernel-survival}
\end{equation}
In our numerical experiments later, we benchmark the above methods against the random survival forests method by \citet{ishwaran_2008} along with our proposed variant of it that combines it with the kernel survival estimator.

\emph{Random survival forests.} Random survival forests are much like standard random forests. During training, each tree is grown using a survival-analysis-based splitting rule. Each leaf is associated with some subset of the training data for which a Kaplan-Meier survival estimate is produced. In other words, for each tree, each leaf is associated with a particular survival function estimate. Then, for a test point~$\fixedFeatureVector$, we find the tree leaves that $\fixedFeatureVector$ belongs to. We average these leaves' survival function estimates to produce the final random survival forest estimate for $\survEnd(\cdot|\fixedFeatureVector)$.

\emph{Adaptive kernel survival estimator.} We propose an alternative approach to making predictions using random survival forests without changing their training procedure. For a test point $\fixedFeatureVector$, to make a final prediction, we instead use the kernel survival estimator given by equation \eqref{eq:kernel-survival}, where we replace the expression $K(\frac{\rho(x, X_j)}h)$ by $\widehat{K}(x, X_j)$, defined as the fraction of trees for which $x$ and training point $X_j$ show up in the same leaf node in the learned forest. Note that interpreting standard random forests as learning kernels was already done by \citet{breiman2000some}.

\textbf{Relating to the Nelson-Aalen estimator.}
The Nelson-Aalen estimator estimates the marginal cumulative hazard function $\Hazard_{\text{marg}}(\timeVar) := -\log\survEnd_{\text{marg}}(\timeVar)={-\log\mathbb{P}(\survVar>\timeVar)}$ \citep{nelson_1969,aalen_1978}. The Nelson-Aalen estimator restricted to training subjects~$\setS$ is given by
\[
\widehat{\Hazard}^{\text{NA}}(\timeVar|\setS)
:=
\sum_{\timeVar'\in \setY_{\setS}}  \frac{d_{\setS}(\timeVar')}{n_{\setS}(\timeVar')}{\ind\{\timeVar' \le \timeVar\}},
\]
using the same variables introduced for the Kaplan-Meier estimator. We can relate the Nelson-Aalen estimator to the Kaplan-Meier one: the first-order Taylor approximation of $-\log\widehat{\survEnd}^{\text{KM}}(\timeVar|\setS)$ is $\widehat{\Hazard}^{\text{NA}}(\timeVar|\setS)$. Because our theoretical analysis of $k$-NN and kernel variants of the Kaplan-Meier survival estimator is in terms of Taylor series expansions of $\log\survEnd$, our proofs extend (with small changes) to $k$-NN and kernel variants of the Nelson-Aalen estimator.

For clarity of exposition, the rest of the paper uses $k$-NN and kernel estimators to refer to the Kaplan-Meier versions rather than the Nelson-Aalen ones unless stated otherwise.

\vspace{-.75em}
\section{Theoretical Guarantees}
\vspace{-.25em}

We first introduce some notation. We denote closed and open balls centered at $\fixedFeatureVector\in\featureSpace$ with radius $r>0$ as
\begin{align*}
\mathcal{B}_{\fixedFeatureVector,r} & :=\{\fixedFeatureVector'\in\mathcal{X}\,:\,\rho(\fixedFeatureVector,\fixedFeatureVector')\le r\},\\
\mathcal{B}_{\fixedFeatureVector,r}^o & :=\{\fixedFeatureVector'\in\mathcal{X}\,:\,\rho(\fixedFeatureVector,\fixedFeatureVector')<r\}.
\end{align*}
For example, $\featureDist(\mathcal{B}_{\fixedFeatureVector,r})$ is the probability that a feature vector sampled from distribution $\featureDist$ lands in $\mathcal{B}_{\fixedFeatureVector,r}$. We define the ``support'' of feature distribution $\featureDist$ as
\[
\text{supp}(\featureDist):=\{\fixedFeatureVector\in\featureSpace\,:\,\featureDist(\mathcal{B}_{\fixedFeatureVector,r})>0\text{ for all }r>0\}.
\]
We denote tail probability functions using ``$\survEnd$'' with and without subscripts. $\survEnd$ without a subscript always refers to the tail of the conditional survival time $\survVar$ distribution $\survEnd(\timeVar|\fixedFeatureVector)={\mathbb{P}(\survVar>\timeVar|\featureVar=\fixedFeatureVector)}$. The tails of the conditional censoring time $\censVar$ and observed time $\obsVar$ distributions are $\censEnd(\timeVar|\fixedFeatureVector):={\mathbb{P}(C>\timeVar|\featureVar=\fixedFeatureVector)}$ and $\obsEnd(\timeVar|\fixedFeatureVector):={\mathbb{P}(\obsVar>\timeVar|\featureVar=\fixedFeatureVector)}$. PDF's of distributions $\mathbb{P}_{\survVar|\featureVar=\fixedFeatureVector}$ and $\mathbb{P}_{\censVar|\featureVar=\fixedFeatureVector}$ are denoted by $\survDensity(\timeVar|\fixedFeatureVector)$ and $\censDensity(\timeVar|\fixedFeatureVector)$. Note that $\obsEnd(\timeVar|\fixedFeatureVector)={\survEnd(\timeVar|\fixedFeatureVector)\censEnd(\timeVar|\fixedFeatureVector)}$, ${\survEnd(\timeVar|\fixedFeatureVector) = 1 - \int_0^\timeVar \survDensity(s|\fixedFeatureVector)ds}$, and ${\censEnd(\timeVar|\fixedFeatureVector) = 1 - \int_0^\timeVar \censDensity(s|\fixedFeatureVector)ds}$.

Our guarantees depend on the following four assumptions:
\begin{itemize}[topsep=-3pt,itemsep=0ex,topsep=-2pt,partopsep=0ex,parsep=0ex]

\item [\textbf{A1.}]
\emph{Feature space $\featureSpace$ and distance $\rho$ form a separable metric space, and feature distribution $\featureDist$ is a Borel probability measure.}
This assumption is technical and ensures that the probability of a feature vector landing in a ball (whether open or closed) is well-defined, and that we only need to care about feature vectors that land in $\text{supp}(\featureDist)$ (the probability of a feature vector landing outside of this support is 0). This assumption is also used in establishing consistency of nearest neighbor classification in metric spaces \citep{cerou_guyader_2006,chaudhuri_dasgupta_2014}.

\item [\textbf{A2.}]
\emph{For all $\fixedFeatureVector\in\text{supp}(\featureDist)$, distributions $\mathbb{P}_{\survVar|\featureVar=\fixedFeatureVector}$ and $\mathbb{P}_{\censVar|\featureVar=\fixedFeatureVector}$ exist and correspond to continuous random variables. Moreover, conditioned on $\featureVar=\fixedFeatureVector$, the indicator random variable $\eventVar={\ind\{\survVar\le\censVar\}}$ cannot almost surely be 0.}
This assumption ensures that functions $\survEnd$, $\censEnd$, $\obsEnd$, $f$, and $g$ defined above are well-defined, ties in observed times $\obsVar_i$'s happen with probability 0, and censoring does not almost surely happen.

\item [\textbf{A3.}]
\emph{There exists $\theta\in(0,\frac{1}{2}]$ and $\timeHorizon\in(0,\infty)$ such that
\[
\obsEnd(\timeHorizon|\fixedFeatureVector)\ge\theta\quad\text{for all }\fixedFeatureVector\in\text{supp}(\featureDist).
\]}In practice, we cannot estimate conditional survival function $\survEnd(\timeVar|\fixedFeatureVector)$ accurately for time $\timeVar$ that is arbitrarily large (e.g., $\timeVar>\max_{i=1,\dots,n}\obsVar_i$). We shall only guarantee accurate estimation of $\survEnd(\timeVar|\fixedFeatureVector)$ for $\timeVar\in[0,\timeHorizon]$.

\item [\textbf{A4.}]
\emph{For any time $\timeVar\in[0,\timeHorizon]$, density function $\survDensity(\timeVar|\fixedFeatureVector)$ and $\censDensity(\timeVar|\fixedFeatureVector)$ are H\"{o}lder continuous in~$\fixedFeatureVector$ with the same exponent~$\holderIndex>0$ but with potentially different constants $\lips_{\textsc{\tiny{\survVar}}}>0$ and $\lips_{\textsc{\tiny{\censVar}}}>0$, i.e., for all $\fixedFeatureVector,\fixedFeatureVector'\in\text{supp}(\featureDist)$,
\begin{align*}
|\survDensity(\timeVar|\fixedFeatureVector)-\survDensity(\timeVar|\fixedFeatureVector')| & \le\lips_{\textsc{\tiny{\survVar}}}\rho(\fixedFeatureVector,\fixedFeatureVector')^{\holderIndex},\\
|\censDensity(\timeVar|\fixedFeatureVector)-\censDensity(\timeVar|\fixedFeatureVector')| & \le\lips_{\textsc{\tiny{\censVar}}}\rho(\fixedFeatureVector,\fixedFeatureVector')^{\holderIndex}.
\end{align*}}
In other words, nearby feature vectors have similar conditional survival and censoring distributions. Thus, feature vectors near $\fixedFeatureVector$ can help us estimate~$\survEnd(\cdot|\fixedFeatureVector)$.

\end{itemize}
Many distributions $\featureDist$, $\survDensity$, and $\censDensity$ satisfy the assumptions above. We provide a few examples at the end of this section.

\endgroup
\begingroup
\setlength\abovedisplayskip{3pt plus 1pt}
\setlength\belowdisplayskip{2pt plus 1pt}
Since $\survDensity(\timeVar|\cdot)$ and $\censDensity(\timeVar|\cdot)$ are H\"{o}lder continuous with common exponent~$\holderIndex$, then so are $\obsEnd(\timeVar|\cdot)$ and $\censEnd(\timeVar|\cdot)\survDensity(\timeVar|\cdot)$, which appear in our analysis. With a bit of algebra, one can show that $\obsEnd(\timeVar|\cdot)$ is H\"{o}lder continuous with parameters ${(\lips_{\textsc{\tiny{\survVar}}}+\lips_{\textsc{\tiny{\censVar}}})t}$ and~$\holderIndex$. Meanwhile, $\censEnd(\timeVar|\cdot)\survDensity(\timeVar|\cdot)$ is H\"{o}lder continuous with parameters $(\lips_{\survSubscript}+\survDensity^*\lips_{\censSubscript}\timeVar)$ and~$\holderIndex$, where
\[
\survDensity^*:=\sup_{\timeVar\in[0,\timeHorizon],\fixedFeatureVector\in\text{supp}(\featureDist)}\survDensity(\timeVar|\fixedFeatureVector).
\]
Our $k$-NN result depends on the constant
\begin{equation*}
\Lips :=\max\Big\{
\frac{2\timeHorizon}\theta(\lips_{\textsc{\tiny{\survVar}}}+\lips_{\textsc{\tiny{\censVar}}}),\,
\lips_{\survSubscript}\timeHorizon + \frac{\survDensity^*\lips_{\censSubscript}\timeHorizon^2}{2}
\Big\}.
\end{equation*}
As we explain shortly, the $k$-NN survival estimator is closely related to two subproblems: $k$-NN CDF estimation and a special case of $k$-NN regression. In the definition of $\Lips$ above, the two parts of the maximization correspond precisely to the CDF estimation and regression components.
\endgroup

We state each of our main theoretical guarantees as a pointwise result, i.e., for any point $\fixedFeatureVector\in\text{supp}(\featureDist)$ and error tolerance $\varepsilon\in(0,1)$, how to guarantee $\sup_{\timeVar\in[0,\timeHorizon]}|\widehat{\survEnd}(\timeVar|\fixedFeatureVector) - \survEnd(\timeVar|\fixedFeatureVector)|\le\varepsilon$ with high probability using estimator $\widehat{\survEnd}$. Translating pointwise guarantees to account for randomness in sampling $\featureVar=\fixedFeatureVector$ from $\featureDist$ can easily be done using standard proof techniques, as we discuss momentarily.

\vspace{-.5em}
\subsection*{$\bm{k}$-NN estimator results}
\vspace{-.25em}

We begin with the nonasymptotic $k$-NN estimator guarantee. Proofs are deferred to the appendix. As a disclaimer, no serious attempt has been made to optimize constants.
\begin{thm}[$k$-NN pointwise bound]
\label{thm:kNN-survival}
Under Assumptions A1--A4, let $\varepsilon\in(0,1)$ be a user-specified error tolerance and define critical distance ${h^*:=(\frac{\varepsilon\theta}{18\Lips})^{1/\holderIndex}}$. For any feature vector $\fixedFeatureVector\in\text{supp}(\featureDist)$ and any choice of number of nearest neighbors $k\in[\frac{72}{\varepsilon\theta^{2}},\!\frac{n\featureDist(\mathcal{B}_{\fixedFeatureVector,h^*})}{2}]$, we have, over randomness in training data,
\begingroup
\setlength\abovedisplayskip{3pt plus 1pt}
\setlength\belowdisplayskip{2pt plus 1pt}
\begin{align}
& \mathbb{P}\Big(
    \sup_{\timeVar\in[0,\timeHorizon]}
      |\widehat{\survEnd}^{k\textsc{-NN}}(\timeVar|\fixedFeatureVector)
       -\survEnd(\timeVar|\fixedFeatureVector)|
        > \varepsilon
  \Big) \nonumber \\[-3\jot]
& \quad
    \le
      \exp\Big(-\frac{k\theta}{8}\Big)
      + \exp\Big(-\frac{n\featureDist(\mathcal{B}_{\fixedFeatureVector,h^*})}{8}\Big) \nonumber \\[-0.5\jot]
& \quad\quad
      + 2\exp\Big(-\frac{k\varepsilon^{2}\theta^{4}}{648}\Big)
      + \frac{8}{\varepsilon}\exp\Big(-\frac{k\varepsilon^{2}\theta^{2}}{162}\Big).
\label{eq:kNN-ptwise-bound}
\end{align}
\end{thm}
\endgroup
The four terms in the above bound correspond to penalties for the following bad events:
\begin{enumerate}[leftmargin=*,topsep=-3pt,itemsep=-1ex,partopsep=-10ex,parsep=1ex]

\item Too few of the~$k$ nearest neighbors survive beyond time~$\timeHorizon$ (in the worst case, none do, so from the data alone, we would suspect Assumption A3 to not hold)

\item The~$k$ nearest neighbors are not all within critical distance~$h^*$ of~$\fixedFeatureVector$ (by Assumption A4, the nearest neighbors should be close to $\fixedFeatureVector$ to guarantee that they provide accurate information about $\survEnd(\cdot|\fixedFeatureVector)$)

\item The number of nearest neighbors~$k$ is too small such that when we form an empirical distribution using their $\obsVar_i$ values, this empirical distribution has not converged to its expectation, which is a CDF (note that when the previous bad event does not happen, then this CDF is approximately $1-\obsEnd(\cdot|\fixedFeatureVector)$)

\item The $k$-NN survival estimator can be viewed as solving a specific $k$-NN regression problem, which averages over the $k$ nearest neighbors' ``labels'' (if $\featureVar_i$ is one of the $k$ nearest neighbors of $\fixedFeatureVector$, then its label is taken to be $-\frac{\eventVar_i\ind\{\obsVar_i\le t\}}{\obsEnd(\obsVar_i|\fixedFeatureVector)}$, i.e., this label depends on an accurate estimate for $\obsEnd(\cdot|\fixedFeatureVector)$, which the previous bad event is about). This last bad event is that the average of these $k$ labels is not close to its expectation due to~$k$ being too small.

\end{enumerate}
In our analysis, preventing bad event~\#1 is pivotal to upper-bounding the $k$-NN survival estimator's error by those of the $k$-NN CDF estimation and $k$-NN regression problems. Subsequently, bad event~\#2 is about controlling the bias of these $k$-NN CDF and $k$-NN regression estimators, i.e., making sure their expectations are close to desired target values. Bad events~\#3 and~\#4 relate to controlling the variances of these $k$-NN CDF and $k$-NN regression estimates.

The observation that CDF estimation and regression subproblems arise is based on nonasymptotic analysis of the standard Kaplan-Meier estimator by \citet{foldes_1981}. For controlling the bias and variance of $k$-NN CDF and $k$-NN regression estimators, we use proof techniques by \citet{chaudhuri_dasgupta_2014}.

\begingroup
\setlength\abovedisplayskip{3pt plus 1pt}
\setlength\belowdisplayskip{2pt plus 1pt}
To understand the consequences of Theorem~\ref{thm:kNN-survival}, especially how it relates to the rate of convergence for the $k$-NN survival estimator, we examine sufficient conditions for which the RHS of bound~\eqref{eq:kNN-ptwise-bound} is at most a user-specified error probability $\probError\in(0,1)$. To achieve this, we can ask that each of the four terms be bounded above by $\probError/4$. In doing so, a simple calculation reveals that the theorem's conditions on $k$ and $n$ are met if
\begin{equation}
k\ge\frac{648}{\varepsilon^{2}\theta^{4}}\log\frac{32}{\varepsilon\probError},
\quad\text{and}\quad n\ge\frac{2k}{\featureDist(\mathcal{B}_{x,h^*})}.
\label{eq:kNN-sufficient}
\end{equation}
This pointwise guarantee highlights a key feature of nearest neighbor methods in that they depend on the \emph{intrinsic} dimension of the data \citep{samory_2011,samory_2013}. For example, consider when the feature space is $\featureSpace=\realNumbers^d$. Even though the data have \emph{extrinsic} dimension~$d$, it could be that $\featureDist(\mathcal{B}_{\fixedFeatureVector,h^*})$ scales as $(h^*)^{d'}$ for some $d'<d$. This could happen if the data reside in a low dimensional portion of the higher dimensional space (e.g., $\text{supp}(\featureDist)$ is a convex polytope of $d'<d$ dimensions within~$\realNumbers^d$). Thus, examining the second inequality of \eqref{eq:kNN-sufficient}, the number of training data $n$ sufficient for guaranteeing a low error in estimating $\survEnd(\cdot|\fixedFeatureVector)$ scales exponentially in the intrinsic dimension at $x$ (roughly, the smallest $d'>0$ for which $\featureDist(\mathcal{B}_{\fixedFeatureVector,r})\sim r^{d'}$ for all small enough~$r$).
\endgroup

Sufficient conditions~\eqref{eq:kNN-sufficient} also tell us when we can consistently estimate $\survEnd(\cdot|\fixedFeatureVector)$ for a fixed $\fixedFeatureVector$. Specifically for any error tolerance $\varepsilon>0$, to have the error probability~$\probError$ go to 0, the condition on $k$ suggests that we take $k\rightarrow\infty$, which also means that $n\rightarrow\infty$. At the same time, the condition relating~$n$ and~$k$ says that we should have $k/n \le \featureDist(\mathcal{B}_{\fixedFeatureVector,h^*})/2$. Recall that $h^*=(\frac{\varepsilon\theta}{18\Lips})^{1/\holderIndex}$, so if we pick $\varepsilon$ to be arbitrarily small, then $\featureDist(\mathcal{B}_{\fixedFeatureVector,h^*})\rightarrow0$, so we want $k/n\rightarrow0$. We remark that choosing $k$ as a function of $n$ to satisfy $k\rightarrow\infty$ and $k/n\rightarrow0$ are the usual conditions on $k$ for $k$-NN classification and regression to be weakly consistent \citep{cover_1967,stone_1977}.

\begingroup
\setlength\abovedisplayskip{3pt plus 1pt}
\setlength\belowdisplayskip{2pt plus 1pt}
As for how $k$ should scale with $n$, this depends on $\featureDist(\mathcal{B}_{\fixedFeatureVector,h^*})$. For example, if $\featureDist(\mathcal{B}_{\fixedFeatureVector,h^*})\sim (h^*)^d$, then the second inequality of sufficient conditions~\eqref{eq:kNN-sufficient} says that $k$ should scale at most as $(h^*)^d n \sim \varepsilon^{d/\holderIndex} n$. In this case, our next result shows that the $k$-NN estimator is strongly consistent. Since $h^*$ is a function of $\varepsilon$, which we now take to go to~0, formally we shall assume that $\featureDist(\mathcal{B}_{x,r})\ge p_{\min}r^d$ for all $r\in(0,r^*]$ for some positive constants $p_{\min}$, $d$, and $r^*$. Thus, as we shrink $\varepsilon$ toward~0, once~$\varepsilon$ becomes small enough (namely $\varepsilon \le \frac{18\Lips (r^*)^\holderIndex}\theta$), then $h^*=(\frac{\varepsilon\theta}{18\Lips})^{1/\holderIndex}\in(0,r^*]$ and so $\featureDist(\mathcal{B}_{x,h^*})\ge p_{\min}(h^*)^d$.
\begin{cor}[$k$-NN strong consistency rate]
\label{thm:kNN-strong-consistency}
Under Assumptions A1--A4, let $x\in\text{supp}(\featureDist)$, and suppose that there exist constants $p_{\min} > 0$, $d > 0$, and $r^*>0$ such that $\featureDist(\mathcal{B}_{x,r})\ge p_{\min}r^d$ for all $r\in(0,r^*]$. Then there are positive numbers $c_1=\Theta\big( \frac1{(\theta\Lips)^{2d/(2\holderIndex+d)}} \big)$, $c_2=\Theta\big( \frac{\theta^{(4\holderIndex+d)/(5\holderIndex+2d)}}{\Lips^{d/(5\holderIndex+2d)}} \big)$, and $c_3=\Theta\big( \frac{\Lips^{d/(2\holderIndex+d)}}{\theta^{(4\holderIndex+d)/(2\holderIndex+d)}} \big)$ such that by choosing the number of nearest neighbors to be
$
k_n:=\lfloor c_1 n^{2\holderIndex/(2\holderIndex+d)}\big(\log(c_2 n)\big)^{d/(2\holderIndex+d)}\rfloor,
$
with probability 1,
\begin{align*}
& \limsup_{n\rightarrow\infty}
  \bigg\{
  \frac{
  \sup_{\timeVar\in[0,\timeHorizon]}
    |\widehat{\survEnd}^{k_n\textsc{-NN}}(\timeVar|\fixedFeatureVector)
     -\survEnd(\timeVar|\fixedFeatureVector)|}{c_3 \big(\frac{\log(c_2 n)}{n}\big)^{\holderIndex/(2\holderIndex+d)}}
  \bigg\}
< 1.
\end{align*}
\end{cor}
The above corollary follows from setting error probability $\probError=1/n^2$ in sufficient conditions~\eqref{eq:kNN-sufficient}, solving the inequalities in the sufficient conditions for $\varepsilon$, $n$, and~$k$ (and thus finding coefficients $c_1$, $c_2$, and $c_3$ above), and finally applying the Borel-Cantelli lemma. Closed-form equations for $c_1$, $c_2$, and $c_3$ are in Appendix~\ref{sec:pf-kNN-strong-consistency}.
\endgroup

\textbf{Near-optimality.} Our nonasymptotic bound~\eqref{eq:kNN-ptwise-bound} turns out to essentially be optimal. Consider when the censoring times always occur after the survival times, i.e., nothing is censored. Then the problem reduces to conditional CDF estimation ($1-S(\cdot|x)$ is a conditional CDF), for which the minimax lower bound for \emph{expected squared error} under slightly more assumptions than we impose is $n^{-2\holderIndex/(2\holderIndex+d)}$ \citep[Theorem~3]{chagny_2014}. Our result implies an upper bound on the expected squared error. First, note that
\begingroup
\setlength\abovedisplayskip{3pt plus 1pt}
\setlength\belowdisplayskip{2pt plus 1pt}
\begin{align}
&\mathbb{E}\Big[\int_0^\timeHorizon(\widehat{\survEnd}^{k_n\textsc{-NN}}(\timeVar|\fixedFeatureVector)-\survEnd(\timeVar|\fixedFeatureVector))^2 dt\Big] \nonumber \\
&\quad \le \timeHorizon \mathbb{E}\Big[\sup_{\timeVar\in[0,\timeHorizon]}|\widehat{\survEnd}^{k_n\textsc{-NN}}(\timeVar|\fixedFeatureVector)-\survEnd(\timeVar|\fixedFeatureVector)|^2\Big].
\label{eq:near-optimality1}
\end{align}
Next, sufficient conditions~\eqref{eq:kNN-sufficient} say that with probability at least $1-\probError$, none of the bad events happen so ${\sup_{\timeVar\in[0,\timeHorizon]}|\widehat{\survEnd}^{k_n\textsc{-NN}}(\timeVar|\fixedFeatureVector)-\survEnd(\timeVar|\fixedFeatureVector)|} \le \varepsilon$ (for which we can square both sides and bring the square into the supremum); otherwise the supremum norm error is at worst~1. Hence,
\begin{equation}
\mathbb{E}\Big[\sup_{\timeVar\in[0,\timeHorizon]}|\widehat{\survEnd}^{k_n\textsc{-NN}}(\timeVar|\fixedFeatureVector)-\survEnd(\timeVar|\fixedFeatureVector)|^2\Big]
\le \varepsilon^2\cdot1 + 1\cdot\probError,
\label{eq:near-optimality2}
\end{equation}
where on the RHS, the first term is the worst-case squared supremum norm error $\varepsilon^2$ when none of the bad events happen (this happens with probability at least ${1-\probError}\le1$), and the second term is the worst-case squared supremum norm error of~1 (this happens with probability at most~$\probError$).

It suffices to set $\probError=\varepsilon^2$ and find precise conditions on~$k$, $n$, and $\varepsilon$ so that sufficient conditions~\eqref{eq:kNN-sufficient} hold (the calculation is similar to the one for deriving Corollary~\ref{thm:kNN-strong-consistency}). By doing this calculation and combining inequalities~\eqref{eq:near-optimality1} and~\eqref{eq:near-optimality2}, we get that the \mbox{$k$-NN} survival estimator has expected squared error $\widetilde{\mathcal{O}}(n^{-2\holderIndex/(2\holderIndex+d)})$, even if there is right-censoring.

\textbf{Results for random test feature vectors.}
As there are a number of standard approaches for translating pointwise guarantees to ones accounting for randomness in sampling $\featureVar=\fixedFeatureVector\sim\featureDist$, we only focus on one such technique and briefly mention some others. Specifically, we consider a simple approach in which we partition the feature space $\featureSpace$ into a ``good'' region $\featureSpace_{\text{good}}$ with sizable probability mass (where many training data are likely to be), and a bad region $\featureSpace_{\text{bad}}$ where we tolerate error (where there are likely to be too few training data). Using the same idea as described in Section 3.3.1 of \citet{george_devavrat_book}, we define the \emph{sufficient mass region} as
\begin{align*}
& \featureSpace_{\text{good}}(\featureDist;p_{\min},d,r^{*})\\
& :=\{\fixedFeatureVector\in\text{supp}(\featureDist)\,\!:\,\!\featureDist(\mathcal{B}_{\fixedFeatureVector,r})\ge p_{\min}r^{d}\;\,\forall r\in(0,r^*]\},
\end{align*}
and $\featureSpace_{\text{bad}}(\featureDist;p_{\min},d,r^{*})=\featureSpace\setminus\featureSpace_{\text{good}}(\featureDist;p_{\min},d,r^{*})$. The sufficient mass region for feature distribution $\featureDist$ corresponds to portions of $\text{supp}(\featureDist)$ that behave like they have dimension~$d$. Returning to the previous example, if $\featureSpace=\realNumbers^d$ and $\text{supp}(\featureDist)$ is a full-dimensional convex polytope, then there exists a $p_{\min}>0$ and $r^*>0$ such that $\featureSpace_{\text{good}}(\featureDist;p_{\min},d,r^*) = \text{supp}(\featureDist)$.

In general, when feature vector $\featureVar\sim\featureDist$ lands in $\featureSpace_{\text{good}}(\featureDist;p_{\min},d,h^*)$, then the conditions of Theorem~\ref{thm:kNN-survival} are satisfied and, moreover, $\featureDist(\mathcal{B}_{\fixedFeatureVector,h^*}) \ge p_{\min}(h^*)^d$. We readily obtain the following corollary.
\begin{cor}[$k$-NN bound for random test point]
Under the same conditions as Theorem~\ref{thm:kNN-survival} except now sampling test point $\featureVar\sim\featureDist$, then over randomness in the training data and $\featureVar$,
\begin{align*}
& \mathbb{P}\Big(
    \sup_{\timeVar\in[0,\timeHorizon]}
      |\widehat{\survEnd}^{k\textsc{-NN}}(\timeVar|\featureVar)-\survEnd(\timeVar|\featureVar)|
        > \varepsilon
  \Big)\\[-2.7\jot]
& \quad
    \le
      \exp\Big(-\frac{k\theta}{8}\Big)
      + \exp\Big(-\frac{np_{\min}(h^*)^d}{8}\Big) \\[-0.7\jot]
& \quad\quad
      + 2\exp\Big(-\frac{k\varepsilon^{2}\theta^{4}}{648}\Big)
      + \frac{8}{\varepsilon}\exp\Big(-\frac{k\varepsilon^{2}\theta^{2}}{162}\Big) \\
& \quad\quad
      + \featureDist\big(\featureSpace_{\text{bad}}(\featureDist;p_{\min},d,h^*)\big).
\end{align*}
\end{cor}
Thus, if there exists $p_{\min}>0$, $d>0$, and $r^*>0$ such that $\featureSpace_{\text{good}}(\featureDist;p_{\min},d,r^*) = \text{supp}(\featureDist)$, then strong consistency of $\widehat{\survEnd}^{k\textsc{-NN}}(\cdot|\featureVar)$ at the rate of Corollary~\ref{thm:kNN-strong-consistency} holds over randomness in training data and $X\sim\featureDist$.
\endgroup

Other approaches are possible to obtain guarantees over randomness in both training data and $\featureVar$ from guarantees for fixed $\featureVar=\fixedFeatureVector$. For example, there are notions similar to the sufficient mass region specific to Euclidean space such as the \emph{strong minimal mass assumption} of \citet{gadat_2016} and the \emph{strong density assumption} of \citet{audibert_2007}. An alternative strategy that stays in separable metric spaces is to use covering numbers from metric entropy. For details, see Section 3.3.3 of \citet{george_devavrat_book}.

\vspace{-.5em}
\subsection*{Kernel estimator results}
\vspace{-.25em}

\begingroup
\setlength\abovedisplayskip{3pt plus 1pt}
\setlength\belowdisplayskip{3pt plus 1pt}

Our kernel result uses an additional decay assumption:
\begin{itemize}[topsep=0pt,itemsep=1ex,partopsep=1ex,parsep=1ex]
\item[\textbf{A5.}] \emph{The kernel function $K$ monotonically decreases, and there exists a standardized distance $\stanDistThresh >0$ such that $K(s)>0$ for all $s\in[0,\stanDistThresh ]$ and $K(s)=0$ for $s>\stanDistThresh$.} This assumption ensures that training data sufficiently far from~$x$ have no impact on our estimation of $\survEnd(\cdot|x)$. (Small proof changes can be made to allow ${K(\stanDistThresh)=0}$, e.g., to handle triangle and Epanechnikov kernels.)
\end{itemize}
Our kernel result depends on the kernel function's maximal and minimal positive values, namely $K(0)$ and $K(\stanDistThresh)$. We let $\kappa := K(\stanDistThresh)/K(0)$, and define
\[
\Lips_K
:=\max\Big\{
\frac{2\timeHorizon}{\theta\kappa}(\lips_{\textsc{\tiny{\survVar}}}+\lips_{\textsc{\tiny{\censVar}}}),\,
\lips_{\survSubscript}\timeHorizon + \frac{\survDensity^*\lips_{\censSubscript}\timeHorizon^2}{2}
\Big\}.
\]
The first term in the maximization (related to CDF estimation) has an extra $1/\kappa$ factor compared to $\Lips$.

As our kernel survival estimator guarantee is similar to that of the $k$-NN estimator, we only present its pointwise version. Deriving a corresponding strong consistency rate, accounting for randomness in sampling~$\featureVar\sim\featureDist$, and showing near-optimality can be done as before. In particular, the two methods have similar asymptotic behavior.
\begin{thm}[Kernel pointwise guarantee]
\label{thm:kernel-survival}
Under Assumptions A1--A5, let $\varepsilon\in(0,1)$ be a user-specified error tolerance. Suppose that the threshold distance satisfies $h\in(0,\frac1\stanDistThresh(\frac{\varepsilon\theta}{18\Lips_K})^{1/\holderIndex}]$, and the number of training data satisfies $n\ge\frac{144}{\varepsilon\theta^{2}\featureDist(\mathcal{B}_{\fixedFeatureVector,\stanDistThresh h})\kappa}$. For any $\fixedFeatureVector\in\text{supp}(\featureDist)$,
\begin{align}
 & \mathbb{P}\Big(\sup_{\timeVar\in[0,\timeHorizon]}|\widehat{\survEnd}^K(\timeVar|\fixedFeatureVector;h)-\survEnd(\timeVar|\fixedFeatureVector)|>\varepsilon\Big) \nonumber \\
 & \;\le\exp\Big(-\frac{n\featureDist(\mathcal{B}_{\fixedFeatureVector,\stanDistThresh h})\theta}{16}\Big) + \exp\Big(-\frac{n\featureDist(\mathcal{B}_{\fixedFeatureVector,\stanDistThresh h})}{8}\Big) \nonumber \\
 & \;\quad+\frac{216}{\varepsilon\theta^2\kappa}\exp\Big(-\frac{n\featureDist(\mathcal{B}_{\fixedFeatureVector,\stanDistThresh h})\varepsilon^{2}\theta^{4}\kappa^4}{11664}\Big) \nonumber \\
 & \;\quad+\frac{8}{\varepsilon}\exp\Big(-\frac{n\featureDist(\mathcal{B}_{\fixedFeatureVector,\stanDistThresh h})\varepsilon^{2}\theta^{2}\kappa^2}{324}\Big). \label{eq:kernel-ptwise-bound}
\end{align}
\end{thm}
\endgroup
As with the $k$-NN analysis, the kernel estimator analysis involves two subproblems, a kernel CDF estimation (i.e., using weighted samples to construct an empirical distribution function) and a kernel regression. We remark that $k$-NN CDF estimation is straightforward to analyze because the different data points have equal weight, so we can apply the Dvoretzky-Kiefer-Wolfowitz (DKW) inequality. To handle weighted empirical distributions, we establish the following nonasymptotic bound.
\begin{prop}[Weighted empirical distribution inequality]
\label{lem:weighted-edf}Let $Z_{1},\dots,Z_{\ell}$ be independent real-valued continuous random variables. Let $\weightVar_{1},\dots,\weightVar_\ell$ be any sequence of nonnegative constants such that $\sum_{i=1}^\ell \weightVar_i > 0$. Consider the following weighted empirical distribution function:
\begingroup
\setlength\abovedisplayskip{3pt plus 1pt}
\setlength\belowdisplayskip{3pt plus 1pt}
\[
\widehat{F}(\timeVar)
:= \sum_{i=1}^\ell
     \frac{\weightVar_i}
          {\sum_{j=1}^\ell \weightVar_j}
     \ind\{Z_i\le \timeVar\}
\quad\text{for }\timeVar\in\realNumbers,
\]
for which we define $F(t):=\mathbb{E}[\widehat{F}(\timeVar)]$. For every $\varepsilon\in(0,1]$,
\[
\mathbb{P}\Big(\sup_{\timeVar\in\realNumbers}|\widehat{F}(\timeVar)-F(\timeVar)|\!>\!\varepsilon\Big)\\
\! \le \! \frac6\varepsilon \exp\!\Big(-\!\frac{2\varepsilon^{2}(\sum_{j=1}^{\ell}\weightVar_j)^{2}}{9\sum_{i=1}^{\ell}\weightVar_i^{2}}\Big).
\]
\endgroup
\end{prop}
\textbf{Box kernel, weighted $\bm{k}$-NN.} If instead the kernel survival estimator is used with a box kernel (uniform weights), then we can use the DKW inequality instead of Proposition~\ref{lem:weighted-edf}, leading to a slightly stronger pointwise guarantee (Theorem~\ref{thm:h-near-survival} in the appendix). We remark that proof ideas for our $k$-NN and kernel survival estimators can be combined to derive results for weighted $k$-NN survival estimators.

\textbf{Choosing~$\bm{k}$ and $\bm{h}$ via a validation set.} Our main results choose $k$ and $h$ in a way that depends on unknown model parameters. In practice, validation data could be used to select $k$ and $h$ via minimizing the integrated Brier score \citep{graf1999assessment}. We obtain a nonasymptotic guarantee for a slight variant of the validation strategy by \citet{lowsky_2013} in Appendix~\ref{sec:validation}. The high-level proof idea is simple. For example, for the $k$-NN estimator $\widehat{\survEnd}^{k\textsc{-NN}}$, suppose we have an independent validation set of size $n$. Provided that the choices of~$k$ that the user sweeps over for validation include one good choice according to Theorem~\ref{thm:kNN-survival}, then for large enough~$n$, estimator $\widehat{\survEnd}^{k\textsc{-NN}}$ has a validation error that approaches that of~$\survEnd$. Our proof is a bit nuanced and requires controlling both additive and multiplicative error in tail probability estimates, using our analysis for Nelson-Aalen-based nearest neighbor and kernel estimators (given in Appendix~\ref{sec:Nelson-Aalen}).

\begingroup
\setlength\abovedisplayskip{4pt plus 1pt}
\setlength\belowdisplayskip{4pt plus 1pt}

\vspace{-.5em}
\subsection*{Distributions satisfying Assumptions A1--A4}
\vspace{-.25em}

We now provide example models that satisfy Assumptions A1--A4. In these examples, the feature space $\featureSpace$ and distance $\rho$ are Euclidean, and the H\"{o}lder exponent is $\holderIndex=1$ (so $\lips_{\survSubscript}$ and $\lips_{\censSubscript}$ are Lipschitz constants).

\begin{example}[Exponential regression]
\label{ex:exp-regress}
Let $\featureSpace=\realNumbers^d$, and $\featureDist$ be any Borel probability measure with compact, convex support (so Assumption A1 is met). We define conditional survival function~$\survEnd(\timeVar|\fixedFeatureVector)$ using the hazard function $\hazard_{\survSubscript}(\timeVar|\fixedFeatureVector)=-\frac{\partial}{\partial\timeVar}\log\survEnd(\timeVar|\fixedFeatureVector)=\hazard_{\survSubscript,0} \exp(\fixedFeatureVector^\top\beta_{\survSubscript})$ with parameters $\hazard_{\survSubscript,0}>0$ and ${\beta_{\survSubscript}\in\realNumbers^d}$. Then
\begin{align*}
\survEnd(\timeVar|\fixedFeatureVector)
&= \exp\Big(-\int_0^\timeVar\hazard_{\survSubscript,0}\exp(\fixedFeatureVector^\top\beta_{\survSubscript})ds\Big) \\
&= \exp(-\hazard_{\survSubscript,0}e^{\fixedFeatureVector^\top\beta_{\survSubscript}}\timeVar),
\end{align*}
which implies that the distribution $\mathbb{P}_{\survVar|\featureVar=\fixedFeatureVector}$ $($which has CDF $1-\survEnd(\cdot|\fixedFeatureVector))$ is exponentially distributed with parameter $\hazard_{\survSubscript,0}e^{\fixedFeatureVector^\top\beta_{\survSubscript}}$. We could similarly define the censoring time conditional distribution through the hazard function $\hazard_{\censSubscript}(\timeVar|\fixedFeatureVector)=\hazard_{\censSubscript} \exp(\fixedFeatureVector^\top\beta_{\censSubscript})$, with $\hazard_{\censSubscript,0}>0$ and $\beta_{\censSubscript}\in\realNumbers^d$. In this case, distribution $\mathbb{P}_{\censVar|\featureVar=\fixedFeatureVector}$ is exponentially distributed with parameter $\hazard_{\censSubscript,0}e^{\fixedFeatureVector^\top\beta_{\censSubscript}}$. At this point, Assumption A2 is also met since for any $\fixedFeatureVector\in\text{supp}(\featureDist)$, distributions $\mathbb{P}_{\survVar|\featureVar=\fixedFeatureVector}$ and $\mathbb{P}_{\censVar|\featureVar=\fixedFeatureVector}$ correspond to continuous random variables.

We now present valid choices for $\theta$ and $\timeHorizon$ for Assumption A3. Recall that the observed time is $\obsVar=\min\{\survVar,\censVar\}$. Conditioned on $\featureVar=\fixedFeatureVector$, the minimum of independent exponential random variables is exponential. In particular, distribution $\mathbb{P}_{\obsVar|\featureVar=\fixedFeatureVector}$ is exponentially distributed with parameter $\omega(\fixedFeatureVector):=\hazard_{\survSubscript,0}e^{\fixedFeatureVector^\top\beta_{\survSubscript}} + \hazard_{\censSubscript,0}e^{\fixedFeatureVector^\top\beta_{\censSubscript}}$. Thus, if we pick $\theta=1/2$, then a valid choice for $\timeHorizon$ would be the smallest possible median of distribution $\mathbb{P}_{\obsVar|\featureVar=\fixedFeatureVector}$ across all $\fixedFeatureVector\in\text{supp}(\featureDist)$. Note that the median of $\mathbb{P}_{\obsVar|\featureVar=\fixedFeatureVector}$ is $({\log2})/{\omega(\fixedFeatureVector)}$. Thus, we can pick $\timeHorizon=\min_{\fixedFeatureVector\in\text{supp}(\featureDist)}\{({\log2})/{\omega(\fixedFeatureVector)}\}$.

Lastly, for Assumption A4, due to $\text{supp}(\featureDist)$ being compact and convex, the conditional survival time density $\survDensity(\timeVar|\cdot)$ has finite Lipschitz constant
\[
\lips_{\survSubscript}=\sup_{\fixedFeatureVector\in\text{supp}(\featureDist),\timeVar\in[0,\timeHorizon]} \Big\|\frac{\partial \survDensity(\timeVar|\fixedFeatureVector)}{\partial \fixedFeatureVector}\Big\|_2,
\]
where $\|\cdot\|_2$ is Euclidean norm, and $\frac{\partial \survDensity(\timeVar|\fixedFeatureVector)}{\partial \fixedFeatureVector} = \survDensity(\timeVar|\fixedFeatureVector)(1-\hazard_{\survSubscript,0}e^{\fixedFeatureVector^\top\beta_{\survSubscript}}\timeVar)\beta_{\survSubscript}.
$
We could similarly choose Lipschitz constant $\lips_{\censSubscript}$ for the conditional censoring time density $\censDensity(\timeVar|\cdot)$.
\end{example}
This exponential regression example can easily be generalized to Weibull regression, which is another proportional hazards model (see Appendix~\ref{sec:weibull-regression}).
\begin{example}[Weibull mixture]
To give an example that is not a proportional hazard model that satisfies Assumptions A1--A4, consider an integer-valued one-dimensional feature vector $\featureVar\sim\text{Uniform}\{1,2,\dots,100\}$. For a threshold ${\nu\in(1,100)}$, if $\featureVar\le\nu$, then we sample survival time~$\survVar$ from a Weibull distribution with shape parameter $q>0$ and scale parameter $\psi_{\survSubscript,1}>0$. Otherwise if $\featureVar>\nu$, then we sample $\survVar$ from a Weibull distribution still with shape parameter $q$ but a different scale parameter $\psi_{\survSubscript,2}>0$. Thus, the marginal distribution of $\survVar$ is a mixture of two Weibull distributions. We similarly define the censoring time $\censVar$ to be a mixture of two Weibull distributions with common shape parameter $q$ and different scale parameters $\psi_{\censSubscript,1}>0$ and $\psi_{\censSubscript,2}>0$; we sample $\censVar$ from the first component using the same threshold $\nu$ as before, i.e., when $\featureVar\le\nu$.

Conditioned on $\featureVar$, the distribution of observed time~$\obsVar=\min\{\survVar,\censVar\}$ is now one of two possible Weibull distributions (the minimum of independent Weibull distributions with shape parameter $q$ is still Weibull with shape~$q$): if ${\featureVar\le\nu}$, then $\obsVar$ is Weibull with shape $q$ and scale ${(\psi_{\survSubscript,1}^{-q}+\psi_{\censSubscript,1}^{-q})^{-1/q}}$. Otherwise $\obsVar$ is Weibull with shape~$q$ and scale ${(\psi_{\survSubscript,2}^{-q}+\psi_{\censSubscript,2}^{-q})^{-1/q}}$. For Assumption A3, we can choose $\theta=1/2$ and $\timeHorizon$ to be the smaller median of the two possible Weibull distributions for $\obsVar$, i.e., $\timeHorizon=\big[ \min\big\{\frac1{\psi_{\survSubscript,1}^{-q}+\psi_{\censSubscript,1}^{-q}},\frac1{\psi_{\survSubscript,2}^{-q}+\psi_{\censSubscript,2}^{-q}}\big\}\log2\big]^{1/q}.$ Lastly, for Assumption A4, since $|\text{supp}(\featureDist)|$ is finite, we can set the Lipschitz constant $\lips_{\survSubscript}$ to be
\[
\lips_{\survSubscript}=\!\!\sup_{\fixedFeatureVector,\fixedFeatureVector'\in\{1,2,\dots,100\}\text{ s.t.~}\fixedFeatureVector\ne\fixedFeatureVector',\timeVar\in[0,\timeHorizon]}\!\!\frac{|\survDensity(\timeVar|\fixedFeatureVector)-\survDensity(\timeVar|\fixedFeatureVector')|}{|\fixedFeatureVector-\fixedFeatureVector'|}.
\]
Lipschitz constant $\lips_{\censSubscript}$ can be chosen similarly.
\end{example}

\endgroup

\vspace{-1em}
\section{Experimental Results}
\label{sec:experiments}
\vspace{-0.25em}

We benchmark the four nonparametric estimators stated in Section \ref{sec:problem-setup} against two baselines: the Cox proportional hazards model \citep{cox_1972}, and a second baseline that explicitly solves the $k$-NN CDF estimation and $k$-NN regression subproblems (in succession) that arise in the theoretical analysis for the $k$-NN survival estimator (we refer to this method as \textsc{cdf-reg}; for simplicity we only consider the $k$-NN variant and not the kernel variant). According to our theory, the $k$-NN survival estimator's error should be upper-bounded by that of \textsc{cdf-reg}. For the $k$-NN, \textsc{cdf-reg}, and kernel methods, we standardize features and use $\ell_2$ and $\ell_1$ distances. For the $k$-NN and \textsc{cdf-reg} methods, we also consider their weighted versions using a triangle kernel.\footnote{Let $\featureVar_{(i)}$ denote the $i$-th nearest neighbor of test point $\fixedFeatureVector$. Then weighted $k$-NN assigns $\featureVar_{(i)}$ to have weight $K\big(\frac{\rho(\fixedFeatureVector,\featureVar_{(i)})}{\rho(\fixedFeatureVector,\featureVar_{(k)})}\big)$.} For the kernel method, we use box and triangle kernels. We also have results for more kernel choices in Appendix~\ref{sec:experiment-details} (the Epanechnikov kernel performs as well as the triangle kernel, and truncated Gaussian kernels tend to perform poorly).

\begin{table}[!t]
\centering
\begin{tabular}{c|c|c|c}
Dataset        & Description               & \# subjects  & \# dim. \\
\hline
\textsc{pbc}   & primary biliary cirrhosis & 276  & 17  \\
\textsc{gbsg2} & breast cancer             & 686  &  8  \\
\textsc{recid} & recidivism                & 1445 & 14  \\
\textsc{kidney}   & dialysis    & 1044 & 53
\end{tabular}
\vspace{-1em}
\caption{Characteristics of the survival datasets used.}
\label{tab:datasets}
\vspace{-1.5em}
\end{table}

\begin{figure*}[!t]
\centering
\includegraphics[width=0.407\linewidth, trim=0pt 9pt 0pt 10pt, clip]{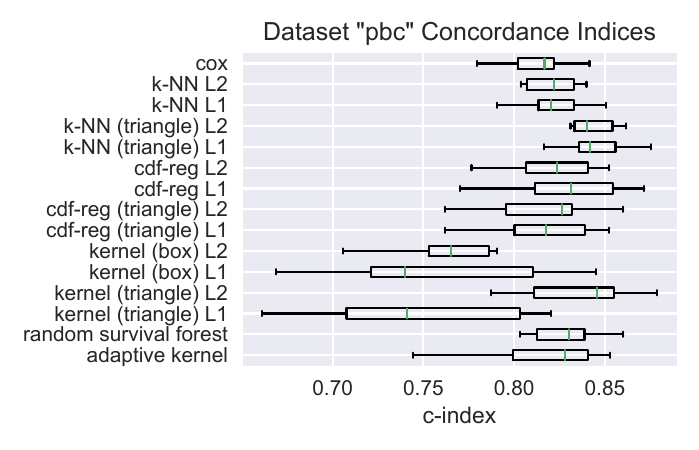}~~~
\includegraphics[width=0.407\linewidth, trim=0pt 9pt 0pt 10pt, clip]{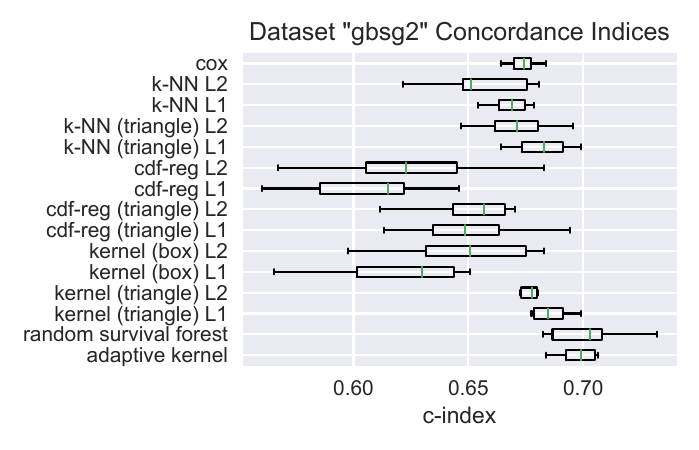} \\ \vspace{.1em}
\includegraphics[width=0.407\linewidth, trim=0pt 9pt 0pt 10pt, clip]{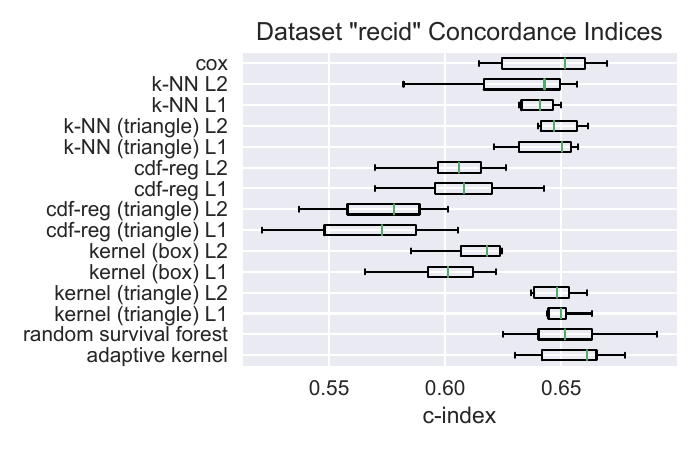}~~~
\includegraphics[width=0.407\linewidth, trim=0pt 9pt 0pt 10pt, clip]{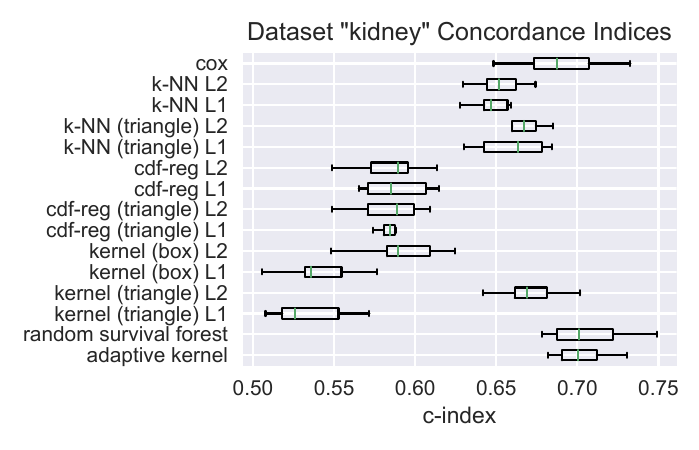}
\vspace{-1.1em}
\caption{Survival analysis prediction results on four datasets using the concordance index (c-index; higher means more accurate prediction). Each dataset is randomly split into 10 train/test splits, resulting in the different c-index scores per method.}
\label{fig:cindices}
\vspace{-1em}
\end{figure*}

We run the above methods on four datasets. Three are publicly available: the Mayo Clinic primary biliary cirrhosis dataset (abbreviated \textsc{pbc}) \citep{fleming_1991}, the German Breast Cancer Study Group 2 dataset (\textsc{gbsg2}) \citep{schumacher_1994}, and the recidivism dataset (\textsc{recid}) from \citet{chung_1991}. The fourth dataset we use is from a study on dialysis patients (\textsc{kidney}) by \citet{ganssauge2016exploring}. For \textsc{pbc}, \textsc{gbsg2}, and \textsc{kidney}, the survival time refers to time until death whereas for \textsc{recid}, the ``survival time'' refers to time until a convicted criminal reoffends. The dataset sizes and number of features are reported in Table~\ref{tab:datasets}. In all cases, subjects with any missing features are removed. For the \textsc{kidney} dataset, features with too many missing entries are also removed.

For each dataset, the basic experiment we run is as follows. We randomly divide the dataset into a 70\%/30\% train/test split. Using the training portion, for all methods except Cox proportional hazards, we run 5-fold cross-validation to select algorithm parameters before training on the full training set and predicting on the test set; prediction error is measured using the standard survival analysis accuracy metric of concordance index (c-index; higher means more accurate) \citep{harrell_1982} (details on c-index calculation and the parameter grids used are in Appendix~\ref{sec:experiment-details}). This basic experiment is repeated 10 times, so that every dataset gets randomly divided into train/test sets 10 different ways. Results are shown in Figure~\ref{fig:cindices}.

We find that random survival forests and the adaptive kernel method (with a kernel learned using random survival forests) tend to achieve similar accuracy scores per dataset. These two methods achieve the best accuracy in the \textsc{gbsg2}, \textsc{recid}, and \textsc{kidney} datasets. However, on the smallest dataset considered (\textsc{pbc} with 276 subjects), while random survival forests and the adaptive kernel method outperform nearly all the other methods, their concordance indices are noticeably lower than those of the weighted $k$-NN and kernel survival estimators (both using triangle kernels). Separately, we find that the $k$-NN survival estimator generally outperforms its corresponding \textsc{cdf-reg} variant (the only exception is in the unweighted $k$-NN case for the smallest dataset \textsc{pbc}). This agrees with our theory that the $k$-NN estimator's error is upper-bounded by that of \textsc{cdf-reg} when the training dataset is sufficiently large.

\vspace{-0.1em}
\section{Conclusions}

By combining contemporary metric-space-based nearest neighbor theory by \citet{chaudhuri_dasgupta_2014} with the classic Kaplan-Meier analysis of \citet{foldes_1981}, we have established new guarantees for nearest neighbor and kernel variants of Kaplan-Meier and Nelson-Aalen estimators. We suspect that other recent theoretical developments in nearest neighbor and kernel methods also carry over to the survival analysis setting, such as adaptive methods for choosing the number of nearest neighbors~$k$ or kernel bandwidth~$h$ \citep{goldenshluger_lepski_2011,samory_2011,goldenshluger_lepski_2013,samory_2013,anava_2016}, and error bounds that are uniform over test feature vectors rather than only over a randomly chosen test vector \citep{samory_2011,samory_2013}. However, these developments do not explain the success of random survival forests and the proposed adaptive kernel variant. When and why do these nonparametric survival estimators work well, and how does their theory differ from that of standard random forests for regression and classification? Are there better ways of learning a kernel for use with kernel survival estimation? These questions outline promising directions for future exploration.

\newpage

\section*{Acknowledgments}

The author thanks Donald K.K. Lee for extremely helpful discussions.

\bibliographystyle{icml2019}
\bibliography{nn_survey}

\newpage

\appendix

\section{Supplemental Material}

These appendices cover all the proofs for the paper. We begin with a high-level analysis outline (Appendix \ref{sec:analysis-overview}) followed by more detailed proofs (Appendices \ref{sec:pf-kNN-survival-lemmas}--\ref{sec:pf-weighted-edf}). An additional example model (Weibull regression) satisfying Assumptions A1--A4 is presented in Appendix~\ref{sec:weibull-regression}. Guarantees for nearest neighbor and kernel variants of the Nelson-Aalen estimator are in Appendix~\ref{sec:Nelson-Aalen}. Additional information on experimental results is in Appendix~\ref{sec:experiment-details}.

Before presenting the proof of the kernel survival estimator result, we present an intermediate result for what we call the fixed-radius NN survival estimator; the proof of the kernel survival estimator will reuse proof ideas used for this fixed-radius NN survival estimator.

\emph{Fixed-radius NN survival estimator.} We find all training subjects with feature vectors at most a user-specified distance~$h>0$ from~$\fixedFeatureVector$. Let $\neighborsNNh(\fixedFeatureVector)\subseteq[n]$ denote their indices. Then the fixed-radius NN estimator is $\widehat{\survEnd}^{\textsc{NN}(h)}(\timeVar|\fixedFeatureVector) := \widehat{\survEnd}^{\text{KM}}(\timeVar|\neighborsNNh(\fixedFeatureVector))$.

This estimator is a special case of the kernel survival estimator with kernel $K(s)={\ind\{s \le 1\}}$. However, because this estimator weights all neighbors found within radius $h$ equally, we can actually derive a stronger guarantee than for the kernel estimator.

\begin{thm}[Fixed-radius NN pointwise guarantees]
\label{thm:h-near-survival}
Under Assumptions A1--A4, let $\varepsilon\in(0,1)$ be a user-specified error tolerance. Suppose that the threshold distance satisfies $h\in(0,h^*]$ with $h^*=(\frac{\varepsilon\theta}{18\Lips})^{1/\holderIndex}$, and the number of training data satisfies $n\ge\frac{144}{\varepsilon\theta^{2}\featureDist(\mathcal{B}_{\fixedFeatureVector,h})}$. For any $\fixedFeatureVector\in\text{supp}(\featureDist)$,
\begin{align}
 & \mathbb{P}\Big(\sup_{\timeVar\in[0,\timeHorizon]}|\widehat{\survEnd}^{\textsc{NN}(h)}(\timeVar|\fixedFeatureVector)-\survEnd(\timeVar|\fixedFeatureVector)|>\varepsilon\Big)\nonumber\\
 & \!\le\exp\!\Big(\!-\!\frac{n\featureDist(\mathcal{B}_{\fixedFeatureVector,h})\theta}{16}\Big)\!+\exp\!\Big(\!-\!\frac{n\featureDist(\mathcal{B}_{\fixedFeatureVector,h})}{8}\Big)\nonumber\\
 & \!\;\,\! +\!2\exp\!\Big(\!-\!\frac{n\featureDist(\mathcal{B}_{\fixedFeatureVector,h})\varepsilon^{2}\theta^{4}}{1296}\!\Big)\hspace{-.07em}\!+\!\hspace{-.07em}\frac{8}{\varepsilon}\exp\!\Big(\!-\!\frac{n\featureDist(\mathcal{B}_{\fixedFeatureVector,h})\varepsilon^{2}\theta^{2}}{324}\!\Big). \label{eq:h-near-ptwise-bound}
\end{align}
Moreover, if there exist constants $p_{\min}>0$, $d>0$, and $r^*>0$ such that $\featureDist(\mathcal{B}_{x,r})\ge p_{\min}r^d$ for all ${r\in(0,r^*]}$, then using the numbers $c_2=\Theta\big( \frac{\theta^{(4\holderIndex+d)/(5\holderIndex+2d)}}{\Lips^{d/(5\holderIndex+2d)}} \big)$ and $c_3=\Theta\big( \frac{\Lips^{d/(2\holderIndex+d)}}{\theta^{(4\holderIndex+d)/(2\holderIndex+d)}} \big)$ as in Corollary~\ref{thm:kNN-strong-consistency}, letting $c_1' := (\frac{\theta c_3}{18\Lips})^{1/\holderIndex} = \Theta( \frac{1}{(\theta\Lips)^{2/(2\holderIndex+d)}} )$,
and choosing threshold
\[
h_n := c_1'\Big(\frac{\log(c_2 n)}{n}\Big)^{\frac1{2\holderIndex+d}},
\]
we have, with probability 1,
\begin{align*}
&
\limsup_{n\rightarrow\infty}
\bigg\{\frac{\sup_{\timeVar\in[0,\timeHorizon]}
    |\widehat{\survEnd}^{\textsc{NN}(h_n)}(\timeVar|\fixedFeatureVector)
     -\survEnd(\timeVar|\fixedFeatureVector)|}{c_3 \big(\frac{\log(c_2 n)}{n}\big)^{\holderIndex/(2\holderIndex+d)}} \bigg\}
< 1.
\end{align*}
\end{thm}
Bound~\eqref{eq:h-near-ptwise-bound} matches that of the $k$-NN estimator (bound~\eqref{eq:kNN-ptwise-bound}) with~$k$ replaced by $\frac12 n\featureDist(\mathcal{B}_{\fixedFeatureVector,h})$, and every instance of $h^*$ in the $k$-NN bound replaced by user-specified threshold~$h$, which we ask to be at most~$h^*$. The main change is that we now directly control how close training subjects must be to~$x$ to be declared as neighbors, but we lose control over how many of them there are. The second term in bound~\eqref{eq:h-near-ptwise-bound} is the penalty for not having at least $\frac12 n\featureDist(\mathcal{B}_{\fixedFeatureVector,h})$ neighbors.

The technical core of the paper resides in the analysis of the $k$-NN survival estimator (proofs of Theorems~\ref{thm:kNN-survival} and Corollary~\ref{thm:kNN-strong-consistency}). Our proofs for the analogous fixed-radius NN and kernel estimator guarantees primarily focus on aspects that differ from the \mbox{$k$-NN} case.

\section{Analysis Outline}
\label{sec:analysis-overview}

\begingroup
\setlength\abovedisplayskip{5.005pt plus 1pt}
\setlength\belowdisplayskip{5.005pt plus 1pt}

We outline the proof strategy for establishing the nonasymptotic \mbox{$k$-NN} estimator result (Theorem~\ref{thm:kNN-survival}). The fixed-radius NN and kernel analyses are similar. We denote $d_{\setS}^{+}(\timeVar)$ to be the number of training subjects in $\setS\subseteq[n]$ who survive beyond time~$\timeVar$, i.e.,
$
d_{\setS}^{+}(\timeVar):=\sum_{j\in\setS}\ind\{\obsVar_j>\timeVar\}.
$

As with the analysis of the Kaplan-Meier estimator by \citet{foldes_1981}, we decompose the log of the \mbox{$k$-NN} estimate $\widehat{\survEnd}^{k\textsc{-NN}}(\timeVar|\fixedFeatureVector)$ into three terms with the help of a Taylor expansion. By Assumption A2, two deaths happen at the same time with probability~0, so
\begin{align*}
\widehat{\survEnd}^{k\textsc{-NN}}(\timeVar|\fixedFeatureVector)
 & =\prod_{i\in\neighborsKNN(\fixedFeatureVector)}\Big(\frac{d_{\neighborsKNN(\fixedFeatureVector)}^{+}(\obsVar_i)}{d_{\neighborsKNN(\fixedFeatureVector)}^{+}(\obsVar_i)+1}\Big)^{\eventVar_i\ind\{\obsVar_i\le \timeVar\}}.
\end{align*}
Taking the log of both sides, and noting that for any positive real number $z$, we have $\log(1+z)=\sum_{\ell=1}^{\infty}\frac{1}{\ell}(\frac{z}{z+1})^{\ell}$, we get
\begin{align}
 & \log\widehat{\survEnd}^{k\textsc{-NN}}(\timeVar|\fixedFeatureVector)\nonumber \\
 & =-\sum_{i\in\neighborsKNN(\fixedFeatureVector)}\!\!\!\!\eventVar_i\ind\{\obsVar_i\le \timeVar\}\log\Big(1+\frac{1}{d_{\neighborsKNN(\fixedFeatureVector)}^{+}(\obsVar_i)}\Big)\nonumber \\
 & =-\sum_{i\in\neighborsKNN(\fixedFeatureVector)}\!\!\!\!\eventVar_i\ind\{\obsVar_i\le \timeVar\}\sum_{\ell=1}^{\infty}\frac{1}{\ell(d_{\neighborsKNN(\fixedFeatureVector)}^{+}(\obsVar_i)+1)^{\ell}}. \nonumber \\
 & = U_1(\timeVar|\fixedFeatureVector)+U_2(\timeVar|\fixedFeatureVector)+U_3(\timeVar|\fixedFeatureVector), \label{eq:kNN-main-decomp}
\end{align}
where
\begin{align*}
U_1(\timeVar|\fixedFeatureVector) & =\frac{1}{k}\!\!\sum_{i\in\neighborsKNN(\fixedFeatureVector)}\!\!\!\!
-\frac{\eventVar_i\ind\{\obsVar_i\le \timeVar\}}{\obsEnd(\obsVar_i|\fixedFeatureVector)},\\
U_2(\timeVar|\fixedFeatureVector) & = - \!\!\!\!\!\!\!\!\sum_{i\in\neighborsKNN(\fixedFeatureVector)}\!\!
\frac{\eventVar_i\ind\{\obsVar_i\le \timeVar\}}{d_{\neighborsKNN(\fixedFeatureVector)}^{+}(\obsVar_i)+1} - U_1(t|x),\\
U_3(\timeVar|\fixedFeatureVector) & =-\!\!\!\!\!\!\!\!\sum_{i\in\neighborsKNN(\fixedFeatureVector)}\!\!\!\!\!\!
\eventVar_i\ind\{\obsVar_i\le \timeVar\}\!\sum_{\ell=2}^{\infty}\!\frac{1}{\ell(d_{\neighborsKNN(\fixedFeatureVector)}^{+}(\obsVar_i)+1)^{\ell}}.
\end{align*}
For large enough $k$ and $n$, it turns out that $U_1(\timeVar|\fixedFeatureVector)$ converges to $\log\survEnd(\timeVar|\fixedFeatureVector)$ while $U_2(\timeVar|\fixedFeatureVector)$ (first-order Taylor approximation error) and $U_3(\timeVar|\fixedFeatureVector)$ (sum of higher-order Taylor series terms) both go to 0.

The first term $U_1(\timeVar|\fixedFeatureVector)$ corresponds to a $k$-NN regression estimate that averages the ``label'' variable $\xi_i:=-\frac{\eventVar_i\ind\{\obsVar_i\le \timeVar\}}{\obsEnd(\obsVar_i|\fixedFeatureVector)}$ across the $k$ nearest neighbors. Note that the label variable $\xi_i$ perfectly knows the observed time $\obsVar$'s tail distribution $\obsEnd(\cdot|\fixedFeatureVector)$. Provided that the $k$ nearest neighbors have feature vectors within distance $h^*$ of $\fixedFeatureVector$, then it turns out that $\mathbb{E}[\xi_i]\approx\log\survEnd(t|x)$. Thus, having the nearest neighbors close to~$x$ aims to control the bias of the $k$-NN regression estimator~$U_1(\timeVar|\fixedFeatureVector)$.

To control the variance of regression estimator~$U_1(\timeVar|\fixedFeatureVector)$, i.e., for the $k$ labels being averaged to be close to its expectation, intuitively we want $k$ to be sufficiently large. However, how fast the average label converges to its expectation depends on whether the label variables $\xi_i$'s are correlated. The joint distribution of these $k$ label variables $\xi_i$'s is not straightforward to analyze. To circumvent this issue, we use a key proof technique by \citet{chaudhuri_dasgupta_2014}. Specifically, let~$\widetilde{X}$ denote the feature vector of the $(k+1)$-st nearest neighbor of~$\fixedFeatureVector$. Then conditioned on $\widetilde{X}$, the $k$ nearest neighbors' feature vectors appear as i.i.d.~samples from $\featureDist$ restricted to the open ball $\mathcal{B}_{\fixedFeatureVector,\rho(\fixedFeatureVector,\widetilde{X})}^o$. Thus, upon conditioning on $\widetilde{X}$, regression estimate $U_1(t|x)$ indeed becomes the average of $k$ label variables $\xi_i$'s that appear i.i.d., so Hoeffding's inequality tells us how fast their average converges to their expectation.

Since the regression estimate $U_1(\timeVar|\fixedFeatureVector)$ assumes perfect knowledge of the distribution of the $Y_i$'s (encoded in the tail probability $\obsEnd(\cdot|\fixedFeatureVector)$), unsurprisingly the first-order Taylor approximation error $U_2(\timeVar|\fixedFeatureVector)$ is about how well we can estimate $\obsEnd(\cdot|\fixedFeatureVector)$. In particular, it turns out that $|U_2(\timeVar|\fixedFeatureVector)|$ can be upper-bounded by how close the empirical distribution of the $k$ nearest neighbors' $\obsVar$ values is to the CDF $1-\obsEnd(\cdot|\fixedFeatureVector)$. Thus, the problem boils down to one of CDF estimation, for which there is once again a bias-variance sort of decomposition. The bias term is controlled by making sure that the $k$ nearest neighbors' feature vectors are within distance $h^*$ of $\fixedFeatureVector$. To control the variance, once again, we apply Chaudhuri and Dasgupta's proof technique of conditioning on the $(k+1)$-st nearest neighbor's feature vector $\widetilde{X}$. By doing this conditioning, the $k$ nearest neighbors' observed times $Y_i$'s become i.i.d., so the DKW inequality can be applied to bound the empirical distribution's deviation from its expectation.

In analyzing both $U_2(\timeVar|\fixedFeatureVector)$ and $U_3(\timeVar|\fixedFeatureVector)$, we remark that a key ingredient needed for our proof is that among the $k$ nearest neighbors, the number of them that survive beyond time $\timeHorizon$ (which is precisely $d_{\neighborsKNN(\fixedFeatureVector)}^{+}(\timeHorizon)$) is sufficiently large. In the equations for $U_2(\timeVar|\fixedFeatureVector)$ and $U_3(\timeVar|\fixedFeatureVector)$, note that $d_{\neighborsKNN(\fixedFeatureVector)}^{+}(\obsVar_i) \ge d_{\neighborsKNN(\fixedFeatureVector)}^{+}(\timeHorizon)$ whenever $\obsVar_i \le \timeHorizon$. Thus by making $d_{\neighborsKNN(\fixedFeatureVector)}^{+}(\timeHorizon)$ large, the denominator terms of $U_2(\timeVar|\fixedFeatureVector)$ and $U_3(\timeVar|\fixedFeatureVector)$ are becoming big. This shrinks $|U_3(\timeVar|\fixedFeatureVector)|$ to 0, and only partially helps in controlling $|U_2(\timeVar|\fixedFeatureVector)|$, with the CDF estimation discussion above fully bringing $|U_2(\timeVar|\fixedFeatureVector)|$ to 0.

\textbf{Relating to the Nelson-Aalen estimator.}
When there are no ties in survival and censoring times, the Nelson-Aalen estimator is given by $$\widehat{\Hazard}^{\textsc{NA}}(t) := \sum_{i=1}^n \frac{\eventVar_i\ind\{\obsVar_i\le\timeVar\}}{d_{[n]}^+(\obsVar_i)+1}.$$ Note that the first term in the definition of $U_2(\timeVar|\fixedFeatureVector)$ is precisely a negated version of a $k$-NN variant of the Nelson-Aalen estimator! By showing that $U_1(\timeVar|\fixedFeatureVector)+U_2(\timeVar|\fixedFeatureVector)$ converges to $\log\survEnd(\timeVar|\fixedFeatureVector)$, we can readily establish a nonasymptotic error bound for a \mbox{$k$-NN} Nelson-Aalen-based estimator for $\Hazard(\timeVar|\fixedFeatureVector):=-\log\survEnd(\timeVar|\fixedFeatureVector)$. We state guarantees for $k$-NN and kernel Nelson-Aalen-based estimators in Appendix~\ref{sec:Nelson-Aalen}.

\endgroup

\section{Proof of Theorem~\ref{thm:kNN-survival}}
\label{sec:pf-kNN-survival-lemmas}

To keep the exposition of the overall proof strategy clear, we defer proofs of supporting lemmas to the end of this section (in Appendices~\ref{sec:pf-lem-kNN-bad-T}--\ref{subsec:pf-lem-kNN-R3}). Much of the high-level proof structure is based on the nonasymptotic analysis of the Kaplan-Meier estimator by \citet{foldes_1981}. In addition to making changes to incorporate nearest neighbor analysis, we also make some technical changes to F\"{o}ldes and Rejt\"{o}'s proof, which we mention in Appendix~\ref{sec:technical-changes}.

Following our analysis outline of Section~\ref{sec:analysis-overview}, we denote~$\widetilde{\featureVar}$ to be the feature vector of the $(k+1)$-st nearest neighbor to $x$. We will be using this variable throughout this section.

As we discussed after the presentation of Theorem~\ref{thm:kNN-survival}, there are four key bad events. We now precisely state what these bad events are. For each bad event, we also show how to control its probability to be arbitrarily small.
After presenting these probability bounds, we explain why none of these bad events happening implies that $|\widehat{\survEnd}^{k\textsc{-NN}}(\timeVar|\fixedFeatureVector) - \survEnd(\timeVar|\fixedFeatureVector)|\le\varepsilon/3$. This factor of $1/3$ is important in the argument by \citet{foldes_1981} that translates an error guarantee for a fixed $\timeVar\in[0,\timeHorizon]$ to one that holds simultaneously across all $\timeVar\in[0,\timeHorizon]$, i.e., ${\sup_{\timeVar\in[0,\timeHorizon]}|\widehat{\survEnd}^{k\textsc{-NN}}(\timeVar|\fixedFeatureVector) - \survEnd(\timeVar|\fixedFeatureVector)|}\le\varepsilon$.

The first bad event is that not enough of the $k$ nearest neighbors survive beyond the time horizon $\timeHorizon$. Note that our convergence arguments for $U_2(\timeVar|\fixedFeatureVector)$ and $U_3(\timeVar|\fixedFeatureVector)$ later require that $d_{\neighborsKNN(\fixedFeatureVector)}^{+}(\timeHorizon)>k\theta/2$. Thus, our first bad event is
\[
\badEvent{k\textsc{-NN}}{\text{bad }\timeHorizon}(\fixedFeatureVector):=\{d_{\neighborsKNN(\fixedFeatureVector)}^{+}(\timeHorizon)\le k\theta/2\}.
\]
We control $\mathbb{P}(\badEvent{k\textsc{-NN}}{\text{bad }\timeHorizon}(\fixedFeatureVector))$ to be arbitrarily small by having the number of nearest neighbors~$k$ be sufficiently large, which in turn requires the number of training data $n\ge k$ to be sufficiently large.
\begin{lem}
\label{lem:kNN-bad-T}
Under Assumptions A1--A3, let $\fixedFeatureVector\in\text{supp}(\featureDist)$ and $\varepsilon\in(0,1)$. We have
\[
\mathbb{P}\big(\badEvent{k\textsc{-NN}}{\text{bad }\timeHorizon}(\fixedFeatureVector)\big)\le\exp\Big(-\frac{k\theta}{8}\Big).
\]
\end{lem}
Next, for the terms $U_1(\timeVar|\fixedFeatureVector)$ to converge to $\log \survEnd(\timeVar|\fixedFeatureVector)$ and $U_2(\timeVar|\fixedFeatureVector)$ to 0, we ask that the $k$ nearest neighbors found for $\fixedFeatureVector$ be within a critical distance $h^*$ that will depend on H\"{o}lder continuity constants of Assumption A4. This leads us to the next bad event:
\[
\badEvent{k\textsc{-NN}}{\text{far neighbors}}(\fixedFeatureVector):=\{\rho(\fixedFeatureVector,\widetilde{X})\ge h^*\}.
\]
Of course, if the $(k+1)$-st nearest neighbor is less than distance $h^*$ away from~$\fixedFeatureVector$, then so are the $k$ nearest neighbors. We control $\mathbb{P}(\badEvent{k\textsc{-NN}}{\text{far neighbors}}(\fixedFeatureVector))$ to be arbtrarily small by making the number of training subjects $n$ sufficiently large. By sampling more training data, the $k$ nearest neighbors found for $\fixedFeatureVector$ will gradually get closer to $\fixedFeatureVector$.
\begin{lem}[\citealt{chaudhuri_dasgupta_2014}, Lemma~9]
\label{lem:chaudhuri_dasgupta_far_neighbors}
Under Assumption A1, if
$k\le\frac{1}{2}n\featureDist(\mathcal{B}_{\fixedFeatureVector,h^*})$, then
\[
\mathbb{P}(\badEvent{k\textsc{-NN}}{\text{far neighbors}}(\fixedFeatureVector))\le\exp\Big(-\frac{n\featureDist(\mathcal{B}_{\fixedFeatureVector,h^*})}{8}\Big).
\]
\end{lem}
This lemma holds for any choice of distance $h^*>0$ although for our analysis, we will choose $h^*=(\frac{\varepsilon\theta}{18\Lips})^{1/\holderIndex}$. This particular choice of $h^*$ is explained later on in Lemmas~\ref{lem:kNN-R1-bias} and~\ref{lem:kNN-R2-bias}.

To get to our next bad event, we first relate $U_2(\timeVar|\fixedFeatureVector)$ to a CDF estimate. Specifically, the function
\[
\estObsEnd^{k\textsc{-NN}}(s|\fixedFeatureVector):=\frac{d_{\neighborsKNN(\fixedFeatureVector)}^{+}(s)}{k}=\frac{1}{k}\sum_{i\in\neighborsKNN(\fixedFeatureVector)}\ind\{\obsVar_i> s\}
\]
is one minus an empirical distribution function. The next lemma bounds $|U_2(\timeVar|\fixedFeatureVector)|$ in terms of $\estObsEnd^{k\textsc{-NN}}$.
\begin{lem}
\label{lem:kNN-R2-bound}
Under Assumptions A1--A3, let $\fixedFeatureVector\in\text{supp}(\featureDist)$ and $\timeVar\in[0,\timeHorizon]$. When event $\badEvent{k\textsc{-NN}}{\text{bad }\timeHorizon}(\fixedFeatureVector)$ does not happen,
\begin{align}
& |U_2(\timeVar|\fixedFeatureVector)|\nonumber \\
& \quad \le\frac{2}{k\theta^{2}}+\frac{2}{\theta^{2}}\sup_{s\in[0,\timeHorizon]}|\obsEnd(s|\fixedFeatureVector)-\mathbb{E}[\estObsEnd^{k\textsc{-NN}}(s|\fixedFeatureVector)|\widetilde{\featureVar}]|\nonumber \\
& \quad \quad+\frac{2}{\theta^{2}}\sup_{s\ge0}|\estObsEnd^{k\textsc{-NN}}(s|\fixedFeatureVector)-\mathbb{E}[\estObsEnd^{k\textsc{-NN}}(s|\fixedFeatureVector)|\widetilde{\featureVar}]|.\label{eq:kNN-R2-bound}
\end{align}
\end{lem}
The third bad event corresponds to the empirical distribution function being too far from its expectation:
\begin{align*}
 & \badEvent{k\textsc{-NN}}{\text{bad EDF}}(\fixedFeatureVector)\\
 & \;:=\Big\{\sup_{s\ge0}|\estObsEnd^{k\textsc{-NN}}(s|\fixedFeatureVector)-\mathbb{E}[\estObsEnd^{k\textsc{-NN}}(s|\fixedFeatureVector)|\widetilde{\featureVar}]|>\frac{\varepsilon\theta^{2}}{36}\Big\},
\end{align*}
where importantly the expectation is, as with handling $U_1(\timeVar|\fixedFeatureVector)$, a function of the $(k+1)$-st nearest neighbor $\widetilde{X}$. We control $\mathbb{P}(\badEvent{k\textsc{-NN}}{\text{bad EDF}}(\fixedFeatureVector))$ to be arbitrarily small by making the number of nearest neighbors~$k$ sufficiently large. The rate of convergence for the empirical distribution function is given by the DKW inequality.
\begin{lem}
\label{lem:kNN-bad-EDF}
Under Assumptions A1--A3, for any $\fixedFeatureVector\in\text{supp}(\featureDist)$,
\[
\mathbb{P}\big(\badEvent{k\textsc{-NN}}{\text{bad EDF}}(\fixedFeatureVector)\big)\le2\exp\Big(-\frac{k\varepsilon^{2}\theta^{4}}{648}\Big).
\]
\end{lem}
The last bad event is that $U_1(\timeVar|\fixedFeatureVector)$ is not close to its expectation $\mathbb{E}[U_1(\timeVar|\fixedFeatureVector)|\widetilde{\featureVar}]$:
\[
\badEvent{k\textsc{-NN}}{\text{bad }U_1}(t,x):=\{|U_1(\timeVar|\fixedFeatureVector)-\mathbb{E}[U_1(\timeVar|\fixedFeatureVector)|\widetilde{\featureVar}]|\ge\varepsilon/18\}.
\]
We control $\mathbb{P}(\badEvent{k\textsc{-NN}}{\text{bad }U_1}(t,x))$ to be small by making the number of nearest neighbors~$k$ is sufficiently large.
\begin{lem}
\label{lem:kNN-bad-R1}
Under Assumptions A1--A3, let $\fixedFeatureVector\in\text{supp}(\featureDist)$ and $\timeVar\in[0,\timeHorizon]$. Then
\[
\mathbb{P}(\badEvent{k\textsc{-NN}}{\text{bad }U_1}(t,x))\le2\exp\Big(-\frac{k\varepsilon^{2}\theta^{2}}{162}\Big).
\]
\end{lem}
At this point, we have collected all four main bad events. When none of these bad events happen, then starting from equation~\eqref{eq:kNN-main-decomp}, applying the triangle inequality a few times, and using inequality~\eqref{eq:kNN-R2-bound}, we get
\begin{align}
 & |\log\widehat{\survEnd}^{k\textsc{-NN}}(\timeVar|\fixedFeatureVector)-\log \survEnd(\timeVar|\fixedFeatureVector)|\nonumber \\
 & =|U_1(\timeVar|\fixedFeatureVector)-\log \survEnd(\timeVar|\fixedFeatureVector)+U_2(\timeVar|\fixedFeatureVector)+U_3(\timeVar|\fixedFeatureVector)|\nonumber \\
 & \le|U_1(\timeVar|\fixedFeatureVector)-\log \survEnd(\timeVar|\fixedFeatureVector)|+|U_2(\timeVar|\fixedFeatureVector)|+|U_3(\timeVar|\fixedFeatureVector)|\nonumber \\
 & \le |U_1(\timeVar|\fixedFeatureVector)-\mathbb{E}[U_1(\timeVar|\fixedFeatureVector)|\widetilde{\featureVar}]| \nonumber \\
 & \quad+|\mathbb{E}[U_1(\timeVar|\fixedFeatureVector)|\widetilde{\featureVar}]-\log \survEnd(\timeVar|\fixedFeatureVector)| +|U_2(\timeVar|\fixedFeatureVector)|+|U_3(\timeVar|\fixedFeatureVector)| \nonumber\\
 & \le|U_1(\timeVar|\fixedFeatureVector)-\mathbb{E}[U_1(\timeVar|\fixedFeatureVector)|\widetilde{\featureVar}]|\nonumber\\
 & \quad+|\mathbb{E}[U_1(\timeVar|\fixedFeatureVector)|\widetilde{\featureVar}]-\log \survEnd(\timeVar|\fixedFeatureVector)|\nonumber\\
 & \quad+\frac{2}{k\theta^{2}}+\frac{2}{\theta^{2}}\sup_{s\in[0,\timeHorizon]}|\obsEnd(s|\fixedFeatureVector)-\mathbb{E}[\estObsEnd^{k\textsc{-NN}}(s|\fixedFeatureVector)|\widetilde{\featureVar}]|\nonumber\\
 & \quad+\frac{2}{\theta^{2}}\sup_{s\ge0}|\estObsEnd^{k\textsc{-NN}}(s|\fixedFeatureVector)-\mathbb{E}[\estObsEnd^{k\textsc{-NN}}(s|\fixedFeatureVector)|\widetilde{\featureVar}]|\nonumber\\
 & \quad+|U_3(\timeVar|\fixedFeatureVector)|. \label{eq:kNN-6-term-bound}
\end{align}
We show that the RHS is at most $\varepsilon/3$ by ensuring that each of its six terms is at most $\varepsilon/18$. The 1st and 5th terms are at most $\varepsilon/18$ since bad events $\badEvent{k\textsc{-NN}}{\text{bad }U_1}(t,x)$ and $\badEvent{k\textsc{-NN}}{\text{bad EDF}}(\fixedFeatureVector)$ do not happen. The 3rd term is at most $\varepsilon/18$ by recalling that the theorem assumes $k\ge\frac{72}{\varepsilon\theta^2}$, so
$
\frac{2}{k\theta^2}
\le\frac{2}{(\frac{72}{\varepsilon\theta^2})\theta^2}
=\frac{\varepsilon}{36}
<\frac{\varepsilon}{18}
$.

The 2nd, 4th, and 6th RHS terms of inequality~\eqref{eq:kNN-6-term-bound} remain to be bounded. We tackle these in the next three lemmas. Note that these lemmas are deterministic. The first two lemmas ask that the $k$ nearest neighbors be sufficiently close to $\fixedFeatureVector$ and make use of H\"{o}lder continuity; these lemmas explain why critical distance $h^*=(\frac{\varepsilon\theta}{18\Lips})^{1/\holderIndex}$ and why $\Lips$ is defined the way it is.
\begin{lem}
\label{lem:kNN-R1-bias}
Under Assumptions A1--A4 $($this lemma uses H\"{o}lder continuity of $\censEnd(\timeVar|\cdot)\survDensity(\timeVar|\cdot))$, let $\fixedFeatureVector\in\text{supp}(\featureDist)$, $\timeVar\in[0,\timeHorizon]$, and $\varepsilon\in(0,1)$. If bad event $\badEvent{k\textsc{-NN}}{\text{far neighbors}}(\fixedFeatureVector)$ does not happen, and $h^*\le[\frac{\varepsilon\theta}{18( \lips_{\textsc{\tiny{\survVar}}}\timeHorizon+(\survDensity^*\lips_{\textsc{\tiny{\censVar}}}\timeHorizon^2)/2 )}]^{1/\holderIndex}$, then
\[
|\mathbb{E}[U_1(\timeVar|\fixedFeatureVector)|\widetilde{\featureVar}]-\log \survEnd(\timeVar|\fixedFeatureVector)|\le\frac{\varepsilon}{18}.
\]
\end{lem}
\begin{lem}
\label{lem:kNN-R2-bias}
Under Assumptions A1--A4 $($this lemma uses H\"{o}lder continuity of $\obsEnd(\timeVar|\cdot))$, let $\fixedFeatureVector\in\text{supp}(\featureDist)$ and $\varepsilon\in(0,1)$. If bad event $\badEvent{k\textsc{-NN}}{\text{far neighbors}}(\fixedFeatureVector)$ does not happen, and $h^*\le[\frac{\varepsilon\theta^{2}}{36(\lips_{\textsc{\tiny{\survVar}}}+\lips_{\textsc{\tiny{\censVar}}})\timeHorizon}]^{1/\holderIndex}$, then
\[
\frac{2}{\theta^{2}}\sup_{s\in[0,\timeHorizon]}|\obsEnd(s|\fixedFeatureVector)-\mathbb{E}[\estObsEnd^{k\textsc{-NN}}(s|\fixedFeatureVector)|\widetilde{\featureVar}]|\le\frac{\varepsilon}{18}.
\]
\end{lem}
\begin{lem}
\label{lem:kNN-R3}
Under Assumptions A1--A3, let $\fixedFeatureVector\in\text{supp}(\featureDist)$, $\timeVar\in[0,\timeHorizon]$, and $\varepsilon\in(0,1)$. If bad event $\badEvent{k\textsc{-NN}}{\text{bad }\timeHorizon}(\fixedFeatureVector)$ does not happen, and $k\ge\frac{72}{\varepsilon\theta^{2}}$, then $|U_3(\timeVar|\fixedFeatureVector)|\le {\varepsilon}/{18}$.
\end{lem}
Putting together the pieces so far, provided that all the bad events do not happen, then we have bounded all six RHS terms of inequality~\eqref{eq:kNN-6-term-bound} by $\varepsilon/18$:
\[
|\log\widehat{\survEnd}^{k\textsc{-NN}}(\timeVar|\fixedFeatureVector)-\log \survEnd(\timeVar|\fixedFeatureVector)|\le6\cdot\frac{\varepsilon}{18}=\frac{\varepsilon}{3}.
\]
For any $a,b\in(0,1]$, we have $|a-b|\le|\log a-\log b|$, so the above inequality implies that we also have
\[
|\widehat{\survEnd}^{k\textsc{-NN}}(\timeVar|\fixedFeatureVector)-\survEnd(\timeVar|\fixedFeatureVector)|\le\frac{\varepsilon}{3}.
\]
To establish Theorem \ref{thm:kNN-survival}, we need to guarantee that $\sup_{\timeVar\in[0,\timeHorizon]}|\widehat{\survEnd}^{k\textsc{-NN}}(\timeVar|\fixedFeatureVector)-\survEnd(\timeVar|\fixedFeatureVector)|\le\varepsilon$. A sufficient condition that accomplishes this task is to ask that ${|\widehat{\survEnd}^{k\textsc{-NN}}(\timeVar|\fixedFeatureVector)-\survEnd(\timeVar|\fixedFeatureVector)|}\le\varepsilon/3$ for a finite collection of times $\timeVar$ within the interval $[0,\timeHorizon]$. Specifically, we partition the interval $[0,\timeHorizon]$ into $L(\varepsilon)$ pieces such that $0=\eta_{0}<\eta_{1}<\cdots<\eta_{L(\varepsilon)}=\timeHorizon$, where:
\begin{itemize}[leftmargin=1.5em,topsep=0pt,itemsep=0ex,partopsep=1ex,parsep=1ex]
\item $\survEnd(\eta_{j-1}|\fixedFeatureVector)-\survEnd(\eta_j|\fixedFeatureVector)\le\varepsilon/3$ for $j=1,\dots,L(\varepsilon)$,
\item $L(\varepsilon)\le4/\varepsilon$.
\end{itemize}
We can always produce a partition satisfying the above conditions because the most $\survEnd$ can change from~0 to~$\timeHorizon$ is by a value of~1 ($\survEnd$ is one minus a CDF and is continuous). In this worst case scenario of $\survEnd$ changing by~1, by placing the points $\eta_j$'s at times where $\survEnd$ drops by exactly $\varepsilon/3$ in value (except across the last piece $[\eta_{L(\varepsilon)-1}, \eta_{L(\varepsilon)}]$, where $\survEnd$ could drop by less than $\varepsilon/3$), then $L(\varepsilon)=\lceil\frac{1}{\varepsilon/3}\rceil=\lceil3/\varepsilon\rceil\le4/\varepsilon$ where the last inequality holds for $\varepsilon\in(0,1]$. When $\survEnd$ changes by less than 1, $L(\varepsilon)$ could be smaller.

We shall ask that $|\widehat{\survEnd}(\eta_j|\fixedFeatureVector)-\survEnd(\eta_j|\fixedFeatureVector)|\le\varepsilon/3$ for each $j=1,2,\dots,L(\varepsilon)$. Note that $\widehat{\survEnd}(\cdot|\fixedFeatureVector)$ is piecewise constant and monotonically decreasing. Moreover, $\widehat{\survEnd}(0|\fixedFeatureVector)=\survEnd(0|\fixedFeatureVector)=1$ (the probability of a death happening at $t=0$ is 0). Thus, by having $\widehat{\survEnd}(\cdot|\fixedFeatureVector)$ differ from $\survEnd(\cdot|\fixedFeatureVector)$ by at most $\varepsilon/3$ at each $\eta_j$ for $j=1,\dots,L(\varepsilon)$, we are guaranteed that $|\widehat{\survEnd}(\timeVar|\fixedFeatureVector)-\survEnd(\timeVar|\fixedFeatureVector)|\le\varepsilon$ for any time $\timeVar\in[0,\timeHorizon]$. In summary, here are all the bad events of interest:
\begin{itemize}[leftmargin=1.5em,topsep=0pt,itemsep=0ex,partopsep=1ex,parsep=1ex]
\item $\badEvent{k\textsc{-NN}}{\text{bad }\timeHorizon}(\fixedFeatureVector)$
\item $\badEvent{k\textsc{-NN}}{\text{far neighbors}}(\fixedFeatureVector)$
\item $\badEvent{k\textsc{-NN}}{\text{bad EDF}}(\fixedFeatureVector)$
\item $\badEvent{k\textsc{-NN}}{\text{bad }U_1}(\timeVar,\fixedFeatureVector)$ for $\timeVar=\eta_{1},\eta_{2},\dots,\eta_{L(\varepsilon)}$
\end{itemize}
The lemmas require $\frac{72}{\varepsilon\theta^{2}}\le k\le\frac{1}{2}n\featureDist(\mathcal{B}_{\fixedFeatureVector,h^*})$, and $h^*\le[\min\big\{\frac{\varepsilon\theta}{18( \lips_{\textsc{\tiny{\survVar}}}\timeHorizon+(\survDensity^*\lips_{\textsc{\tiny{\censVar}}}\timeHorizon^2)/2 )},\,\frac{\varepsilon\theta^{2}}{36(\lips_{\textsc{\tiny{\survVar}}}+\lips_{\textsc{\tiny{\censVar}}})\timeHorizon}\big\}]^{1/\holderIndex}$.
Union bounding over all the bad events,
\begin{align*}
 & \mathbb{P}(\text{at least one bad event happens})\\
 & \quad \le\mathbb{P}\big(\badEvent{k\textsc{-NN}}{\text{bad }\timeHorizon}(\fixedFeatureVector)\big)+\mathbb{P}\big(\badEvent{k\textsc{-NN}}{\text{far neighbors}}(\fixedFeatureVector)\big)\\
 & \quad \quad+\mathbb{P}\big(\badEvent{k\textsc{-NN}}{\text{bad EDF}}(\fixedFeatureVector)\big)+\sum_{\ell=1}^{L(\varepsilon)}\mathbb{P}\big(\badEvent{k\textsc{-NN}}{\text{bad }U_1}(\eta_{\ell},\fixedFeatureVector)\big)\\
 & \quad \le\exp\Big(-\frac{k\theta}{8}\Big)+\exp\Big(-\frac{n\featureDist(\mathcal{B}_{\fixedFeatureVector,h^*})}{8}\Big)\\
 & \quad \quad+2\exp\Big(-\frac{k\varepsilon^{2}\theta^{4}}{648}\Big)+\frac{8}{\varepsilon}\exp\Big(-\frac{k\varepsilon^{2}\theta^{2}}{162}\Big). \tag*{$\square$}
\end{align*}

\subsection{Proof of Lemma \ref{lem:kNN-bad-T}}
\label{sec:pf-lem-kNN-bad-T}

The key idea is that regardless of where each nearest neighbor $x'\in\neighborsKNN(\fixedFeatureVector)$ lands in feature space $\featureSpace$, the probability that its observed time (the corresponding $\obsVar$ variable) exceeds $\timeHorizon$ is $\obsEnd(\timeHorizon|\fixedFeatureVector')\ge\theta$ (Assumption A3). This means that $d_{\neighborsKNN(\fixedFeatureVector)}^{+}(\timeHorizon)$ stochastically dominates a $\text{Binomial}(k,\theta)$ random variable. Hence,
\begin{align*}
\mathbb{P}(\badEvent{k\textsc{-NN}}{\text{bad }\timeHorizon}(\fixedFeatureVector)) & =\mathbb{P}\Big(d_{\neighborsKNN(\fixedFeatureVector)}^{+}(\timeHorizon)\le\frac{k\theta}{2}\Big)\\
 & \le\mathbb{P}\Big(\text{Binomial}(k,\theta)\le\frac{k\theta}{2}\Big)\\
 & \le\exp\Big(-\frac{1}{2\theta}\cdot\frac{(k\theta-\frac{k\theta}{2})^{2}}{k}\Big)\\
 & =\exp\Big(-\frac{k\theta}{8}\Big),
\end{align*}
where the second inequality uses a Chernoff bound for the binomial distribution. Note that the version of the Chernoff bound we use is the one in Section 2.1 of~\citet{georgehc_thesis}.\hfill$\square$

\subsection{Proof of Lemma \ref{lem:chaudhuri_dasgupta_far_neighbors}}

This proof is by \citet[Lemma~9]{chaudhuri_dasgupta_2014}. Let $\fixedFeatureVector\in\text{supp}(\featureDist)$ and $h^*>0$. Let~$\widetilde{X}$ denote the $(k+1)$-st nearest neighbor of $\fixedFeatureVector$, and $N_{x,h^*}\sim\text{Binomial}(n,\featureDist(\mathcal{B}_{\fixedFeatureVector,h^*}))$ denote the number of training data that land within distance~$h^*$ of~$\fixedFeatureVector$. Note that $\rho(\fixedFeatureVector, \widetilde{X})\ge h^*$ implies that $N_{x,h^*}\le k$. Therefore, with the help of a Chernoff bound for the binomial distribution \citep[Section 2.1]{georgehc_thesis} (with the assumption $1\le k\le\frac12 n\featureDist(\mathcal{B}_{\fixedFeatureVector,h^*})$),
\begin{align*}
& \mathbb{P}(\rho(\fixedFeatureVector, \widetilde{X})\ge h^*) \\
& \quad \le \mathbb{P}(N_{x,h^*}\le k) \\
& \quad \le \exp\Big(-\frac{(n\featureDist(\mathcal{B}_{\fixedFeatureVector,h^*}) - k)^2}{2n\featureDist(\mathcal{B}_{\fixedFeatureVector,h^*})}\Big) \\
& \quad \le \exp\Big(-\frac{(n\featureDist(\mathcal{B}_{\fixedFeatureVector,h^*}) - \frac12 n\featureDist(\mathcal{B}_{\fixedFeatureVector,h^*}))^2}{2n\featureDist(\mathcal{B}_{\fixedFeatureVector,h^*})}\Big) \\
& \quad = \exp\Big(-\frac{n\featureDist(\mathcal{B}_{\fixedFeatureVector,h^*})}{8}\Big). \tag*{$\square$}
\end{align*}

\subsection{Proof of Lemma \ref{lem:kNN-R2-bound}\label{subsec:pf-lem-kNN-R2-bound}}

We abbreviate the set of $k$ nearest training subjects $\neighborsKNN(\fixedFeatureVector)$ as the set $\setS$. We frequently use the fact that the function $d_{\setS}^{+}$ monotonically decreases. Provided that bad event $\badEvent{k\textsc{-NN}}{\text{bad }\timeHorizon}(\fixedFeatureVector)$ does not happen, then we have $d_{\setS}^{+}(\timeVar)>k\theta/2$ for all $\timeVar\in[0,\timeHorizon]$. Then
\begin{align*}
 & |U_2(\timeVar|\fixedFeatureVector)|\\
 & \; =\bigg|\frac{1}{k}\sum_{i\in\setS}\eventVar_i\ind\{\obsVar_i\le \timeVar\}\Big[\frac{k}{d_{\setS}^{+}(\obsVar_i)+1}-\frac{1}{\obsEnd(\obsVar_i|\fixedFeatureVector)}\Big]\bigg|\\
 & \; \le\frac{1}{k}\sum_{i\in\setS}\eventVar_i\ind\{\obsVar_i\le \timeVar\}\Big|\frac{k}{d_{\setS}^{+}(\obsVar_i)+1}-\frac{1}{\obsEnd(\obsVar_i|\fixedFeatureVector)}\Big|\\
 & \; \le\frac{1}{k}\sum_{i\in\setS}\eventVar_i\ind\{\obsVar_i\le \timeVar\}\sup_{s\in[0,\timeHorizon]}\Big|\frac{k}{d_{\setS}^{+}(s)+1}-\frac{1}{\obsEnd(s|\fixedFeatureVector)}\Big|\\
 & \; \le\sup_{s\in[0,\timeHorizon]}\Big|\frac{k}{d_{\setS}^{+}(s)+1}-\frac{1}{\obsEnd(s|\fixedFeatureVector)}\Big|\\
 & \; = \sup_{s\in[0,\timeHorizon]}\Big|\frac{k \obsEnd(s|\fixedFeatureVector) - d_{\setS}^{+}(s) - 1}{(d_{\setS}^{+}(s)+1)\obsEnd(s|\fixedFeatureVector)}\Big|\\
 & \; \le\frac{k}{(d_{\setS}^{+}(\timeHorizon)+1)\obsEnd(\timeHorizon|\fixedFeatureVector)}\sup_{s\in[0,\timeHorizon]}\Big|\obsEnd(s|\fixedFeatureVector)-\frac{d_{\setS}^{+}(s)}{k}-\frac{1}{k}\Big|\\
 & \; \le\frac{k}{d_{\setS}^{+}(\timeHorizon)\theta}\sup_{s\in[0,\timeHorizon]}\Big|\obsEnd(s|\fixedFeatureVector)-\frac{d_{\setS}^{+}(s)}{k}-\frac{1}{k}\Big|\\
 & \; <\frac{2}{k\theta}\cdot\frac{k}{\theta}\sup_{s\in[0,\timeHorizon]}\Big|\obsEnd(s|\fixedFeatureVector)-\frac{d_{\setS}^{+}(s)}{k}-\frac{1}{k}\Big|\\
 & \; =\frac{2}{\theta^{2}}\sup_{s\in[0,\timeHorizon]}\Big|\obsEnd(s|\fixedFeatureVector)-\frac{d_{\setS}^{+}(s)}{k}-\frac{1}{k}\Big|\\
 & \; \le\frac{2}{\theta^{2}}\Big(\frac{1}{k}+\sup_{s\in[0,\timeHorizon]}\Big|\obsEnd(s|\fixedFeatureVector)-\frac{d_{\setS}^{+}(s)}{k}\Big|\Big).
\end{align*}
Using abbreviation $\estObsEnd(s):=\estObsEnd^{k\textsc{-NN}}(s|\fixedFeatureVector)={d_{\setS}^{+}(s)}/{k}$,
\begin{align*}
 & \sup_{s\in[0,\timeHorizon]}|\obsEnd(s|\fixedFeatureVector)-\estObsEnd(s)| \nonumber\\
 & \le\!\!\sup_{s\in[0,\timeHorizon]}\!\!|\obsEnd(s|\fixedFeatureVector)\!-\!\mathbb{E}[\estObsEnd(s)|\widetilde{\featureVar}]| \!+\sup_{s\ge0}|\estObsEnd(s)\!-\!\mathbb{E}[\estObsEnd(s)|\widetilde{\featureVar}]|.
\end{align*}
Putting together the two inequalities above,
\begin{align*}
 |U_2(\timeVar|\fixedFeatureVector)|
 & \le\frac{2}{k\theta^{2}}+\frac{2}{\theta^{2}}\sup_{s\in[0,\timeHorizon]}|\obsEnd(s|\fixedFeatureVector)-\mathbb{E}[\estObsEnd(s)|\widetilde{\featureVar}]|\\
 & \quad+\frac{2}{\theta^{2}}\sup_{s\ge0}|\estObsEnd(s)-\mathbb{E}[\estObsEnd(s)|\widetilde{\featureVar}]|. \tag*{$\square$}
\end{align*}

\subsection{Proof of Lemma \ref{lem:kNN-bad-EDF}}

This proof technique is from \citet[Lemma 10]{chaudhuri_dasgupta_2014}, modified to handle the survival analysis setup. The randomness can be described as follows:
\begin{enumerate}[leftmargin=*,topsep=0pt,itemsep=0ex,partopsep=1ex,parsep=1ex]

\item Sample a feature vector $\widetilde{X}\in\featureSpace$ from the marginal distribution of the $(k+1$)-st nearest neighbor of~$\fixedFeatureVector$.

\item Sample $k$ feature vectors i.i.d.~from $\featureDist$ conditioned on landing in the ball $\mathcal{B}_{\fixedFeatureVector,\rho(\fixedFeatureVector,\widetilde{X})}^o$.

\item Sample $n-k-1$ feature vectors i.i.d.~from $\featureDist$ conditioned on landing in $\featureSpace\setminus\mathcal{B}_{\fixedFeatureVector,\rho(\fixedFeatureVector,\widetilde{X})}^o$.

\item Randomly permute the $n$ feature vectors sampled.

\item For each feature vector $\featureVar_i$, sample its corresponding observed time $\obsVar_i$ and censoring indicator $\eventVar_i$.

\end{enumerate}
As a technical remark, the above description of randomness requires Assumption A1 to hold in addition to using randomized tie breaking when finding the $k$ nearest neighbors. Moreover, to incorporate this tie breaking into the theory, the definition of the open ball needs to be changed slightly, upon which the proof strategy still carries through. For details, see Section 2.7 in the Appendix of \citet{chaudhuri_dasgupta_2014}.

The points sampled in step 2 are precisely the $k$ nearest neighbors of~$\fixedFeatureVector$. Thus, using the $\obsVar_i$ variables corresponding specifically to the feature vectors generated in step 2 (let's call these $k$ variables $\obsVar_{(1)},\dots,\obsVar_{(k)}$), construct the function $\Psi_{s}(\widetilde{X}):=\frac{1}{k}\sum_{\ell=1}^{k}\ind\{\obsVar_{(\ell)}>s\}$.

\allowdisplaybreaks

Note that $\estObsEnd^{k\textsc{-NN}}(s|\fixedFeatureVector)=\Psi_{s}(\widetilde{X})$, and after conditioning on $\widetilde{X}$, empirical distribution function $1-\Psi_{s}(\widetilde{X})$ is constructed from i.i.d.~samples from the CDF
\[
1-\mathbb{E}[\Psi(s)\,|\,\widetilde{X}]=1-\underbrace{\mathbb{P}(\obsVar>s\,|\,\featureVar\in\mathcal{B}_{\fixedFeatureVector,\rho(\fixedFeatureVector,\widetilde{X})}^o)}_{:=\overline{\Psi}(\widetilde{X})}.
\]
Letting $\mathbb{P}_{\widetilde{X}}$ refer to the marginal distribution
of~$\widetilde{X}$ (from step 1 of the procedure above), then by the DKW inequality,
\begin{align*}
 & \mathbb{P}\Big(\sup_{s\ge0}|\estObsEnd^{k\textsc{-NN}}(s|\fixedFeatureVector)-\mathbb{E}[\estObsEnd^{k\textsc{-NN}}(s|\fixedFeatureVector)|\widetilde{\featureVar}]|>\frac{\varepsilon\theta^{2}}{36}\Big)\\
 & =\mathbb{P}\Big(\sup_{s\ge0}|\Psi_{s}(\widetilde{X})-\overline{\Psi}(\widetilde{X})|>\frac{\varepsilon\theta^{2}}{36}\Big)\\
 & =\int_{\featureSpace}\mathbb{P}\Big(\sup_{s\ge0}|\Psi_{s}(\widetilde{X})-\overline{\Psi}(\widetilde{X})|>\frac{\varepsilon\theta^{2}}{36}\,\Big|\,\widetilde{X}=\widetilde{x}\Big)d\mathbb{P}_{\widetilde{X}}(\widetilde{x})\\
 & \le\int_{\featureSpace}2\exp\Big(-\frac{k\varepsilon^{2}\theta^{4}}{648}\Big)d\mathbb{P}_{\widetilde{X}}(\widetilde{x})\\
 & =2\exp\Big(-\frac{k\varepsilon^{2}\theta^{4}}{648}\Big). \tag*{$\square$}
\end{align*}

\subsection{Proof of Lemma \ref{lem:kNN-bad-R1}}

Again, we use the proof technique by \citet[Lemma 10]{chaudhuri_dasgupta_2014}, slightly modified. The randomness can be described as follows:
\begin{enumerate}[leftmargin=*,topsep=0pt,itemsep=0ex,partopsep=1ex,parsep=1ex]

\item Sample a feature vector $\widetilde{X}\in\featureSpace$ from the marginal distribution of the $(k+1$)-st nearest neighbor of~$\fixedFeatureVector$.

\item Sample $k$ feature vectors i.i.d.~from $\featureDist$ conditioned on landing in the ball $\mathcal{B}_{\fixedFeatureVector,\rho(\fixedFeatureVector,\widetilde{X})}^o$.

\item Sample $n-k-1$ feature vectors i.i.d.~from $\featureDist$ conditioned on landing in $\featureSpace\setminus\mathcal{B}_{\fixedFeatureVector,\rho(\fixedFeatureVector,\widetilde{X})}^o$.

\item Randomly permute the $n$ feature vectors sampled.

\item For each feature vector $\featureVar_i$, sample its corresponding observed time $\obsVar_i$ and censoring indicator $\eventVar_i$.

\item Let $\xi_i=-\frac{\eventVar_i\ind\{\obsVar_i\le \timeVar\}}{\obsEnd(\obsVar_i|\fixedFeatureVector)}$ for each $i$.

\end{enumerate}
The points sampled in step 2 are the $k$ nearest neighbors of $\fixedFeatureVector$. In particular, $U_1(\timeVar|\fixedFeatureVector)$ is the average of $k$ terms that become i.i.d.~after we condition on the $(k+1$)-st nearest neighbor $\widetilde{X}$:
\[
U_1(\timeVar|\fixedFeatureVector)=\frac{1}{k}\sum_{\ell=1}^{k}\xi_{\ell}(\widetilde{X}),
\]
where $\xi_{\ell}(\widetilde{X})$ is the $\xi_i$ variable corresponding to one of the feature vectors drawn in step 2 (which depends on $\widetilde{X}$). Each $\xi_{\ell}(\widetilde{X})$ has expectation
\begin{align}
\overline{\xi}(\widetilde{X})
&:=\mathbb{E}_{\obsVar,\eventVar}\Big[-\frac{\eventVar\ind\{\obsVar\le \timeVar\}}{\obsEnd(\obsVar|\fixedFeatureVector)}\,\Big|\,\featureVar\in\mathcal{B}_{\fixedFeatureVector,\rho(\fixedFeatureVector,\widetilde{X})}^o\Big] \nonumber \\
&\hspace{.25em}=\mathbb{E}[U_1(\timeVar|\fixedFeatureVector)|\widetilde{\featureVar}],
\label{eq:Ek-U1}
\end{align}
which is a function of random variable $\widetilde{X}$. Moreover, each $\xi_{\ell}(\widetilde{X})$ is bounded in $[-\frac{1}{\obsEnd(\timeVar|\fixedFeatureVector)},0]$ (note that Assumption A2 ensures that $\eventVar$ in the expectation is not almost surely 0).

Letting $\mathbb{P}_{\widetilde{X}}$ refer to the marginal distribution of the $(k+1)$-st nearest neighbor (from step 1 of the procedure above), then by Hoeffding's inequality, $\obsEnd(\cdot|\fixedFeatureVector)$ monotonically decreasing, and Assumption A3,
\begin{align*}
 & \mathbb{P}\Big(|U_1(\timeVar|\fixedFeatureVector)-\mathbb{E}[U_1(\timeVar|\fixedFeatureVector)|\widetilde{\featureVar}]|\ge\frac{\varepsilon}{12}\Big)\\
 & =\mathbb{P}\Big(\Big|\frac{1}{k}\sum_{\ell=1}^{k}\xi_{\ell}(\widetilde{X})-\overline{\xi}(\widetilde{X})\Big|\ge\frac{\varepsilon}{12}\Big)\\
 & =\int_{\featureSpace}\mathbb{P}\Big(\Big|\frac{1}{k}\sum_{\ell=1}^{k}\xi_{\ell}(\widetilde{X})-\overline{\xi}(\widetilde{X})\Big|\ge\frac{\varepsilon}{18}\,\Big|\,\widetilde{X}=\widetilde{x}\Big)d\mathbb{P}_{\widetilde{X}}(\widetilde{x})\\
 & \le\int_{\featureSpace}2\exp\Big(-\frac{k\varepsilon^{2}[\obsEnd(\timeVar|\fixedFeatureVector)]^2}{162}\Big)d\mathbb{P}_{\widetilde{X}}(\widetilde{x})\\
 & =2\exp\Big(-\frac{k\varepsilon^{2}[\obsEnd(\timeVar|\fixedFeatureVector)]^2}{162}\Big)\\
 & \le2\exp\Big(-\frac{k\varepsilon^{2}\theta^2}{162}\Big). \tag*{$\square$}
\end{align*}

\subsection{Proof of Lemma \ref{lem:kNN-R1-bias}\label{subsec:pf-lem-kNN-R1-bias}}

Recall from equation~\eqref{eq:Ek-U1} in Lemma \ref{lem:kNN-bad-R1}'s proof that
\[
 \mathbb{E}[U_1(\timeVar|\fixedFeatureVector)|\widetilde{\featureVar}]
 = \mathbb{E}_{\obsVar,\eventVar}\Big[-\frac{\eventVar\ind\{\obsVar\le \timeVar\}}{\obsEnd(\obsVar|\fixedFeatureVector)}\,\Big|\,\featureVar\in\mathcal{B}_{\fixedFeatureVector,\rho(\fixedFeatureVector,\widetilde{X})}^o\Big],
\]
where $\widetilde{X}$ is the $(k+1)$-st nearest neighbor of $\fixedFeatureVector$.
With abbreviation $\mathcal{B}^o:=\mathcal{B}_{\fixedFeatureVector,\rho(\fixedFeatureVector,\widetilde{X})}^o$,
\begin{align}
 &|\mathbb{E}[U_1(\timeVar|\fixedFeatureVector)|\widetilde{\featureVar}] -\log \survEnd(\timeVar|\fixedFeatureVector)| \nonumber \\
 & = |\mathbb{E}[U_1(\timeVar|\fixedFeatureVector) -\log \survEnd(\timeVar|\fixedFeatureVector)\,|\,\widetilde{\featureVar}]| \nonumber \\
 & = \Big| \frac{\int_{\mathcal{B}^o} \{\mathbb{E}_{\obsVar,\eventVar}[-\frac{\eventVar\ind\{\obsVar\le \timeVar\}}{\obsEnd(\obsVar|\fixedFeatureVector)}\,| \featureVar=\fixedFeatureVector'] - \log \survEnd(\timeVar|\fixedFeatureVector)\} d\featureDist(\fixedFeatureVector') }{\featureDist(\mathcal{B}^o)} \Big|\nonumber \\
 & \le \frac{\int_{\mathcal{B}^o} |\mathbb{E}_{\obsVar,\eventVar}[-\frac{\eventVar\ind\{\obsVar\le \timeVar\}}{\obsEnd(\obsVar|\fixedFeatureVector)}\,| \featureVar=\fixedFeatureVector'] - \log \survEnd(\timeVar|\fixedFeatureVector)| d\featureDist(\fixedFeatureVector') }{\featureDist(\mathcal{B}^o)} \label{eq:k-nn-survival-ptwise-R1-helper2b}
\end{align}
As we show next, for any $x'\in\mathcal{B}_{\fixedFeatureVector,h^*},$
\begin{equation}
\Big|\mathbb{E}_{\obsVar,\eventVar}\Big[-\frac{\eventVar\ind\{\obsVar\le \timeVar\}}{\obsEnd(\obsVar|\fixedFeatureVector)}\,\Big|\,\featureVar=\fixedFeatureVector'\Big] -\log \survEnd(\timeVar|\fixedFeatureVector)\Big|\le\frac{\varepsilon}{18},\label{eq:k-nn-survival-ptwise-helper3}
\end{equation}
which, combined with inequality \eqref{eq:k-nn-survival-ptwise-R1-helper2b} and noting that $\rho(\fixedFeatureVector,\widetilde{X})\le (h^*)^{\holderIndex}$, implies that
\[
|\mathbb{E}[U_1(\timeVar|\fixedFeatureVector)|\widetilde{\featureVar}]-\log \survEnd(\timeVar|\fixedFeatureVector)|\le\frac{\varepsilon}{18}.
\]
This means that conditioning on event $\badEvent{k\textsc{-NN}}{\text{far neighbors}}(\fixedFeatureVector)$ not happening, we deterministically have ${|\mathbb{E}[U_1(\timeVar|\fixedFeatureVector)|\widetilde{\featureVar}]-\log \survEnd(\timeVar|\fixedFeatureVector)|}\le\varepsilon/18$.

We now just need to show that inequality \eqref{eq:k-nn-survival-ptwise-helper3}
holds. First, note that $\log \survEnd(\timeVar|\fixedFeatureVector)$ is equal to the following expectation:
\begin{align}
 & \mathbb{E}_{\obsVar,\eventVar}\bigg[-\frac{\eventVar\ind\{\obsVar\le \timeVar\}}{\obsEnd(\obsVar|\fixedFeatureVector)}\,\bigg|\,\featureVar=\fixedFeatureVector\bigg]\nonumber\\
 & \quad =-\int_0^{\timeVar}\Big[\int_{s}^{\infty}\frac{1}{\obsEnd(s|\fixedFeatureVector)}d\mathbb{P}_{C|\featureVar=\fixedFeatureVector}(c)\Big]d\mathbb{P}_{\survVar|\featureVar=\fixedFeatureVector}(s)\nonumber\\
 & \quad =-\int_0^{\timeVar}\frac{1}{\obsEnd(s|\fixedFeatureVector)}\Big[\int_{s}^{\infty}d\mathbb{P}_{C|\featureVar=\fixedFeatureVector}(c)\Big]d\mathbb{P}_{\survVar|\featureVar=\fixedFeatureVector}(s)\nonumber\\
 & \quad =-\int_0^{\timeVar}\frac{1}{\obsEnd(s|\fixedFeatureVector)}\censEnd(s|\fixedFeatureVector)\survDensity(s|\fixedFeatureVector)ds\nonumber\\
 & \quad =-\int_0^{\timeVar}\frac{1}{\survEnd(s|\fixedFeatureVector)\censEnd(s|\fixedFeatureVector)}\censEnd(s|\fixedFeatureVector)\survDensity(s|\fixedFeatureVector)ds\nonumber\\
 & \quad =-\int_0^{\timeVar}\frac{1}{\survEnd(s|\fixedFeatureVector)}\survDensity(s|\fixedFeatureVector)ds\nonumber\\
 & \quad =\log \survEnd(\timeVar|\fixedFeatureVector)-\log\underbrace{\survEnd(0|\fixedFeatureVector)}_{1}\nonumber\\
 & \quad =\log \survEnd(\timeVar|\fixedFeatureVector)\label{eq:logF-integral},
\end{align}
where we have used the fact that $\frac{d}{dx}\log \survEnd(\timeVar|\fixedFeatureVector)=-\frac{\survDensity(\timeVar|\fixedFeatureVector)}{\survEnd(\timeVar|\fixedFeatureVector)}$
since $\survEnd$ is 1 minus the CDF and $f$ is the PDF of distribution
$\mathbb{P}_{\survVar|\featureVar=\fixedFeatureVector}$, and also Assumption A2 ensures that $\eventVar$ is not almost surely 0 (so that the integrals above are valid).

For any $\fixedFeatureVector'$ within distance $h^*$ of $\fixedFeatureVector$, using an integral calculation similar to the one above,
\begin{align*}
 & \bigg|\mathbb{E}_{\obsVar,\eventVar}\bigg[-\frac{\eventVar\ind\{\obsVar\le \timeVar\}}{\obsEnd(\obsVar|\fixedFeatureVector)}\;\bigg|\;\featureVar=\fixedFeatureVector'\bigg]-\log \survEnd(\timeVar|\fixedFeatureVector)\bigg|\\
 & =\bigg|-\int_0^{\timeVar}\frac{1}{\obsEnd(s|\fixedFeatureVector)}\censEnd(s|\fixedFeatureVector')\survDensity(s|\fixedFeatureVector')ds\\
 & \quad\quad+\int_0^{\timeVar}\frac{1}{\obsEnd(s|\fixedFeatureVector)}\censEnd(s|\fixedFeatureVector)\survDensity(s|\fixedFeatureVector)ds\bigg|\\
 & =\bigg|\int_0^{\timeVar}\frac{1}{\obsEnd(s|\fixedFeatureVector)}(\censEnd(s|\fixedFeatureVector)\survDensity(s|\fixedFeatureVector)-\censEnd(s|\fixedFeatureVector')\survDensity(s|\fixedFeatureVector'))ds\bigg|\\
 & \le\int_0^{\timeVar}\frac{1}{\obsEnd(s|\fixedFeatureVector)}\big|\censEnd(s|\fixedFeatureVector)\survDensity(s|\fixedFeatureVector)-\censEnd(s|\fixedFeatureVector')\survDensity(s|\fixedFeatureVector')\big|ds.
\end{align*}
Using the fact that $\obsEnd(\cdot|\fixedFeatureVector)$ monotonically decreases, Assumptions
A3 and A4 (in particular, recall that $\censEnd(s|\cdot)\survDensity(s|\cdot)$ is H\"{o}lder continuous with parameters $(\lips_{\textsc{\tiny{\survVar}}}+\survDensity^*\lips_{\textsc{\tiny{\censVar}}}s)$ and $\holderIndex$), and
the choice of critical distance $h^*\le[\frac{\varepsilon\theta}{18( \lips_{\textsc{\tiny{\survVar}}}\timeHorizon+({\survDensity^*\lips_{\textsc{\tiny{\censVar}}}\timeHorizon^2})/2 )}]^{1/\holderIndex}$,
\begin{align*}
 & \int_0^{\timeVar}\frac{1}{\obsEnd(s|\fixedFeatureVector)}\big|\censEnd(s|\fixedFeatureVector)\survDensity(s|\fixedFeatureVector)-\censEnd(s|\fixedFeatureVector')\survDensity(s|\fixedFeatureVector')\big|ds\\
 & \quad \le\frac{1}{\obsEnd(\timeVar|\fixedFeatureVector)}\int_0^{\timeVar}\big|\censEnd(s|\fixedFeatureVector)\survDensity(s|\fixedFeatureVector)-\censEnd(s|\fixedFeatureVector')\survDensity(s|\fixedFeatureVector')\big|ds\\
 & \quad \le\frac{1}{\obsEnd(\timeVar|\fixedFeatureVector)}\int_0^{\timeVar}( \lips_{\textsc{\tiny{\survVar}}}+\survDensity^*\lips_{\textsc{\tiny{\censVar}}}s )\rho(\fixedFeatureVector,\fixedFeatureVector')^\holderIndex ds\\
 & \quad \le\frac{1}{\obsEnd(\timeVar|\fixedFeatureVector)}\int_0^{\timeVar} (\lips_{\textsc{\tiny{\survVar}}}+\survDensity^*\lips_{\textsc{\tiny{\censVar}}}s) (h^*)^\holderIndex ds\\
 & \quad =\frac{(h^*)^\holderIndex}{\obsEnd(\timeVar|\fixedFeatureVector)}\Big(\lips_{\textsc{\tiny{\survVar}}}\timeVar+\frac{\survDensity^*\lips_{\textsc{\tiny{\censVar}}}\timeVar^2}2\Big)\\
 & \quad \le\frac{(h^*)^\holderIndex}{\theta}\Big(\lips_{\textsc{\tiny{\survVar}}}\timeHorizon+\frac{\survDensity^*\lips_{\textsc{\tiny{\censVar}}}\timeHorizon^2}2\Big) \\
 & \quad \le \frac{\big([\frac{\varepsilon\theta}{18( \lips_{\textsc{\tiny{\survVar}}}\timeHorizon+({\survDensity^*\lips_{\textsc{\tiny{\censVar}}}\timeHorizon^2})/2 )}]^{1/\holderIndex}\big)^\holderIndex}{\theta}\Big(\lips_{\textsc{\tiny{\survVar}}}\timeHorizon+\frac{\survDensity^*\lips_{\textsc{\tiny{\censVar}}}\timeHorizon^2}2\Big) \\
 & \quad =\frac{\varepsilon}{18}.
\end{align*}
which establishes inequality \eqref{eq:k-nn-survival-ptwise-helper3}.\hfill$\square$

\subsection{Proof of Lemma \ref{lem:kNN-R2-bias}\label{subsec:pf-lem-kNN-R2-bias}}

Recall the description of randomness in the proof of Lemma~\ref{lem:kNN-bad-R1}. Let $\widetilde{X}$ denote the $(k+1)$-st nearest neighbor. Since bad event $\badEvent{k\textsc{-NN}}{\text{far neighbors}}(\fixedFeatureVector)$ does not happen, we know that $\rho(\fixedFeatureVector,\widetilde{X})\le (h^*)^{\holderIndex}$. This means that, using the fact that $\obsEnd(s|\cdot)$ is H\"{o}lder continuous with parameters $(\lips_{\textsc{\tiny{\survVar}}}+\lips_{\textsc{\tiny{\censVar}}})s$ and $\holderIndex$,
\begin{align*}
 & |\obsEnd(s|\fixedFeatureVector)-\mathbb{E}[\estObsEnd^{k\textsc{-NN}}(s|\fixedFeatureVector)|\widetilde{\featureVar}]|\nonumber \\
 & \quad = |\obsEnd(s|\fixedFeatureVector)-\mathbb{P}(\obsVar>s\,|\,\featureVar\in\mathcal{B}_{\fixedFeatureVector,\rho(\fixedFeatureVector,\widetilde{X})}^o)|\nonumber \\
 & \quad =\Bigg|\obsEnd(s|\fixedFeatureVector)-\frac{\int_{\mathcal{B}_{\fixedFeatureVector,\rho(\fixedFeatureVector,\widetilde{X})}^o}\obsEnd(s|\fixedFeatureVector')d\featureDist(\fixedFeatureVector')}{\featureDist(\mathcal{B}_{\fixedFeatureVector,\rho(\fixedFeatureVector,\widetilde{X})}^o)}\Bigg|\nonumber \\
 & \quad =\Bigg|\frac{\int_{\mathcal{B}_{\fixedFeatureVector,\rho(\fixedFeatureVector,\widetilde{X})}^o}[\obsEnd(s|\fixedFeatureVector)-\obsEnd(s|\fixedFeatureVector')]d\featureDist(\fixedFeatureVector')}{\featureDist(\mathcal{B}_{\fixedFeatureVector,\rho(\fixedFeatureVector,\widetilde{X})}^o)}\Bigg|\nonumber \\
 & \quad \le\frac{\int_{\mathcal{B}_{\fixedFeatureVector,\rho(\fixedFeatureVector,\widetilde{X})}^o}|\obsEnd(s|\fixedFeatureVector)-\obsEnd(s|\fixedFeatureVector')|d\featureDist(\fixedFeatureVector')}{\featureDist(\mathcal{B}_{\fixedFeatureVector,\rho(\fixedFeatureVector,\widetilde{X})}^o)}\nonumber \\
 & \quad \le\frac{\int_{\mathcal{B}_{\fixedFeatureVector,\rho(\fixedFeatureVector,\widetilde{X})}^o} (\lips_{\textsc{\tiny{\survVar}}}+\lips_{\textsc{\tiny{\censVar}}})s \rho(\fixedFeatureVector,\fixedFeatureVector')^\holderIndex d\featureDist(\fixedFeatureVector')}{\featureDist(\mathcal{B}_{\fixedFeatureVector,\rho(\fixedFeatureVector,\widetilde{X})}^o)}\nonumber \\
 & \quad \le\frac{ (\lips_{\textsc{\tiny{\survVar}}}+\lips_{\textsc{\tiny{\censVar}}})s (h^*)^\holderIndex\int_{\mathcal{B}_{\fixedFeatureVector,\rho(\fixedFeatureVector,\widetilde{X})}^o}d\featureDist(\fixedFeatureVector')}{\featureDist(\mathcal{B}_{\fixedFeatureVector,\rho(\fixedFeatureVector,\widetilde{X})}^o)}\nonumber \\
 & \quad = (\lips_{\textsc{\tiny{\survVar}}}+\lips_{\textsc{\tiny{\censVar}}})s(h^*)^\holderIndex.
\end{align*}
Taking the supremum of both sides over $s\in[0,\timeHorizon]$, multiplying through
by $\frac{2}{\theta^{2}}$, and noting that $h^*\le[\frac{\varepsilon\theta^{2}}{36(\lips_{\textsc{\tiny{\survVar}}}+\lips_{\textsc{\tiny{\censVar}}})\timeHorizon}]^{1/\holderIndex}$,
we obtain
\begin{align*}
 & \frac{2}{\theta^{2}}\sup_{s\in[0,\timeHorizon]}|\obsEnd(s|\fixedFeatureVector)-\mathbb{E}[\estObsEnd^{k\textsc{-NN}}(s|\fixedFeatureVector)|\widetilde{\featureVar}]|\\
 & \quad \le\frac{2}{\theta^{2}}(\lips_{\textsc{\tiny{\survVar}}}+\lips_{\textsc{\tiny{\censVar}}})\timeHorizon (h^*)^\holderIndex\\
 & \quad \le\frac{2}{\theta^{2}}(\lips_{\textsc{\tiny{\survVar}}}+\lips_{\textsc{\tiny{\censVar}}})\timeHorizon \Big(\Big[\frac{\varepsilon\theta^{2}}{36(\lips_{\textsc{\tiny{\survVar}}}+\lips_{\textsc{\tiny{\censVar}}})\timeHorizon}\Big]^{1/\holderIndex}\Big)^\holderIndex=\frac{\varepsilon}{18}. \tag*{$\square$}
\end{align*}

\subsection{Proof of Lemma \ref{lem:kNN-R3}}
\label{subsec:pf-lem-kNN-R3}

We abbreviate the set of $k$ nearest training subjects $\neighborsKNN(\fixedFeatureVector)$ as the set $\setS$. Since bad event $\badEvent{k\textsc{-NN}}{\text{bad }\timeHorizon}(\fixedFeatureVector)$ does not happen, we have $d_{\setS}^{+}(\timeHorizon)>k\theta/2$. Note that $|U_3(\timeVar|\fixedFeatureVector)|=\sum_{i\in\setS}\Xi_i$, where
\begin{align*}
\Xi_i & :=\eventVar_i\ind\{\obsVar_i\le \timeVar\}\sum_{\ell=2}^{\infty}\frac{1}{\ell(d_{\setS}^{+}(\obsVar_i)+1)^{\ell}}.
\end{align*}
Using the fact that $d_{\setS}^{+}$ monotonically decreases, and that $\sum_{\ell=2}^{\infty}\frac{1}{\ell(z+1)^{\ell}}=\log(1+\frac{1}z)-\frac{1}{z+1}\le\frac{1}{(z+1)^{2}}$ for all $z\ge0.46241$,
\begin{align*}
\Xi_i &\le \sum_{\ell=2}^{\infty}\frac{1}{\ell(d_{\setS}^{+}(\timeVar)+1)^{\ell}} \le \frac{1}{(d_{\setS}^{+}(\timeVar)+1)^{2}} \\
&\le \frac{1}{(d_{\setS}^{+}(\timeHorizon)+1)^{2}} \le\frac{1}{(d_{\setS}^{+}(\timeHorizon))^{2}} \le\frac{4}{k^{2}\theta^{2}}.
\end{align*}
Lastly, using the assumption that $k\ge\frac{72}{\varepsilon\theta^{2}}$,
\begin{align*}
|U_3(\timeVar|\fixedFeatureVector)|&=\sum_{i\in\setS}\Xi_i\le\frac{4|\setS|}{k^{2}\theta^{2}}=\frac{4k}{k^{2}\theta^{2}}=\frac{4}{k\theta^{2}}\le\frac{\varepsilon}{18}. \tag*{$\square$}
\end{align*}

\subsection{Technical Changes to the Analysis by \citet{foldes_1981}}
\label{sec:technical-changes}

Our event $\badEvent{k\textsc{-NN}}{\text{bad }\timeHorizon}(\fixedFeatureVector)$ not happening ensures that $d_{\neighborsKNN(\fixedFeatureVector)}^{+}(\timeHorizon)>k\theta/2$. F\"{o}ldes and Rejt\"{o} instead condition on two separate bad events, the first being $\{\max_{i\in\neighborsKNN(\fixedFeatureVector)}\obsVar_i\le \timeHorizon\}$. When this bad event does not happen, then the number of survivors beyond time $\timeHorizon$ satisfies $d_{\neighborsKNN(\fixedFeatureVector)}^{+}(\timeHorizon)\ge1$. This is a bit too weak of a requirement on $d_{\neighborsKNN(\fixedFeatureVector)}^{+}(\timeHorizon)$. As a result, F\"{o}ldes and Rejt\"{o} condition on a second bad event not happening to guarantee that (slightly rephrased to be in our setup's context) $d_{\neighborsKNN(\fixedFeatureVector)}^{+}(\timeHorizon)>k [\obsEnd(\timeHorizon|\fixedFeatureVector)]^2$ (they ensure that this holds with high probability using Bernstein's inequality). Effectively this means that they have an extra bad event that they condition on not happening.

Next, in the partitioning of $[0,\timeHorizon]$ into $L(\varepsilon)$ pieces, F\"{o}ldes and Rejt\"{o} actually have all bad events except $\{\max_{i\in\neighborsKNN(\fixedFeatureVector)}\obsVar_i\le \timeHorizon\}$ being repeated for $t=\eta_{1},\dots,\eta_{L(\varepsilon)}$. Put another way, their final bound is looser since they multiply many more terms by~$L(\varepsilon)$.

Lastly, F\"{o}ldes and Rejt\"{o} use versions of the DKW and Bernstein's inequalities with vintage constants that have since been improved. Notably, nowadays the DKW inequality generally refers to the refinement by \citet{massart_1990}.

\section{Proof of Corollary~\ref{thm:kNN-strong-consistency}}
\label{sec:pf-kNN-strong-consistency}

The basic idea of the proof is to solve for $\varepsilon$ and $n$ that satisfy both: i) sufficient conditions \eqref{eq:kNN-sufficient} with error probability set to be equal to $\probError=1/n^2$, and ii) $\varepsilon\le\frac{18\Lips (r^*)^\holderIndex}{\theta}$. The choice of $\probError$ is not special and is chosen so that summing it from~$n=1$ to $n=\infty$ results in a finite number, upon which the Borel-Cantelli lemma finishes the proof. (This proof would still work but with different constants if $\probError=1/n^{\nu}$ for any $\nu>1$ due to convergence of hyperharmonic series.) There is a small technical hiccup of making sure that there is a valid integer to set $k$ to be. The rest is a fair amount of algebra involving the Lambert W function. We provide the details for just the $k$-NN case below.

Let $\probError=1/n^{2}$. Recall that $h^*=(\frac{\varepsilon\theta}{18\Lips})^{1/\holderIndex}$. Then one can easily check that each of the terms in bound~\eqref{eq:kNN-ptwise-bound} is at most~$\probError/4$ when~$k$ and~$n$ satisfy
\[
\frac{648}{\varepsilon^{2}\theta^{4}}\log\frac{32n^{2}}{\varepsilon}\le k\le\frac{np_{\min}}{2}\Big(\frac{\varepsilon\theta}{18\Lips}\Big)^{d/\holderIndex}.
\]
We shall show how to set $\varepsilon\in(0, \frac{18\Lips (r^*)^\holderIndex}\theta]$ as a function of $n$ (along with additional conditions on $n$) such that
\begin{equation}
\frac{649}{\varepsilon^{2}\theta^{4}}\log\frac{32n^{2}}{\varepsilon}\le\frac{np_{\min}}{2}\Big(\frac{\varepsilon\theta}{18\Lips}\Big)^{d/\holderIndex}.\label{eq:kNN-almost-sure-helper1}
\end{equation}
Having the constant 649 is intentional. When inequality \eqref{eq:kNN-almost-sure-helper1} holds, then
\[
\frac{648}{\varepsilon^{2}\theta^{4}}\log\frac{32n^{2}}{\varepsilon}+1
\!<\!
\frac{649}{\varepsilon^{2}\theta^{4}}\log\frac{32n^{2}}{\varepsilon}
\!\le\!
\frac{np_{\min}}{2}\Big(\frac{\varepsilon\theta}{18\Lips}\Big)^{d/\holderIndex},
\]
which guarantees there to be at least one integer between $\frac{648}{\varepsilon^{2}\theta^{4}}\log\frac{32}{\varepsilon\probError}$ and $\frac{np_{\min}}{2}(\frac{\varepsilon\theta}{18\Lips})^{d/\holderIndex}$. Hence, a valid choice for $k$ is
\[
k=\Big\lfloor\frac{np_{\min}}{2}\Big(\frac{\varepsilon\theta}{18\Lips}\Big)^{d/\holderIndex}\Big\rfloor.
\]
The following pair of lemmas help us obtain a choice for $\varepsilon$ as well as conditions on how large $n$ should be; these lemmas are fundamentally about the Lambert W function.
\begin{lem}
[Lemma 3.6.11 of \citealt{george_devavrat_book}, combined with Theorem 2.1 of \citealt{hoorfar_2008}]
\label{lem:W0-conditions}
Let $W_{0}$ be the principal branch of the Lambert W function. For any $a>0,b>0,c>0$, and $z\in(0,b)$, we have
\[
z^{c}\ge a\log\frac{b}{z}
\]
if either of the following is true:
\begin{itemize}
\item [(a)]We have
\begin{equation}
z\ge b\exp\bigg(-\frac{1}{c}W_{0}\Big(\frac{cb^{c}}{a}\Big)\bigg).
\label{eq:Chen-Shah-Lem-3-6-11-sufficient-condition}
\end{equation}
\item [(b)]We have 
\[
\frac{cb^{c}}{a}\ge e\quad\text{and}\quad z\ge\Big[\frac{a}{c}\log\Big(\frac{cb^{c}}{a}\Big)\Big]^{1/c}.
\]
\end{itemize}
\end{lem}

\begin{proof}
This lemma with only part (a) is precisely Lemma 3.6.11 of \citet{george_devavrat_book}. Under the assumption that $\frac{cb^c}{a}\ge e$, then applying Theorem 2.1 of \citet{hoorfar_2008},
\[
W_{0}\Big(\frac{cb^{c}}{a}\Big)\ge\log\Big(\frac{cb^{c}}{a}\Big)-\log\log\Big(\frac{cb^{c}}{a}\Big).
\]
Thus, a sufficient condition to guarantee that inequality \eqref{eq:Chen-Shah-Lem-3-6-11-sufficient-condition} holds is to ask that
\begin{align*}
z & \ge b\exp\bigg(-\frac{1}{c}\Big[\log\Big(\frac{cb^{c}}{a}\Big)-\log\log\Big(\frac{cb^{c}}{a}\Big)\Big]\bigg)\\
 & =\Big[\frac{a}{c}\log\Big(\frac{cb^{c}}{a}\Big)\Big]^{1/c}.\qedhere
\end{align*}
\end{proof}
\begin{lem}
\label{lem:W-1-conditions}
Let $W_{-1}$ be the lower branch of the Lambert W function. For any $a>0$, $b>0$, and $z>0$,
\[
z\ge a\log z+b
\]
if any of the following is true:
\begin{itemize}
\item [(a)]We have $\frac{b}{a}+\log a\le1$.
\item [(b)]We have $\frac{b}{a}+\log a>1$ and
\[
z\ge-aW_{-1}\Big(-\frac{1}{ae^{b/a}}\Big).
\]
\item [(c)]We have $\frac{b}{a}+\log a>1$ and
\[
z\ge a\big(1+\sqrt{2\log(ae^{b/a-1})}+\log(ae^{b/a-1})\big).
\]
\end{itemize}
\end{lem}

\begin{proof}
To prove (a), using the assumption that $\frac{b}{a}+\log a\le1$, and recalling that $\log z\le z-1$ for all $z>0$,
\begin{align*}
a\log z+b & =a\Big(\log\frac{z}{a}+\frac{b}{a}+\log a\Big)\\
 & \le a\Big(\log\frac{z}{a}+1\Big)\\
 & \le a\Big(\frac{z}{a}-1+1\Big)\\
 & =z.
\end{align*}
To prove (b), first off, note that under the assumption that $\frac{b}{a}+\log a>1$, then $-\frac{1}{e}<-\frac{1}{ae^{b/a}}<0$, so $W_{-1}(-\frac{1}{ae^{b/a}})$ is well-defined. Next, assumption $z\ge-aW_{-1}\big(-\frac{1}{ae^{b/a}}\big)$ can be rewritten as
\begin{equation}
-\frac{z}{a}\le W_{-1}\Big(-\frac{1}{ae^{b/a}}\Big).\label{eq:W-1-sufficient-condition}
\end{equation}
At this point, noting that the inverse of $W_{-1}$ (namely $W_{-1}^{-1}(s)=se^{s}$, where $s\in(-\infty,-1]$) is a monotonically decreasing function, applying the inverse of $W_{-1}$ to both sides of the above inequality yields
\[
-\frac{z}{a}e^{-z/a}\ge-\frac{1}{ae^{b/a}}.
\]
Rearranging terms yields $z\ge a\log z+b$, as desired.

Lastly, the proof for (c) just builds on (b). Using Theorem 1 of \citet{lambert_w},
\[
W_{-1}\Big(-\frac{1}{ae^{b/a}}\Big)>-1-\sqrt{2\log(ae^{b/a-1})}-\log(ae^{b/a-1}).
\]
A sufficient condition that guarantees inequality~\eqref{eq:W-1-sufficient-condition} to hold is that
\[
-\frac{z}{a}\le-1-\sqrt{2\log(ae^{b/a-1})}-\log(ae^{b/a-1}),
\]
i.e.,
\[
z\ge a\big(1+\sqrt{2\log(ae^{b/a-1})}+\log(ae^{b/a-1})\big).\qedhere
\]
\end{proof}
Using Lemma \ref{lem:W0-conditions} (with $a=\frac{2\cdot649}{n \theta^4 p_{\min} (\frac{\theta}{18\Lips})^{d/\holderIndex} }$, $b=32n^{2}$, $c=\frac{d}{\holderIndex}+2$, and $z=\varepsilon$) and a bit of algebra, inequality~\eqref{eq:kNN-almost-sure-helper1} holds if
\begin{align*}
n & \ge\Big(\frac{e}{\chi}\Big)^{\frac{1}{\frac{2d}{\holderIndex}+5}},\\
\varepsilon & \ge \Big[ \frac{2\cdot649\cdot(\frac{2d}{\holderIndex}+5)}{(\frac{d}{\holderIndex}+2) \theta^4 p_{\min} (\frac{\theta}{18\Lips})^{\frac{d}{\holderIndex}}} \cdot\frac{1}{n}\cdot\log(\chi^{\frac{1}{\frac{2d}{\holderIndex}+5}}n)\Big]^{\frac{1}{\frac{d}{\holderIndex}+2}},
\end{align*}
where
\begin{equation}
\chi := \frac{(\frac{d}{\holderIndex} + 2) (32)^{\frac{d}{\holderIndex}+2} \theta^4 p_{\min} (\frac{\theta}{18\Lips})^{\frac{d}{\holderIndex}}}{2\cdot649}.\label{eq:kNN-almost-sure-main-constant}
\end{equation}
In particular, we shall choose
\[
\varepsilon=\Big[ \frac{2\cdot649\cdot(\frac{2d}{\holderIndex}+5)}{(\frac{d}{\holderIndex}+2) \theta^4 p_{\min} (\frac{\theta}{18\Lips})^{\frac{d}{\holderIndex}}} \cdot\frac{1}{n}\cdot\log(\chi^{\frac{1}{\frac{2d}{\holderIndex}+5}}n)\Big]^{\frac{1}{\frac{d}{\holderIndex}+2}}.
\]
To make sure that $\varepsilon\le\frac{18\Lips (r^*)^\holderIndex}\theta$, we require that
\begin{align}
n & \!\ge\! \frac{2\cdot649}{(\frac{d}{\holderIndex}+2)p_{\min}(18\theta\Lips)^2(r^*)^{2\alpha+d}} \big[({\textstyle \frac{2d}{\holderIndex}\!+\!5})\log n\!+\!\log\chi\big].\label{eq:kNN-almost-sure-helper2}
\end{align}
Using Lemma \ref{lem:W-1-conditions} (with $a=\frac{2\cdot649\cdot(\frac{2d}{\holderIndex}+5)}{(\frac{d}{\holderIndex}+2)p_{\min}(18\theta\Lips)^2(r^*)^{2\alpha+d}}$, $b=\frac{2\cdot649\cdot\log\chi}{(\frac{d}{\holderIndex}+2)p_{\min}(18\theta\Lips)^2(r^*)^{2\alpha+d}}$, and $z=n$), and defining
\[
u:=\log\Big(\Big[ \frac{2\cdot649\cdot(\frac{2d}{\holderIndex}+5)}{(\frac{d}{\holderIndex}+2)p_{\min}(18\theta\Lips)^2(r^*)^{2\alpha+d}} \Big]\frac{\chi^{\frac{1}{\frac{2d}{\holderIndex}+5}}}{e}\Big),
\]
then condition \eqref{eq:kNN-almost-sure-helper2} holds if $u\le0$ or, in the event that $u>0$, if we further constrain $n$ to satisfy
\begin{align*}
n & \ge \frac{2\cdot649\cdot(\frac{2d}{\holderIndex}+5)}{(\frac{d}{\holderIndex}+2)p_{\min}(18\theta\Lips)^2(r^*)^{2\alpha+d}} (1+\sqrt{2u}+u).
\end{align*}
In summary, define
\begin{align*}
c_1 &:= \frac12 p_{\min}^{\frac{2\holderIndex}{2\holderIndex+d}} \Big( \frac{649(5\holderIndex+2d)}{162(2\holderIndex+d)} \Big)^{\frac{d}{2\holderIndex+d}} \Big( \frac1{\Lips\theta} \Big)^{\frac{2d}{2\holderIndex+d}}, \\
c_2 &:= \chi^{\frac{1}{2d+5}} = \Big[\frac{512(2\holderIndex+d)p_{\min}\theta^4}{649\holderIndex}\Big(\frac{16\theta}{9\Lips}\Big)^{\frac{d}{\holderIndex}}\Big]^{\frac{\holderIndex}{5\holderIndex+2d}}, \\
c_3 &:= \Big[ \frac{2\cdot649\cdot(\frac{2d}{\holderIndex}+5)}{(\frac{d}{\holderIndex}+2) \theta^4 p_{\min} (\frac{\theta}{18\Lips})^{\frac{d}{\holderIndex}}} \Big]^{\frac{1}{\frac{d}{\holderIndex}+2}} \\
    &\hspace{.25em}= \Big[\frac{1298(5\holderIndex+2d)}{(2\holderIndex+d)p_{\min}\theta^4}\Big(\frac{18\Lips}{\theta}\Big)^{\frac{d}{\holderIndex}}\Big]^{\frac{\holderIndex}{2\holderIndex+d}} \\
c_4 &:= \frac{2\cdot649\cdot(\frac{2d}{\holderIndex}+5)}{(\frac{d}{\holderIndex}+2)p_{\min}(18\theta\Lips)^2(r^*)^{2\alpha+d}} \\
    &\hspace{.25em}= \frac{649(5\holderIndex+2d)}{162(2\holderIndex+d)p_{\min}(r^*)^{2\holderIndex+d}\theta^2\Lips^2}.
\end{align*}
Note that $u=\log (c_2 c_4 / e)$. Set
\begin{align*}
n_{0} & \!:=\!\begin{cases}
\big\lceil\frac{e^{\holderIndex/(5\holderIndex+2d)}}{c_2}\big\rceil & \!\text{if }\frac{c_2 c_4}e \!\le\! 1,\\
\big\lceil\max\big\{\frac{e^{\holderIndex/(5\holderIndex+2d)}}{c_2},\\
\quad c_4(1+\sqrt{2\log\frac{c_2 c_4}{e}}+\log\frac{c_2 c_4}{e})\big\}\big\rceil & \!\text{if }\frac{c_2 c_4}e \!>\! 1.
\end{cases}
\end{align*}
Then for any $n\ge n_{0}$, the conditions that we discussed for $n$ are met, so we can choose
\begin{align}
k_n & :=\Big\lfloor c_1 n^{\frac{2\holderIndex}{2\holderIndex+d}}\big(\log(c_2 n)\big)^{\frac{d}{2\holderIndex+d}}\Big\rfloor, \nonumber \\
\varepsilon_n & := c_3 \Big(\frac{\log(c_2 n)}{n}\Big)^{\frac{\holderIndex}{2\holderIndex+d}} \label{eq:kNN-almost-sure-how-to-choose-eps}
\end{align}
to achieve
\[
\mathbb{P}\Big( \sup_{\timeVar\in[0,\timeHorizon]} |\widehat{\survEnd}^{k_n\textsc{-NN}}(\timeVar|\fixedFeatureVector) - \survEnd(\timeVar|\fixedFeatureVector)| \ge\varepsilon_n\Big)\le\frac{1}{n^{2}}.
\]
As a result, we have
\begin{align*}
 & \sum_{n=1}^{\infty}\mathbb{P}\Big( \sup_{\timeVar\in[0,\timeHorizon]} |\widehat{\survEnd}^{k_n\textsc{-NN}}(\timeVar|\fixedFeatureVector) - \survEnd(\timeVar|\fixedFeatureVector)| \ge\varepsilon_n\Big)\\
 & \le n_0 + \sum_{n=n_{0}}^{\infty}\frac{1}{n^{2}}\le n_0 + \sum_{n=1}^{\infty}\frac{1}{n^{2}}= n_0 + \frac{\pi^{2}}{6}<\infty,
\end{align*}
so by the Borel-Cantelli lemma,
\[
\mathbb{P}\Big(\limsup_{n\rightarrow\infty}\Big\{ \sup_{\timeVar\in[0,\timeHorizon]} |\widehat{\survEnd}^{k_n\textsc{-NN}}(\timeVar|\fixedFeatureVector) - \survEnd(\timeVar|\fixedFeatureVector)| \ge\varepsilon_n\Big\}\Big)=0.
\]

\section{Proof of Theorem~\ref{thm:h-near-survival}}
\label{sec:pf-thm-h-near-survival}

The proof of the fixed-radius NN estimator is similar to that of the $k$-NN estimator and actually does not require the more nuanced analysis of \citet{chaudhuri_dasgupta_2014}. In particular, in proving the $k$-NN estimator guarantee, we took the expectation $\mathbb{E}[\,\cdot\,|\widetilde{\featureVar}]$, where $\widetilde{\featureVar}$ was the feature vector of the \mbox{$(k+1)$-st} nearest neighbor of~$\fixedFeatureVector$. This conditioning made the $k$ nearest neighbors appear i.i.d. The analysis for the fixed-radius NN estimator is simpler in that with the threshold distance $h>0$ fixed, the training data that land within distance~$h$ are i.i.d.~as is. However, the bad events do slightly change since now there could be no neighbors found within distance~$h$ of~$\fixedFeatureVector$. Whereas previously the number of neighbors was fixed, now the number of neighbors being random. Thus, instead of conditioning on the \mbox{$(k+1)$-st} nearest neighbor, we condition on the number of neighbors.

We focus on the proof of the main fixed-radius NN estimator nonasymptotic bound~\eqref{eq:h-near-ptwise-bound}. The proof of the strong consistency result is the same as that of the $k$-NN estimator with the only change being that we do not need to worry about $k$ (in proving the $k$-NN strong consistency result, Corollary~\ref{thm:kNN-strong-consistency}, we had a short extra step that makes sure that $k$ can be chosen to be a valid integer; for establishing the fixed-radius NN strong consistency result, we do not need this extra step although even if we use it, the choices for $c_1$, $c_2$, $c_3$, and $n_0$ still work). We then pick the threshold distance to be $h=h^*=(\frac{\varepsilon\theta}{18\Lips})^{1/\holderIndex}$, where $\varepsilon$ is chosen as in equation~\eqref{eq:kNN-almost-sure-how-to-choose-eps}.

We proceed to proving the nonasymptotic bound~\eqref{eq:h-near-ptwise-bound}. Let $\fixedFeatureVector\in\text{supp}(\featureDist)$ and $N_{x,h}=|\mathcal{N}_{\textsc{NN}(h)}(\fixedFeatureVector)|$ denote the number of neighbors found within distance $h$ of $\fixedFeatureVector$. Using the same reasoning as for the $k$-NN estimator,
\begin{align*}
 & \log\widehat{\survEnd}^{\textsc{NN}(h)}(\timeVar|\fixedFeatureVector)\\
 & \quad =\log\prod_{i\in\mathcal{N}_{\textsc{NN}(h)}(\fixedFeatureVector)}\Big(\frac{d_{\mathcal{N}_{\textsc{NN}(h)}(\fixedFeatureVector)}^{+}(\obsVar_i)}{d_{\mathcal{N}_{\textsc{NN}(h)}(\fixedFeatureVector)}^{+}(\obsVar_i)+1}\Big)^{\eventVar_i\ind\{\obsVar_i\le \timeVar\}}\\
 & \quad =V_1(\timeVar|\fixedFeatureVector)+V_2(\timeVar|\fixedFeatureVector)+V_3(\timeVar|\fixedFeatureVector),
\end{align*}
where
\begin{align*}
V_1(\timeVar|\fixedFeatureVector) & =-\frac{1}{N_{x,h}}\!\sum_{\substack{i\in\mathcal{N}_{\textsc{NN}(h)}(\fixedFeatureVector)\\
\text{s.t.~}\obsVar_i\le \timeVar
}
}\!\!\!\!\eventVar_i\frac{1}{\obsEnd(\obsVar_i|\fixedFeatureVector)},\\
V_2(\timeVar|\fixedFeatureVector) & =-\frac{1}{N_{x,h}}\!\sum_{\substack{i\in\mathcal{N}_{\textsc{NN}(h)}(\fixedFeatureVector)\\
\text{s.t.~}\obsVar_i\le \timeVar
}
}\!\!\!\!\eventVar_i\Big[\frac{N_{x,h}}{d_{\mathcal{N}_{\textsc{NN}(h)}(\fixedFeatureVector)}^{+}(\obsVar_i)+1}\!\\
&\quad\qquad\qquad\qquad\qquad\;-\!\frac{1}{\obsEnd(\obsVar_i|\fixedFeatureVector)}\Big],\\
V_3(\timeVar|\fixedFeatureVector) & =-\sum_{\substack{i\in\mathcal{N}_{\textsc{NN}(h)}(\fixedFeatureVector)\\
\text{s.t.~}\obsVar_i\le \timeVar
}
}\!\!\!\!\eventVar_i\sum_{\ell=2}^{\infty}\frac{1}{\ell(d_{\mathcal{N}_{\textsc{NN}(h)}(\fixedFeatureVector)}^{+}(\obsVar_i)+1)^{\ell}}.
\end{align*}
Defining $\estObsEnd^{\textsc{NN}(h)}(s|\fixedFeatureVector):=\frac{d_{\mathcal{N}_{\textsc{NN}(h)}(\fixedFeatureVector)}^{+}(s)}{N_{x,h}}$, then the bad events are:
\begin{itemize}[leftmargin=1.5em,topsep=0pt,itemsep=0ex,partopsep=1ex,parsep=1ex]
\item $\badEvent{\textsc{NN}(h)}{\text{few neighbors}}(\fixedFeatureVector):=\{N_{x,h}\le\frac{n\featureDist(\mathcal{B}_{\fixedFeatureVector,h})}{2}\}$
\item $\badEvent{\textsc{NN}(h)}{\text{bad }\timeHorizon}(\fixedFeatureVector):=\{d_{\mathcal{N}_{\textsc{NN}(h)}(\fixedFeatureVector)}^{+}(\timeHorizon)\le\frac{N_{x,h}\theta}{2}\}$
\item $\badEvent{\textsc{NN}(h)}{\text{bad EDF}}(\fixedFeatureVector):=$\\
$\big\{{\underset{s\ge0}{\sup}|\estObsEnd^{\textsc{NN}(h)}(s|\fixedFeatureVector)-\mathbb{E}[\estObsEnd^{\textsc{NN}(h)}(s|\fixedFeatureVector)|N_{x,h}]|}>\frac{\varepsilon\theta^{2}}{36}\big\}$
\item $\badEvent{\textsc{NN}(h)}{\text{bad }V_1}(t,x):=\{|V_1(\timeVar|\fixedFeatureVector)-\mathbb{E}[V_1(\timeVar|\fixedFeatureVector)|N_{x,h}]|\ge\frac{\varepsilon}{18}\}$
\end{itemize}
Once all of these bad events do not happen, then applying a very similar proof to the $k$-NN estimator yields Theorem \ref{thm:kNN-survival}. Note that as before, we actually want $\badEvent{\textsc{NN}(h)}{\text{bad }V_1}(t,x)$ to hold for a finite collection of times $t=\eta_{1},\dots,\eta_{L(\varepsilon)}$ within interval $[0,\timeHorizon]$.

\interdisplaylinepenalty=10000

We remark that the union bounding over the bad events is done slightly differently for the fixed-radius NN estimator. In particular, at least one of the bad events happening can actually be written as the union over the following events:
\begin{itemize}[leftmargin=1.5em,topsep=0pt,itemsep=0ex,partopsep=1ex,parsep=1ex]
\item $\badEvent{\textsc{NN}(h)}{\text{few neighbors}}(\fixedFeatureVector)$
\item $\badEvent{\textsc{NN}(h)}{\text{bad }\timeHorizon}(\fixedFeatureVector)\cap[\badEvent{\textsc{NN}(h)}{\text{few neighbors}}(\fixedFeatureVector)]^{c}$
\item $\badEvent{\textsc{NN}(h)}{\text{bad EDF}}(\fixedFeatureVector)\cap[\badEvent{\textsc{NN}(h)}{\text{few neighbors}}(\fixedFeatureVector)]^{c}$
\item $\badEvent{\textsc{NN}(h)}{\text{bad }V_1}(t,x)\cap[\badEvent{\textsc{NN}(h)}{\text{few neighbors}}(\fixedFeatureVector)]^{c}$ for $t=\eta_{1},\dots,\eta_{L(\varepsilon)}$
\end{itemize}
We use the fact that for any two events $\mathcal{E}_1$ and $\mathcal{E}_2$, $\mathbb{P}(\mathcal{E}_1\cap\mathcal{E}_2)=\mathbb{P}(\mathcal{E}_1)\mathbb{P}(\mathcal{E}_2|\mathcal{E}_1)\le\mathbb{P}(\mathcal{E}_2|\mathcal{E}_1)$. Then
\begin{align*}
 & \mathbb{P}(\text{at least one bad event happens})\\
 & \quad \le\mathbb{P}\big(\badEvent{\textsc{NN}(h)}{\text{few neighbors}}(\fixedFeatureVector)\big)\\
 & \quad \quad+\mathbb{P}\big(\badEvent{\textsc{NN}(h)}{\text{bad }\timeHorizon}(\fixedFeatureVector)\cap[\badEvent{\textsc{NN}(h)}{\text{few neighbors}}(\fixedFeatureVector)]^{c}\big)\\
 & \quad \quad+\mathbb{P}\big(\badEvent{\textsc{NN}(h)}{\text{bad EDF}}(\fixedFeatureVector)\cap[\badEvent{\textsc{NN}(h)}{\text{few neighbors}}(\fixedFeatureVector)]^{c}\big)\\
 & \quad \quad+\sum_{\ell=1}^{L(\varepsilon)}\mathbb{P}\big(\badEvent{\textsc{NN}(h)}{\text{bad }V_1}(\eta_{\ell},x)\cap[\badEvent{\textsc{NN}(h)}{\text{few neighbors}}(\fixedFeatureVector)]^{c}\big)\\
 & \quad \le\mathbb{P}\big(\badEvent{\textsc{NN}(h)}{\text{few neighbors}}(\fixedFeatureVector)\big)\\
 & \quad \quad+\mathbb{P}\big(\badEvent{\textsc{NN}(h)}{\text{bad }\timeHorizon}(\fixedFeatureVector)\,\big|\,[\badEvent{\textsc{NN}(h)}{\text{few neighbors}}(\fixedFeatureVector)]^{c}\big)\\
 & \quad \quad+\mathbb{P}\big(\badEvent{\textsc{NN}(h)}{\text{bad EDF}}(\fixedFeatureVector)\,\big|\,[\badEvent{\textsc{NN}(h)}{\text{few neighbors}}(\fixedFeatureVector)]^{c}\big)\\
 & \quad \quad+\sum_{\ell=1}^{L(\varepsilon)}\mathbb{P}\big(\badEvent{\textsc{NN}(h)}{\text{bad }V_1}(\eta_{\ell},x)\,\big|\,[\badEvent{\textsc{NN}(h)}{\text{few neighbors}}(\fixedFeatureVector)]^{c}\big).
\end{align*}
The rest of this section is on giving upper bounds for the four different probability terms that appear on the RHS, and also on why when all of these bad events do not happen, we indeed have $|\widehat{\survEnd}^{\textsc{NN}(h)}(\timeVar|\fixedFeatureVector)-\survEnd(\timeVar|\fixedFeatureVector)|\le\varepsilon/3$ for any $\timeVar\in[0,\timeHorizon]$, which using the argument from proving Theorem \ref{thm:kNN-survival} with carefully chosen points $\eta_{1},\dots,\eta_{L(\varepsilon)}$ is sufficient to guarantee that $\sup_{\timeVar\in[0,\timeHorizon]}|\widehat{\survEnd}^{\textsc{NN}(h)}(\timeVar|\fixedFeatureVector)-\survEnd(\timeVar|\fixedFeatureVector)|\le\varepsilon$.
\begin{lem}
Under Assumption A1, let $\fixedFeatureVector\in\text{supp}(\featureDist)$. Let $N_{x,h}$ be the number of nearest neighbors found within distance $h$ of $\fixedFeatureVector$. Then
\[
\mathbb{P}\big(\badEvent{\textsc{NN}(h)}{\text{few neighbors}}(\fixedFeatureVector)\big)\le\exp\Big(-\frac{n\featureDist(\mathcal{B}_{\fixedFeatureVector,h})}{8}\Big).
\]
\end{lem}

\begin{proof}
Since $N_{x,h}\sim\text{Binomial}(n,\featureDist(\mathcal{B}_{\fixedFeatureVector,h}))$,
the claim follows from applying a Chernoff bound for the binomial distribution \citep[Section 2.1]{georgehc_thesis}.
\end{proof}
\begin{lem}
\label{lem:h-near-bad-T}
Under Assumptions A1--A3, let $\fixedFeatureVector\in\text{supp}(\featureDist)$. We have
\begin{align*}
 & \mathbb{P}\big(\badEvent{\textsc{NN}(h)}{\text{bad }\timeHorizon}(\fixedFeatureVector)\,\big|\,[\badEvent{\textsc{NN}(h)}{\text{few neighbors}}(\fixedFeatureVector)]^{c}\big)\\
 & \quad \le\exp\Big(-\frac{n\featureDist(\mathcal{B}_{\fixedFeatureVector,h})\theta}{16}\Big).
\end{align*}
\end{lem}

\begin{proof}
By conditioning on $N_{x,h}=k$ for any $k\in\{1,\dots,n\}$, then a proof similar to that of Lemma~\ref{lem:kNN-bad-T} yields
\[
\mathbb{P}\Big(d_{\mathcal{N}_{\textsc{NN}(h)}(\fixedFeatureVector)}^{+}(\timeHorizon)\le\frac{k\theta}{2}\,\Big|\,N_{x,h}=k\Big)\le\exp\Big(-\frac{k\theta}{8}\Big).
\]
We now use a worst-case argument that appears many times in later proofs. Let $k_{0}$ be the smallest integer larger than $\frac{1}{2}n\featureDist(\mathcal{B}_{\fixedFeatureVector,h})$. Then
\begin{align*}
 & \mathbb{P}\big(\badEvent{\textsc{NN}(h)}{\text{bad }\timeHorizon}(\fixedFeatureVector)\,\big|\,[\badEvent{\textsc{NN}(h)}{\text{few neighbors}}(\fixedFeatureVector)]^{c}\big)\\
 & \quad =\mathbb{P}\Big(d_{\mathcal{N}_{\textsc{NN}(h)}(\fixedFeatureVector)}^{+}(\timeHorizon)\le\frac{N_{x,h}\theta}{2}\,\Big|\,N_{x,h}\ge k_{0}\Big)\\
 & \quad =\frac{\begin{bmatrix}\sum_{k=k_{0}}^{n}\mathbb{P}(N_{x,h}=k)\quad\qquad\qquad\qquad\\
\times\mathbb{P}\big(d_{\mathcal{N}_{\textsc{NN}(h)}(\fixedFeatureVector)}^{+}(\timeHorizon)\le\frac{k\theta}{2}\,\big|\,N_{x,h}=k\big)
\end{bmatrix}}{\mathbb{P}(N_{x,h}\ge k_{0})}\\
 & \quad \le\frac{\sum_{k=k_{0}}^{n}\mathbb{P}(N_{x,h}=k)\exp\big(-\frac{k_{0}\theta}{8}\big)}{\mathbb{P}(N_{x,h}\ge k_{0})}\\
 & \quad =\exp\big(-\frac{k_{0}\theta}{8}\big)\\
 & \quad \le\exp\big(-\frac{n\featureDist(\mathcal{B}_{\fixedFeatureVector,h})\theta}{16}\big).\qedhere
\end{align*}
\end{proof}
\begin{lem}
Under Assumptions A1--A3, let $\fixedFeatureVector\in\text{supp}(\featureDist)$ and $\timeVar\in[0,\timeHorizon]$. When bad events $\badEvent{\textsc{NN}(h)}{\text{few neighbors}}(\fixedFeatureVector)$ and $\badEvent{\textsc{NN}(h)}{\text{bad }\timeHorizon}(\fixedFeatureVector)$ do not happen,
\begin{align*}
 & |V_2(\timeVar|\fixedFeatureVector)|\\
 & \le\frac{2}{N_{x,h}\theta^{2}}\\
 & \quad+\frac{2}{\theta^{2}}\sup_{s\in[0,\timeHorizon]}|\obsEnd(s|\fixedFeatureVector)-\mathbb{E}[\estObsEnd^{\textsc{NN}(h)}(s|\fixedFeatureVector)\,|\,N_{x,h}]|\\
 & \quad+\frac{2}{\theta^{2}}\sup_{s\ge0}|\estObsEnd^{\textsc{NN}(h)}(s|\fixedFeatureVector)-\mathbb{E}[\estObsEnd^{\textsc{NN}(h)}(s|\fixedFeatureVector)\,|\,N_{x,h}]|,
\end{align*}
where the RHS is a function of random variable $N_{x,h}$ (which is greater than 0 since bad event $\badEvent{\textsc{NN}(h)}{\text{few neighbors}}(\fixedFeatureVector)$ does not happen).
\end{lem}

\begin{proof}
See the proof of Lemma \ref{lem:kNN-R2-bound} as given in Appendix \ref{subsec:pf-lem-kNN-R2-bound}, where we replace $\setS=\neighborsKNN(\fixedFeatureVector)$ with $\setS=\neighborsNNh(\fixedFeatureVector)$, $k$ with $N_{x,h}$, and bad event $\badEvent{k\textsc{-NN}}{\text{bad }\timeHorizon}(\fixedFeatureVector)$ with $\badEvent{\textsc{NN}(h)}{\text{bad }\timeHorizon}(\fixedFeatureVector)$. Also instead of using expectation $\mathbb{E}[\,\cdot\,|\widetilde{\featureVar}]$ (i.e., conditioning on the $(k+1)$-st nearest neighbor), we use $\mathbb{E}[\,\cdot\,|N_{x,h}]$.
\end{proof}
\begin{lem}
\label{lem:h-near-bad-EDF}
Under Assumptions A1--A3, let $\fixedFeatureVector\in\text{supp}(\featureDist)$. We have
\begin{align*}
 & \mathbb{P}\big(\badEvent{\textsc{NN}(h)}{\text{bad EDF}}(\fixedFeatureVector)\,\big|\,[\badEvent{\textsc{NN}(h)}{\text{few neighbors}}(\fixedFeatureVector)]^{c}\big)\\
 & \le2\exp\Big(-\frac{n\featureDist(\mathcal{B}_{\fixedFeatureVector,h})\varepsilon^{2}\theta^{4}}{1296}\Big).
\end{align*}
\end{lem}
\begin{proof}
By conditioning on $N_{x,h}=k$ for any $k\in\{1,\dots,n\}$, then $1-\estObsEnd^{\textsc{NN}(h)}(s|\fixedFeatureVector)$ is an empirical distribution with samples drawn i.i.d.~from CDF $1-{\mathbb{P}(\obsVar>s\,|\,\featureVar\in\mathcal{B}_{\fixedFeatureVector,h})}$. By the DKW inequality,
\[
\mathbb{P}\big(\badEvent{\textsc{NN}(h)}{\text{bad EDF}}(\fixedFeatureVector)\,\big|\,N_{x,h}=k\big)\le2\exp\Big(-\frac{k\varepsilon^{2}\theta^{4}}{648}\Big).
\]
A worst-case argument similar to the one in the ending of Lemma~\ref{lem:h-near-bad-T}'s proof says that ${\mathbb{P}\big(\badEvent{\textsc{NN}(h)}{\text{bad EDF}}(\fixedFeatureVector)\,\big|\,N_{x,h}>\frac{1}{2}n\featureDist(\mathcal{B}_{\fixedFeatureVector,h})\big)}$ satisfies the above inequality with $k$ replaced by $\frac{1}{2}n\featureDist(\mathcal{B}_{\fixedFeatureVector,h})$.
\end{proof}
\begin{lem}
\label{lem:h-near-bad-V1}
Under Assumptions A1--A3, let $\fixedFeatureVector\in\text{supp}(\featureDist)$,
$\timeVar\in[0,\timeHorizon]$, and $\varepsilon\in(0,1)$. We have
\begin{align*}
 & \mathbb{P}\big(\badEvent{\textsc{NN}(h)}{\text{bad }V_1}(t,x)\,\big|\,[\badEvent{\textsc{NN}(h)}{\text{few neighbors}}(\fixedFeatureVector)]^{c}\big)\\
 & \quad \le2\exp\Big(-\frac{n\featureDist(\mathcal{B}_{\fixedFeatureVector,h})\varepsilon^{2}\theta^{2}}{324}\Big).
\end{align*}
\end{lem}
\begin{proof}
Note that $V_1(\timeVar|\fixedFeatureVector)$ and $\mathbb{E}[V_1(\timeVar|\fixedFeatureVector)\,|\,N_{x,h}]$ can both be written as functions of random variable $N_{x,h}$, provided that $N_{x,h}$ is positive. Specifically,
\[
V_1(\timeVar|\fixedFeatureVector) = \frac{1}{N_{x,h}} \sum_{\ell=1}^{N_{x,h}} \xi_\ell,
\]
where random variables $\xi_1,\dots,\xi_{N_{x,h}}$ are sampled i.i.d.~from the same distribution as random variable $-\frac{\eventVar\ind\{\obsVar\le \timeVar\}}{\obsEnd(\obsVar|\fixedFeatureVector)}$ (where feature vector~$X$ is sampled from $\featureDist$ restricted to ball $\mathcal{B}_{\fixedFeatureVector,h}$, and observed time~$\obsVar$ and censoring indicator~$\eventVar$ as sampled as usual conditioned on~$X$). Each $\xi_\ell$ is bounded in $[-\frac{1}{\obsEnd(\timeVar|\fixedFeatureVector)},0]$ and has expectation
\begin{align}
\overline{\xi}(N_{x,h})
&:=\mathbb{E}_{\obsVar,\eventVar}\Big[-\frac{\eventVar\ind\{\obsVar\le \timeVar\}}{\obsEnd(\obsVar|\fixedFeatureVector)}\,\Big|\,\featureVar\in\mathcal{B}_{\fixedFeatureVector,h}\Big] \nonumber \\
&\hspace{.25em}=\mathbb{E}[V_1(\timeVar|\fixedFeatureVector)\,|\,N_{x,h}]. \nonumber
\end{align}
Thus, using Hoeffding's inequality, for any $k \in \{1,\dots,n\}$,
\begin{align*}
 & \mathbb{P}\Big(|V_1(\timeVar|\fixedFeatureVector)-\mathbb{E}[V_1(\timeVar|\fixedFeatureVector)\,|\,N_{x,h}]|\ge\frac{\varepsilon}{18}\,\Big|\,N_{x,h}=k\Big) \\
 & \quad = \mathbb{P}\Big(\Big|\frac{1}{N_{x,h}} \sum_{\ell=1}^{N_{x,h}} \xi_\ell - \overline{\xi}(N_{x,h})\Big|\ge\frac{\varepsilon}{18}\,\Big|\,N_{x,h}=k\Big) \\
 & \quad \le2\exp\Big(-\frac{k\varepsilon^{2}[\obsEnd(\timeVar|\fixedFeatureVector)]^2}{162}\Big)\le2\exp\Big(-\frac{k\varepsilon^{2}\theta^{2}}{162}\Big).
\end{align*}
A worst-case argument similar to the one in the ending of Lemma~\ref{lem:h-near-bad-T}'s proof yields
\begin{align*}
 & \mathbb{P}\big(\badEvent{\textsc{NN}(h)}{\text{bad }V_1}(t,x)\,\big|\,[\badEvent{\textsc{NN}(h)}{\text{few neighbors}}(\fixedFeatureVector)]^{c}\big)\\
 & \quad \le\exp\big(-\frac{n\featureDist(\mathcal{B}_{\fixedFeatureVector,h})\varepsilon^{2}\theta^{2}}{324}\big).\qedhere
\end{align*}
\end{proof}
Then when none of the bad events happen,
\begin{align*}
 & |\log\widehat{\survEnd}^{\textsc{NN}(h)}(\timeVar|\fixedFeatureVector)-\log \survEnd(\timeVar|\fixedFeatureVector)|\\
 & \le|V_1(\timeVar|\fixedFeatureVector)-\mathbb{E}[V_1(\timeVar|\fixedFeatureVector)\,|\,N_{x,h}]|\\
 & \quad+|\mathbb{E}[V_1(\timeVar|\fixedFeatureVector)\,|\,N_{x,h}]-\log \survEnd(\timeVar|\fixedFeatureVector)|+\frac{2}{N_{x,h}\theta^{2}}\\
 & \quad+\frac{2}{\theta^{2}}\sup_{s\in[0,\timeHorizon]}|\obsEnd(s|\fixedFeatureVector)-\mathbb{E}[\estObsEnd^{\textsc{NN}(h)}(s|\fixedFeatureVector)\,|\,N_{x,h}]|\\
 & \quad+\frac{2}{\theta^{2}}\sup_{s\ge0}|\estObsEnd^{\textsc{NN}(h)}(s|\fixedFeatureVector)-\mathbb{E}[\estObsEnd^{\textsc{NN}(h)}(s|\fixedFeatureVector)\,|\,N_{x,h}]|,\\
 & \quad+V_3(\timeVar|\fixedFeatureVector).
\end{align*}
The 1st and 5th terms on the RHS are at most $\frac{\varepsilon}{18}$ since bad events $\badEvent{\textsc{NN}(h)}{\text{bad }V_1}(t,x)$ and $\badEvent{\textsc{NN}(h)}{\text{bad EDF}}(\fixedFeatureVector)$ do not happen (these bad events also rely on $\badEvent{\textsc{NN}(h)}{\text{few neighbors}}(\fixedFeatureVector)$ not happening so that $N_{x,h}>0$). The theorem assumes that $n\ge\frac{144}{\varepsilon\theta^{2}\featureDist(\mathcal{B}_{\fixedFeatureVector,h})}$, so the 3rd term is at most $\frac{2}{N_{x,h}\theta^{2}}<\frac{4}{n\featureDist(\mathcal{B}_{\fixedFeatureVector,h})\theta^{2}}\le\frac{\varepsilon}{36}<\frac{\varepsilon}{18}$. The 2nd, 4th, and 6th terms can be bounded in a similar manner as we did for the $k$-NN estimator.
\begin{lem}
Under Assumptions A1--A4 $($this lemma uses H\"{o}lder continuity of $\censEnd(\timeVar|\cdot)\survDensity(\timeVar|\cdot))$, let $\fixedFeatureVector\in\text{supp}(\featureDist)$, $\timeVar\in[0,\timeHorizon]$, and $\varepsilon\in(0,1)$. If bad event $\badEvent{\textsc{NN}(h)}{\text{few neighbors}}(\fixedFeatureVector)$ does not happen, and the threshold distance satisfies $h\le[\frac{\varepsilon\theta}{18(\lips_{\textsc{\tiny{\survVar}}}\timeHorizon+(\survDensity^*{\lips}_{\textsc{\tiny{\censVar}}}\timeHorizon^2)/2)}]^{1/\holderIndex}$, then
\[
|\mathbb{E}[V_1(\timeVar|\fixedFeatureVector)\,|\,N_{x,h}]-\log \survEnd(\timeVar|\fixedFeatureVector)|\le\frac{\varepsilon}{18}.
\]
\end{lem}
\begin{proof}
See the proof for Lemma \ref{lem:kNN-R1-bias} as given in Appendix \ref{subsec:pf-lem-kNN-R1-bias}. The main change is that we do not have to condition on the $(k+1)$-st nearest neighbor of $\fixedFeatureVector$. Instead, conditioning on $N_{x,h}=k$ for integer $k$ in $(\frac{1}{2}n\featureDist(\mathcal{B}_{\fixedFeatureVector,h}),n]$, then $V_1(\timeVar|\fixedFeatureVector)$ is the average of $k$ i.i.d.~bounded random variables each with expectation $\mathbb{E}_{\obsVar,\eventVar}[-\frac{\eventVar\ind\{\obsVar\le \timeVar\}}{\obsEnd(\obsVar|\fixedFeatureVector)}\,|\,\featureVar\in\mathcal{B}_{\fixedFeatureVector,h}]$.
\end{proof}

\begin{lem}
Under Assumptions A1--A4 $($this lemma uses H\"{o}lder continuity of $\obsEnd(\timeVar|\cdot))$, let $\fixedFeatureVector\in\text{supp}(\featureDist)$ and $\varepsilon\in(0,1)$. If bad event $\badEvent{\textsc{NN}(h)}{\text{few neighbors}}(\fixedFeatureVector)$ does not happen, and the threshold distance satisfies $h\le[\frac{\varepsilon\theta^{2}}{36(\lips_{\textsc{\tiny{\survVar}}}+\lips_{\textsc{\tiny{\censVar}}})\timeHorizon)}]^{1/\holderIndex}$, then
\[
\frac{2}{\theta^{2}}\sup_{s\in[0,\timeHorizon]}|\obsEnd(s|\fixedFeatureVector)-\mathbb{E}[\estObsEnd^{\textsc{NN}(h)}(s|\fixedFeatureVector)\,|\,N_{x,h}]|\le\frac{\varepsilon}{18}.
\]
\end{lem}

\begin{proof}
See the proof for Lemma \ref{lem:kNN-R2-bias} as given in Appendix \ref{subsec:pf-lem-kNN-R2-bias}. Once again, the main change is that we do not have to condition on the $(k+1)$-st nearest neighbor of $\fixedFeatureVector$. Instead, conditioning on $N_{x,h}=k$ for integer $k$ in $(\frac{1}{2}n\featureDist(\mathcal{B}_{\fixedFeatureVector,h}),n]$, then $1-\estObsEnd^{\textsc{NN}(h)}(s|\fixedFeatureVector)$ is an empirical distribution constructed based on i.i.d.~samples from CDF ${1-\mathbb{E}[\estObsEnd^{\textsc{NN}(h)}(s|\fixedFeatureVector)\,|\,N_{x,h}=k]}={1-\mathbb{P}(\obsVar>s\,|\,\featureVar\in\mathcal{B}_{\fixedFeatureVector,h})}$.
\end{proof}

\begin{lem}
Under Assumptions A1--A3, let $\fixedFeatureVector\in\text{supp}(\featureDist)$, $\timeVar\in[0,\timeHorizon]$, and $\varepsilon\in(0,1)$. If bad events $\badEvent{\textsc{NN}(h)}{\text{few neighbors}}(\fixedFeatureVector)$ and $\badEvent{\textsc{NN}(h)}{\text{bad }\timeHorizon}(\fixedFeatureVector)$ do not happen, and $n\ge\frac{144}{\varepsilon\theta^{2}\featureDist(\mathcal{B}_{\fixedFeatureVector,h})}$, then $|V_3(\timeVar|\fixedFeatureVector)|\le {\varepsilon}/{18}.$
\end{lem}

\begin{proof}
See the proof of Lemma \ref{lem:kNN-R3} as given in Appendix \ref{subsec:pf-lem-kNN-R3},
where we replace $\setS=\neighborsKNN(\fixedFeatureVector)$ with $\setS=\mathcal{N}_{\textsc{NN}(h)}(\fixedFeatureVector)$, $k$
with $N_{x,h}$, and bad event $\badEvent{k\textsc{-NN}}{\text{bad }\timeHorizon}(\fixedFeatureVector)$ with $\badEvent{\textsc{NN}(h)}{\text{bad }\timeHorizon}(\fixedFeatureVector)$.
\end{proof}

\section{Proof of Theorem~\ref{thm:kernel-survival}}
\label{sec:pf-thm-kernel-survival}

First off, we state a longer version of the kernel pointwise theorem that includes a strong consistency result. This is the version of the theorem we prove in this section.

\begin{thm}[Kernel pointwise guarantees]
\label{thm:kernel-survival-long}
Under Assumptions A1--A5, let $\varepsilon\in(0,1)$ be a user-specified error tolerance. Suppose that the threshold distance satisfies $h\in(0,\frac1\stanDistThresh(\frac{\varepsilon\theta}{18\Lips_K})^{1/\holderIndex}]$, and the number of training data satisfies $n\ge\frac{144}{\varepsilon\theta^{2}\featureDist(\mathcal{B}_{\fixedFeatureVector,\stanDistThresh h})\kappa}$. For any $\fixedFeatureVector\in\text{supp}(\featureDist)$,
\begin{align*}
 & \mathbb{P}\Big(\sup_{\timeVar\in[0,\timeHorizon]}|\widehat{\survEnd}^K(\timeVar|\fixedFeatureVector;h)-\survEnd(\timeVar|\fixedFeatureVector)|>\varepsilon\Big) \nonumber \\
 & \;\le\exp\Big(-\frac{n\featureDist(\mathcal{B}_{\fixedFeatureVector,\stanDistThresh h})\theta}{16}\Big) + \exp\Big(-\frac{n\featureDist(\mathcal{B}_{\fixedFeatureVector,\stanDistThresh h})}{8}\Big) \nonumber \\
 & \;\quad+\frac{216}{\varepsilon\theta^2\kappa}\exp\Big(-\frac{n\featureDist(\mathcal{B}_{\fixedFeatureVector,\stanDistThresh h})\varepsilon^{2}\theta^{4}\kappa^4}{11664}\Big) \nonumber \\
 & \;\quad+\frac{8}{\varepsilon}\exp\Big(-\frac{n\featureDist(\mathcal{B}_{\fixedFeatureVector,\stanDistThresh h})\varepsilon^{2}\theta^{2}\kappa^2}{324}\Big). \nonumber
\end{align*}
Moreover, if there exist constants $p_{\min}>0$, $d>0$, and $r^*>0$ such that $\featureDist(\mathcal{B}_{x,r})\ge p_{\min}r^d$ for all $r\in(0,r^*]$, then we get the same strong consistency behavior as in Theorem~\ref{thm:h-near-survival} with the numbers $c_1'$, $c_2$ and $c_3$ replaced by
$c_1''=\Theta\big( \frac{1}{\stanDistThresh (\theta\Lips_K\kappa^2)^{2/(2\holderIndex+d)}} \big)$, $c_2''=\Theta\big( \frac{1}{(\theta\Lips_K)^{d/(5\holderIndex+2d)}\kappa^{(d-2\holderIndex)/(5\holderIndex+2d)}} \big)$, and $c_3''=\Theta\big( \frac{(\Lips_K)^{d/(2\holderIndex+d)}}{\theta^{(4\holderIndex+d)/(2\holderIndex+d)}\kappa^{4\holderIndex/(2\holderIndex+d)}} \big)$.
\end{thm}
For the kernel estimator, there is a fair amount more notation to keep track of. To keep the equations from becoming unwieldy, we adopt the following abbreviations. First off, the training subjects with nonzero kernel weight are precisely the ones with feature vectors landing in the ball $\mathcal{B}_{\fixedFeatureVector,\stanDistThresh h}$. We denote the number of these subjects as $N:=N_{x,\stanDistThresh h}\sim\text{Binomial}(n,\featureDist(\mathcal{B}_{\fixedFeatureVector,\stanDistThresh h}))$. We denote their data points as $(\featureVar_{(1)},\obsVar_{(1)},\eventVar_{(1)})$, $\dots$, $(\featureVar_{(N)},\obsVar_{(N)},\eventVar_{(N)})$; we treat the ordering of these points as uniform at random (the points could be thought of as being generated i.i.d.~first by sampling a feature vector $X$ from $\featureDist$ restricted to $\mathcal{B}_{\fixedFeatureVector,\stanDistThresh h}$, and then sampling observed time~$\obsVar$ and censoring indicator~$\eventVar$ as usual). We use the abbreviations $K_{(i)}:=K(\frac{\rho(\fixedFeatureVector,\featureVar_{(i)})}{h})$, $d_{K}^{+}(\timeVar):=\sum_{j=1}^N K_{(j)}\ind\{\obsVar_{(j)}>\timeVar\}$, and
\[
\estObsEnd^{K}(\timeVar):=\frac{d_{K}^{+}(\timeVar)}{\sum_{\ell=1}^{N}K_{(\ell)}}=\sum_{j=1}^{N}\frac{K_{(j)}}{\sum_{\ell=1}^{N}K_{(\ell)}}\ind\{\obsVar_{(j)}>t\}.
\]
Let $\mathbb{E}_{\{\obsVar\}}$ denote the expectation only over the nearest neighbors' observed times $\obsVar_{(1)},\dots,\obsVar_{(N)}$ (so we are conditioning on $N,\featureVar_{(1)},\dots \featureVar_{(N)}$). Similarly, we let $\mathbb{E}_{\{\obsVar,\eventVar\}}$ denote the expectation only over only the nearest neighbors' observed times and censoring indicators $(\obsVar_{(1)},\eventVar_{(1)}),\dots,(\obsVar_{(N)},\eventVar_{(N)})$.

Using the same reasoning as for the $k$-NN estimator,
\begin{align*}
 \log\widehat{\survEnd}^{K}(\timeVar|x;h)
 & =\log\prod_{i=1}^{N}\Big(\frac{d_{K}^{+}(\obsVar_{(i)})}{d_{K}^{+}(\obsVar_{(i)})+K_{(i)}}\Big)^{\eventVar_{(i)}\ind\{\obsVar_{(i)}\le \timeVar\}}\\
 & =W_1(\timeVar|\fixedFeatureVector)+W_2(\timeVar|\fixedFeatureVector)+W_3(\timeVar|\fixedFeatureVector),
\end{align*}
where
\begin{align*}
 & W_1(\timeVar|\fixedFeatureVector) =-\sum_{i=1}^{N}\frac{K_{(i)}\eventVar_{(i)}\ind\{\obsVar_{(i)}\le \timeVar\}\frac{1}{\obsEnd(\obsVar_{(i)}|\fixedFeatureVector)}}{\sum_{j=1}^{N}K_{(j)}}, \\
 & W_2(\timeVar|\fixedFeatureVector)\\
 & =-\sum_{i=1}^{N}\frac{K_{(i)}\eventVar_{(i)}\ind\{\obsVar_{(i)}\le \timeVar\}\big[\frac{\sum_{\ell=1}^{N}K_{(\ell)}}{d_{K}^{+}(\obsVar_{(i)})+K_{(i)}}-\frac{1}{\obsEnd(\obsVar_{(i)}|\fixedFeatureVector)}\big]}{\sum_{j=1}^{N}K_{(j)}},
\end{align*}\begin{align*}
W_3(\timeVar|\fixedFeatureVector) & =-\sum_{i=1}^{N}\eventVar_{(i)}\ind\{\obsVar_{(i)}\le \timeVar\}\sum_{\ell=2}^{\infty}\frac{1}{\ell(\frac{d_{K}^{+}(\obsVar_{(i)})}{K_{(i)}}+1)^{\ell}}.
\end{align*}
The bad events are as follows:
\begin{itemize}[leftmargin=*,topsep=0pt,itemsep=0ex,partopsep=1ex,parsep=1ex]
\item $\badEvent{\textsc{NN}(\stanDistThresh h)}{\text{few neighbors}}(\fixedFeatureVector)$ is the same bad event as for the fixed-radius NN estimator except using threshold distance $\stanDistThresh h$ instead of $h$
\item $\badEvent{\textsc{NN}(\stanDistThresh h)}{\text{bad }\timeHorizon}(\fixedFeatureVector)$ is another bad event borrowed from the fixed-radius NN estimator
\item $\badEvent{\text{kernel}}{\text{bad weighted EDF}}(\fixedFeatureVector):=\big\{\sup_{s\ge0}\big|\estObsEnd^{K}(\timeVar)-\mathbb{E}_{\{\obsVar\}}\big[\estObsEnd^{K}(\timeVar)\big]\big|>\frac{\varepsilon\theta^{2}K(\stanDistThresh)}{36K(0)}\big\}$ is analogous to event $\badEvent{\textsc{NN}(h)}{\text{bad EDF}}(\fixedFeatureVector)$
\item $\badEvent{\text{kernel}}{\text{bad }W_1}(t,x):=\{|W_1(\timeVar|\fixedFeatureVector)-\mathbb{E}_{\{\obsVar,\eventVar\}}[W_1(\timeVar|\fixedFeatureVector)]|\ge\frac{\varepsilon}{18}\}$ is analogous to event $\badEvent{\textsc{NN}(h)}{\text{bad }V_1}(t,x)$, and as before we ask that this holds at specific points $t=\eta_{1},\dots,\eta_{L(\varepsilon)}$ (using the same construction as in the proof of Theorem~\ref{thm:kNN-survival})
\end{itemize}
We show how to prevent bad events $\badEvent{\text{kernel}}{\text{bad weighted EDF}}(\fixedFeatureVector)$ and $\badEvent{\text{kernel}}{\text{bad }W_1}(t,x)$ in the next two lemmas.
\begin{lem}
\label{lem:kernel-bad-wEDF}
Under Assumptions A1--A3 and A5, let $\fixedFeatureVector\in\text{supp}(\featureDist)$ and $\varepsilon\in(0,1)$. Then
\begin{align*}
 & \mathbb{P}\big(\badEvent{\text{kernel}}{\text{bad weighted EDF}}(\fixedFeatureVector)\,\big|\,[\badEvent{\textsc{NN}(\stanDistThresh h)}{\text{few neighbors}}(\fixedFeatureVector)]^{c}\big)\\
 & \le \frac{216 K(0)}{\varepsilon\theta^{2}K(\stanDistThresh)}
 \exp\Big(-\frac{n\featureDist(\mathcal{B}_{\fixedFeatureVector,\stanDistThresh h})\varepsilon^{2}\theta^{4}K^{4}(\stanDistThresh)}{11664 K^{4}(0)}\Big).
\end{align*}
\end{lem}
\begin{proof}
Conditioned on $N,\featureVar_{(1)},\dots,\featureVar_{(N)}$ with $N$ positive (recall that $K_{(i)}$ depends on $\featureVar_{(i)}$), then $\estObsEnd^K(\timeVar)$ appears to be constructed from independent weighted samples, where the weights are deterministic and, moreover, $1-\estObsEnd^K(\timeVar)$ is precisely a weighted empirical distribution with expectation $1-\mathbb{E}_{\{\obsVar\}}[\estObsEnd^{K}(\timeVar)]=1-\sum_{j=1}^N\frac{K_{(j)}}{\sum_{\ell=1}^N K_{(\ell)}}\obsEnd(t|X_{(\ell)})$, where $\obsEnd(\cdot|X_{(\ell)})$ is continuous as a consequence of Assumption~A2. Thus, by conditioning on the event
\[
\mathcal{A}:=\{N=k,\featureVar_{(1)}=x_{(1)},\dots,\featureVar_{(k)}=x_{(k)}\}
\]
for any integer $k\in\{1,\dots,n\}$, and any choices for $x_{(1)},\dots,x_{(k)}\in\mathcal{B}_{\fixedFeatureVector,\stanDistThresh h}$, we can then apply Proposition~\ref{lem:weighted-edf} (with $\ell=k$ and noting that $\sum_{i=1}^k \weightVar_i \ge k K(\stanDistThresh)$ and
$\sum_{i=1}^k \weightVar_i^2 \le k K^2(0)$) to get
\begin{align*}
 & \mathbb{P}\Big(\sup_{\timeVar\ge0}|\estObsEnd^{K}(\timeVar)-\mathbb{E}_{\{\obsVar\}}[\estObsEnd^{K}(\timeVar)]|>\frac{\varepsilon\theta^{2}K(\stanDistThresh)}{36K(0)}\,\Big|\,\mathcal{A}\Big)\\
 & \quad \le \frac{216 K(0)}{\varepsilon\theta^{2}K(\stanDistThresh)}\exp\Big(-\frac{k \varepsilon^{2}\theta^{4}K^{4}(\stanDistThresh)}{5832 K^{4}(0)}\Big).
\end{align*}
This inequality holds for all $\fixedFeatureVector_{(1)},\dots,\fixedFeatureVector_{(k)}\in\mathcal{B}_{x,\stanDistThresh h}$, so we can marginalize over $\featureVar_{(1)},\dots,\featureVar_{(k)}$ to get:
\begin{align*}
 & \mathbb{P}\Big(\sup_{\timeVar\ge0}|\estObsEnd^{K}(\timeVar)-\mathbb{E}_{\{\obsVar\}}[\estObsEnd^{K}(\timeVar)]|>\frac{\varepsilon\theta^{2}K(\stanDistThresh)}{36K(0)}\,\Big|\,N=k \Big)\\
 & \quad \le \frac{216 K(0)}{\varepsilon\theta^{2}K(\stanDistThresh)}\exp\Big(-\frac{k \varepsilon^{2}\theta^{4}K^{4}(\stanDistThresh)}{5832 K^{4}(0)}\Big).
\end{align*}
Finally, conditioned on $[\badEvent{\textsc{NN}(\stanDistThresh h)}{\text{few neighbors}}(\fixedFeatureVector)]^{c}=\{N>\frac{1}{2}n\featureDist(\mathcal{B}_{\fixedFeatureVector,\stanDistThresh h})\}$, a worst-case argument similar to the one used at the end of Lemma~\ref{lem:h-near-bad-T}'s proof yields the claim.
\end{proof}
\begin{lem}
\label{lem:kernel-bad-W1}
Under Assumptions A1--A3 and A5, let $\fixedFeatureVector\in\text{supp}(\featureDist)$, $\timeVar\in[0,\timeHorizon]$, and $\varepsilon\in(0,1)$. We have
\begin{align*}
 & \mathbb{P}\big(\badEvent{\text{kernel}}{\text{bad }W_1}(t,x)\,\big|\,[\badEvent{\textsc{NN}(\stanDistThresh h)}{\text{few neighbors}}(\fixedFeatureVector)]^{c}\big)\\
 & \quad \le2\exp\Big(-\frac{n\featureDist(\mathcal{B}_{\fixedFeatureVector,\stanDistThresh h})\varepsilon^{2}\theta^{2}K^{2}(\stanDistThresh)}{324 K^{2}(0)}\Big).
\end{align*}
\end{lem}
\begin{proof}
The proof is similar to that of Lemma~\ref{lem:h-near-bad-V1}. Note that
\[
W_1(\timeVar|\fixedFeatureVector) = \sum_{i=1}^{N}\underbrace{-\frac{K_{(i)}}{\sum_{j=1}^{k}K_{(j)}}\frac{\eventVar_{(i)}\ind\{\obsVar_{(i)}\le \timeVar\}}{\obsEnd(\obsVar_{(i)}|\fixedFeatureVector)}}_{\text{bounded in }\big[-\big(\frac{K_{(i)}}{\sum_{j=1}^{N}K_{(j)}}\big)\frac{1}{\obsEnd(\timeVar|\fixedFeatureVector)},\,0\big]}.
\]
Conditioned on $N,\featureVar_{(1)},\dots,\featureVar_{(N)}$ with $N$ positive, then $W_1(\timeVar|\fixedFeatureVector)$ becomes a sum over independent random variables. Meanwhile, $\mathbb{E}_{\{\obsVar,\eventVar\}}[W_1(\timeVar|\fixedFeatureVector)]$ is precisely the expectation of $W_1(\timeVar|\fixedFeatureVector)$ conditioned on $N,\featureVar_{(1)},\dots,\featureVar_{(N)}$. Hence, by conditioning on the event
\[
\mathcal{A}:=\{N=k,\featureVar_{(1)}=x_{(1)},\dots,\featureVar_{(k)}=x_{(k)}\}
\]
for any $k\in\{1,\dots,n\}$, and any choices of $x_{(1)},\dots,x_{(k)}\in\mathcal{B}_{\fixedFeatureVector,\stanDistThresh h}$, and denoting $\weightVar_{(i)}:=K(\frac{\rho(\fixedFeatureVector,x_{(i)})}h)$, Hoeffding's inequality gives
\begin{align*}
 & \mathbb{P}\Big(|W_1(\timeVar|\fixedFeatureVector)-\mathbb{E}_{\{\obsVar,\eventVar\}}[W_1(\timeVar|\fixedFeatureVector)]|\ge\frac{\varepsilon}{18}\,\Big|\,\mathcal{A}\Big)\\
 & \quad \le2\exp\Big(-\frac{\varepsilon^{2}(\sum_{j=1}^{k}\weightVar_{(j)})^{2}[\obsEnd(\timeVar|\fixedFeatureVector)]^2}{162\sum_{i=1}^{k}\weightVar_{(i)}^{2}}\Big)\\
 & \quad \le2\exp\Big(-\frac{k\varepsilon^{2}K^{2}(\stanDistThresh)[\obsEnd(\timeVar|\fixedFeatureVector)]^2}{162K^{2}(0)}\Big)\\
 & \quad \le2\exp\Big(-\frac{k\varepsilon^{2}K^{2}(\stanDistThresh)\theta^{2}}{162K^{2}(0)}\Big).
\end{align*}
We complete the proof the same way as in Lemma \ref{lem:kernel-bad-wEDF}'s proof, marginalizing over $\featureVar_{(1)},\dots,\featureVar_{(k)}$ and using a worst-case analysis argument to replace $k$ with $\frac12 n\featureDist(\mathcal{B}_{\fixedFeatureVector,\stanDistThresh h})$.
\end{proof}
Now that we have the bad events sorted out, the argument for why them not happening guarantees that $\sup_{\timeVar\in[0,\timeHorizon]}|\widehat{\survEnd}^{K}(\timeVar|x;h)-\survEnd(\timeVar|\fixedFeatureVector)|\le\varepsilon$ proceeds in the same manner as for the $k$-NN and fixed-radius NN analyses. We first upper-bound $|W_2(\timeVar|\fixedFeatureVector)|$.
\begin{lem}
\label{lem:kernel-W2-decomp}
Under Assumptions A1--A3 and A5, let $\fixedFeatureVector\in\text{supp}(\featureDist)$ and $\timeVar\in[0,\timeHorizon]$. When bad events $\badEvent{\textsc{NN}(\stanDistThresh h)}{\text{few neighbors}}(\fixedFeatureVector)$ and $\badEvent{\textsc{NN}(\stanDistThresh h)}{\text{bad }\timeHorizon}(\fixedFeatureVector)$ do not happen,
\begin{align*}
 |W_2(\timeVar|\fixedFeatureVector)|
 & \le\frac{2K(0)}{NK(\stanDistThresh)\theta^{2}}\\
 & \quad+\frac{2K(0)}{K(\stanDistThresh)\theta^{2}}\sup_{\timeVar\in[0,\timeHorizon]}|\mathbb{E}_{\{\obsVar\}}[\estObsEnd^{K}(\timeVar)]-\obsEnd(\timeVar|\fixedFeatureVector)|\\
 & \quad+\frac{2K(0)}{K(\stanDistThresh)\theta^{2}}\sup_{\timeVar\ge0}|\estObsEnd^{K}(\timeVar)-\mathbb{E}_{\{\obsVar\}}[\estObsEnd^{K}(\timeVar)]|.
\end{align*}
\end{lem}
Thus, when bad events $\badEvent{\textsc{NN}(\stanDistThresh h)}{\text{few neighbors}}(\fixedFeatureVector)$ and $\badEvent{\textsc{NN}(\stanDistThresh h)}{\text{bad }\timeHorizon}(\fixedFeatureVector)$ do not happen,
\begin{align}
 & |\log\widehat{\survEnd}^{K}(\timeVar|x;h)-\log \survEnd(\timeVar|\fixedFeatureVector)|\nonumber \\
 & \quad \le|W_1(\timeVar|\fixedFeatureVector)-\mathbb{E}_{\{\obsVar,\eventVar\}}[W_1(\timeVar|\fixedFeatureVector)]|\nonumber \\
 & \quad \quad+|\mathbb{E}_{\{\obsVar,\eventVar\}}[W_1(\timeVar|\fixedFeatureVector)]-\log \survEnd(\timeVar|\fixedFeatureVector)|
   +\frac{2K(0)}{NK(\stanDistThresh)\theta^{2}}\nonumber \\
 & \quad \quad+\frac{2K(0)}{K(\stanDistThresh)\theta^{2}}\sup_{\timeVar\in[0,\timeHorizon]}|\mathbb{E}_{\{\obsVar\}}[\estObsEnd^{K}(\timeVar)]-\obsEnd(\timeVar|\fixedFeatureVector)|\nonumber \\
 & \quad \quad+\frac{2K(0)}{K(\stanDistThresh)\theta^{2}}\sup_{\timeVar\ge0}|\estObsEnd^{K}(\timeVar)-\mathbb{E}_{\{\obsVar\}}[\estObsEnd^{K}(\timeVar)]|\nonumber\\
 & \quad \quad+|W_3(\timeVar|\fixedFeatureVector)|.\label{eq:kernel-decomp}
\end{align}
If we can upper-bound each of the RHS terms by ${\varepsilon}/{18}$, then we would be done since the rest of the proof is identical to the ending of the $k$-NN proof.

On the RHS of inequality \eqref{eq:kernel-decomp}, the 1st and 5th terms are at most $\frac{\varepsilon}{18}$ when bad events $\badEvent{\text{kernel}}{\text{bad }W_1}(t,x)$ and $\badEvent{\text{kernel}}{\text{bad weighted EDF}}(\fixedFeatureVector)$ do not happen. The 5th term is less than $\frac{\varepsilon}{18}$ when $n\ge\frac{144K(0)}{\varepsilon\theta^{2}\featureDist(\mathcal{B}_{\fixedFeatureVector,\stanDistThresh h})K(\stanDistThresh)}$ and $\badEvent{\textsc{NN}(\stanDistThresh h)}{\text{few neighbors}}(\fixedFeatureVector)$ does not happen (so $N>\frac12 n\featureDist(\mathcal{B}_{\fixedFeatureVector,\stanDistThresh h})$).

The rest of the section is on proving Lemma~\ref{lem:kernel-W2-decomp} and then bounding the 2nd, 4th, and 6th RHS terms (Lemmas~\ref{lem:kernel-W1-bias},~\ref{lem:kernel-Hk-bias}, and~\ref{lem:kernel-W3}).
\begin{proof}
[Proof of Lemma~\ref{lem:kernel-W2-decomp}]
When bad events $\badEvent{\textsc{NN}(\stanDistThresh h)}{\text{few neighbors}}(\fixedFeatureVector)$ and $\badEvent{\textsc{NN}(\stanDistThresh h)}{\text{bad }\timeHorizon}(\fixedFeatureVector)$ do not happen, we are guaranteed that $N$ is an integer within $(\frac{1}{2}n\featureDist(\mathcal{B}_{\fixedFeatureVector,\stanDistThresh h}),n]$, and $d_{K}^{+}(\timeHorizon)\ge K(\stanDistThresh)d_{\mathcal{N}_{\textsc{NN}(\stanDistThresh h)}}^{+}(\timeHorizon)>K(\stanDistThresh)\frac{N\theta}{2}$. Using H\"{o}lder's inequality and a bit of algebra,
\begin{align*}
 & |W_2(\timeVar|\fixedFeatureVector)|\\
 & =\Bigg|\sum_{i=1}^{N}\Big({\textstyle \frac{K_{(i)}}{\sum_{j=1}^{N}K_{(j)}}\Big)}\eventVar_{(i)}\ind\{\obsVar_{(i)}\le \timeVar\}\\
 & \quad\quad\quad\;\times\bigg[\frac{\sum_{\ell=1}^{N}K_{(\ell)}}{d_{K}^{+}(\obsVar_{(i)})+K_{(i)}}-\frac{1}{\obsEnd(\obsVar_{(i)}|\fixedFeatureVector)}\bigg]\Bigg| \\
 & \le\max_{i=1,\dots,N}\bigg|\eventVar_{(i)}\ind\{\obsVar_{(i)}\le \timeVar\}\Big[{\textstyle \frac{\sum_{\ell=1}^{N}K_{(\ell)}}{d_{K}^{+}(\obsVar_{(i)})+K_{(i)}}-\frac{1}{\obsEnd(\obsVar_{(i)}|\fixedFeatureVector)}}\Big]\bigg| \\
 & =\max_{i=1,\dots,N}\bigg|\varUpsilon_{(i)}\bigg[\varPhi(\obsVar_{(i)})+\varPsi(\obsVar_{(i)})+\frac{K_{(i)}}{\sum_{\ell=1}^{N}K_{(\ell)}}\bigg]\bigg|,
\end{align*}
where
\begin{align*}
\varUpsilon_{(i)} &:=\frac{\eventVar_{(i)}\ind\{\obsVar_{(i)}\le \timeVar\}\sum_{\ell=1}^{N}K_{(\ell)}}{(d_{K}^{+}(\obsVar_{(i)})+K_{(i)})\obsEnd(\obsVar_{(i)}|\fixedFeatureVector)}, \\
\varPhi(\timeVar) & :=\estObsEnd^{K}(\timeVar)-\mathbb{E}_{\{\obsVar\}}[\estObsEnd^{K}(\timeVar)],\\
\varPsi(\timeVar) & :=\mathbb{E}_{\{\obsVar\}}[\estObsEnd^{K}(\timeVar)]-\obsEnd(\timeVar|\fixedFeatureVector).
\end{align*}
We can keep upper-bounding to get:
\begin{align}
 |W_2(\timeVar|\fixedFeatureVector)|
 & \le\Big[\max_{i=1,\dots,N}\varUpsilon_{(i)}\Big]\sup_{s\ge0}|\varPhi(s)|\nonumber \\
 & \quad+\Big[\max_{i=1,\dots,N}\varUpsilon_{(i)}\Big]\sup_{s\in[0,\timeHorizon]}|\varPsi(s)|\nonumber \\
 & \quad+\max_{i=1,\dots,N}\frac{\varUpsilon_{(i)}K_{(i)}}{\sum_{\ell=1}^{N}K_{(\ell)}}.\label{eq:pf-lem-W2-decomp-helper1}
\end{align}
We upper-bound $\max_{i=1,\dots,N}\varUpsilon_{(i)}$ by upper-bounding $\varUpsilon_{(i)}$ for every $i$:
\begin{align}
\varUpsilon_{(i)} & =\frac{\eventVar_{(i)}\ind\{\obsVar_{(i)}\le \timeVar\}\sum_{\ell=1}^{N}K_{(\ell)}}{(d_{K}^{+}(\obsVar_{(i)})+K_{(i)})\obsEnd(\obsVar_{(i)}|\fixedFeatureVector)}\nonumber \\
 & \le\frac{\eventVar_{(i)}\ind\{\obsVar_{(i)}\le \timeVar\}\sum_{\ell=1}^{N}K_{(\ell)}}{(d_{K}^{+}(\timeVar)+K_{(i)})\obsEnd(\timeVar|\fixedFeatureVector)}\nonumber \\
 & \le\frac{\eventVar_{(i)}\ind\{\obsVar_{(i)}\le \timeVar\}\sum_{\ell=1}^{N}K_{(\ell)}}{(d_{K}^{+}(\timeHorizon)+K_{(i)})\obsEnd(\timeHorizon|\fixedFeatureVector)}\nonumber \\
 & \le\frac{\ind\{\obsVar_{(i)}\le \timeVar\}\sum_{\ell=1}^{N}K_{(\ell)}}{(d_{K}^{+}(\timeHorizon)+K_{(i)})\theta}\nonumber \\
 & \le\frac{\ind\{\obsVar_{(i)}\le \timeVar\}\sum_{\ell=1}^{N}K_{(\ell)}}{d_{K}^{+}(\timeHorizon)\theta}\nonumber \\
 & <\frac{\ind\{\obsVar_{(i)}\le \timeVar\}NK(0)}{K(\stanDistThresh)\frac{N\theta}{2}\theta}\nonumber \\
 & =\frac{2K(0)}{K(\stanDistThresh)\theta^{2}}.\label{eq:pf-lem-W2-decomp-helper2}
\end{align}
Next, we bound $\frac{\varUpsilon_{(i)}K_{(i)}}{\sum_{\ell=1}^{N}K_{(\ell)}}$:
\begin{align}
 \frac{\varUpsilon_{(i)}K_{(i)}}{\sum_{\ell=1}^{N}K_{(\ell)}}
 & =\frac{\eventVar_{(i)}\ind\{\obsVar_{(i)}\le \timeVar\}K_{(i)}}{(d_{K}^{+}(\obsVar_{(i)})+K_{(i)})\obsEnd(\obsVar_{(i)}|\fixedFeatureVector)}\nonumber \\
 & \le\frac{\eventVar_{(i)}\ind\{\obsVar_{(i)}\le \timeVar\}K_{(i)}}{(d_{K}^{+}(\timeVar)+K_{(i)})\obsEnd(\timeVar|\fixedFeatureVector)}\nonumber \\
 & \le\frac{K_{(i)}}{(d_{K}^{+}(\timeVar)+K_{(i)})\obsEnd(\timeVar|\fixedFeatureVector)}\nonumber \\
 & \le\frac{K_{(i)}}{d_{K}^{+}(\timeVar)\obsEnd(\timeVar|\fixedFeatureVector)}\nonumber \\
 & \le\frac{K_{(i)}}{d_{K}^{+}(\timeHorizon)\obsEnd(\timeHorizon|\fixedFeatureVector)}\nonumber \\
 & \le\frac{K_{(i)}}{d_{K}^{+}(\timeHorizon)\theta}\nonumber \\
 & <\frac{K(0)}{K(\stanDistThresh)\frac{N\theta}{2}\theta}\nonumber \\
 & =\frac{2K(0)}{K(\stanDistThresh)\theta^{2}N}.\label{eq:pf-lem-W2-decomp-helper3}
\end{align}
Combining inequalities \eqref{eq:pf-lem-W2-decomp-helper1}, \eqref{eq:pf-lem-W2-decomp-helper2}, and \eqref{eq:pf-lem-W2-decomp-helper3} finishes the proof.
\end{proof}
\begin{lem}
\label{lem:kernel-W1-bias}
Under Assumptions A1--A5, let $\fixedFeatureVector\in\text{supp}(\featureDist)$, $\timeVar\in[0,\timeHorizon]$, and $\varepsilon\in(0,1)$. When bad event $\badEvent{\textsc{NN}(\stanDistThresh h)}{\text{few neighbors}}(\fixedFeatureVector)$ does not hold, and the threshold distance satisfies $h\le \frac1{\stanDistThresh}[\frac{\varepsilon\theta}{18 (\lips_{\textsc{\tiny{\survVar}}}\timeHorizon+(\survDensity^*{\lips}_{\textsc{\tiny{\censVar}}}\timeHorizon^2)/2)}]^{1/\holderIndex}$,
\[
|\mathbb{E}_{\{\obsVar,\eventVar\}}[W_1(\timeVar|\fixedFeatureVector)]-\log \survEnd(\timeVar|\fixedFeatureVector)|\le\frac{\varepsilon}{18}.
\]
\end{lem}
\begin{proof}
Note that $\mathbb{E}_{\obsVar,\eventVar}[W_1(\timeVar|\fixedFeatureVector)]$ is a function of random variables $N,\featureVar_{(1)},\dots,\featureVar_{(N)}$. Since bad event $\badEvent{\textsc{NN}(\stanDistThresh h)}{\text{few neighbors}}(\fixedFeatureVector)$ does not happen, we know $N>\frac{1}{2}n\featureDist(\mathcal{B}_{\fixedFeatureVector,\stanDistThresh h})$. We have
\begin{align*}
 & \mathbb{E}_{\{\obsVar,\eventVar\}}[W_1(\timeVar|\fixedFeatureVector)]\\
 & =\mathbb{E}_{\{\obsVar,\eventVar\}}\bigg[-\sum_{i=1}^{N}\frac{K_{(i)}}{\sum_{j=1}^{N}K_{(j)}}\frac{\eventVar_{(i)}\ind\{\obsVar_{(i)}\le \timeVar\}}{\obsEnd(\obsVar_{(i)}|\fixedFeatureVector)}\bigg]\\
 & =\sum_{i=1}^{N}\frac{K_{(i)}}{\sum_{j=1}^{N}K_{(j)}}\mathbb{E}_{\obsVar_{(i)},\eventVar_{(i)}}\Big[-\frac{\eventVar_{(i)}\ind\{\obsVar_{(i)}\le \timeVar\}}{\obsEnd(\obsVar_{(i)}|\fixedFeatureVector)}\Big],
\end{align*}
where
\begin{align*}
 & \mathbb{E}_{\obsVar_{(i)},\eventVar_{(i)}}\Big[-\frac{\eventVar_{(i)}\ind\{\obsVar_{(i)}\le \timeVar\}}{\obsEnd(\obsVar_{(i)}|\fixedFeatureVector)}\Big]\\
 & =-\int_0^{\timeVar}\Big[\int_{s}^{\infty}\frac{1}{\obsEnd(s|\fixedFeatureVector)}d\mathbb{P}_{C|\featureVar=\fixedFeatureVector_{(i)}}(c)\Big]d\mathbb{P}_{\survVar|X=\featureVar_{(i)}}(s)\\
 & =-\int_0^{\timeVar}\frac{1}{\obsEnd(s|\fixedFeatureVector)}\censEnd(s|\featureVar_{(i)})\survDensity(s|\featureVar_{(i)})ds.
\end{align*}
Recall from equation~\eqref{eq:logF-integral} that
\[
\log \survEnd(\timeVar|\fixedFeatureVector)=-\int_0^{\timeVar}\frac{1}{\obsEnd(s|\fixedFeatureVector)}\censEnd(s|\fixedFeatureVector)\survDensity(s|\fixedFeatureVector)ds.
\]
Therefore,
\begin{align*}
& \mathbb{E}_{\{\obsVar,\eventVar\}}[W_1(\timeVar|\fixedFeatureVector)] - \log \survEnd(\timeVar|\fixedFeatureVector) \\
& \!=\! \sum_{i=1}^{N}\frac{K_{(i)}}{\sum_{j=1}^{N}K_{(j)}} \\
& \;\, \times \!\!\! \int_0^{\timeVar}\!\!\frac{1}{\obsEnd(s|\fixedFeatureVector)}[\censEnd(s|\fixedFeatureVector)\survDensity(s|\fixedFeatureVector) \!-\! \censEnd(s|\featureVar_{(i)})\survDensity(s|\featureVar_{(i)})]ds.
\end{align*}
Thus, using H\"{o}lder's inequality and since $\censEnd(s|\cdot)\survDensity(s|\cdot)$ is H\"{o}lder continuous with parameters $(\lips_{\textsc{\tiny{\survVar}}}+\survDensity^*\lips_{\textsc{\tiny{\censVar}}}s)$ and~$\holderIndex$,
\begin{align*}
 & \big|\mathbb{E}_{\{\obsVar,\eventVar\}}[W_1(\timeVar|\fixedFeatureVector)]-\log \survEnd(\timeVar|\fixedFeatureVector)\big|\\
 & \le\max_{i=1,\dots,N}\bigg|\int_0^{\timeVar}\frac{1}{\obsEnd(s|\fixedFeatureVector)}[\censEnd(s|\fixedFeatureVector)\survDensity(s|\fixedFeatureVector)\\
 & \quad\qquad\qquad\qquad\qquad\quad-\censEnd(s|\featureVar_{(i)})\survDensity(s|\featureVar_{(i)})]ds\bigg|\\
 & \le\max_{i=1,\dots,N}\int_0^{\timeVar}\frac{1}{\obsEnd(s|\fixedFeatureVector)}|\censEnd(s|\fixedFeatureVector)\survDensity(s|\fixedFeatureVector)\\
 & \quad\qquad\qquad\qquad\qquad\quad-\censEnd(s|\featureVar_{(i)})\survDensity(s|\featureVar_{(i)})|ds\\
 & \le\frac{1}{\obsEnd(\timeVar|\fixedFeatureVector)}\max_{i=1,\dots,N}\int_0^{\timeVar}( \lips_{\textsc{\tiny{\survVar}}}+\survDensity^*\lips_{\textsc{\tiny{\censVar}}}s )\rho(\fixedFeatureVector,\featureVar_{(i)})^\holderIndex ds\\
 & \le\frac{1}{\obsEnd(\timeVar|\fixedFeatureVector)}\max_{i=1,\dots,N}\int_0^{\timeVar}( \lips_{\textsc{\tiny{\survVar}}}+\survDensity^*\lips_{\textsc{\tiny{\censVar}}}s )(\stanDistThresh h)^\holderIndex ds\\
 & =\frac{(\stanDistThresh h)^\holderIndex}{\obsEnd(\timeVar|\fixedFeatureVector)}\Big(\lips_{\textsc{\tiny{\survVar}}}\timeVar+\frac{\survDensity^*{\lips}_{\textsc{\tiny{\censVar}}}\timeVar^2}2\Big)\\
 & \le\frac{(\stanDistThresh h)^\holderIndex}{\theta}\Big(\lips_{\textsc{\tiny{\survVar}}}\timeHorizon+\frac{\survDensity^*{\lips}_{\textsc{\tiny{\censVar}}}\timeHorizon^2}2\Big)\\
 & \le\frac{\varepsilon}{18},
\end{align*}
where the last inequality uses the fact that
$
h\le\frac1{\stanDistThresh}[\frac{\varepsilon\theta}{18 (\lips_{\textsc{\tiny{\survVar}}}\timeHorizon+(\survDensity^*\lips_{\textsc{\tiny{\censVar}}}\timeHorizon^2)/2)}]^{1/\holderIndex}
$.
\end{proof}
\begin{lem}
\label{lem:kernel-Hk-bias}
Under Assumptions A1--A5, let $\fixedFeatureVector\in\text{supp}(\featureDist)$, $\timeVar\in[0,\timeHorizon]$, and $\varepsilon\in(0,1)$. When bad event $\badEvent{\textsc{NN}(\stanDistThresh h)}{\text{few neighbors}}(\fixedFeatureVector)$ does not happen, and the threshold distance satisfies $h\le\frac1{\stanDistThresh}[\frac{\varepsilon\theta^{2}K(\stanDistThresh)}{36(\lips_{\textsc{\tiny{\survVar}}}+\lips_{\textsc{\tiny{\censVar}}})\timeHorizon K(0)}]^{1/\holderIndex}$,
\[
\frac{2K(0)}{K(\stanDistThresh)\theta^{2}}\sup_{\timeVar\in[0,\timeHorizon]}|\mathbb{E}_{\{\obsVar\}}[\estObsEnd^{K}(\timeVar)]-\obsEnd(\timeVar|\fixedFeatureVector)|\le\frac{\varepsilon}{18}.
\]
\end{lem}

\begin{proof}
Note that $\mathbb{E}_{\{\obsVar\}}[\estObsEnd^{K}(\timeVar)]$ is a function of random variables $N,\featureVar_{(1)},\dots,\featureVar_{(N)}$. Since bad event $\badEvent{\textsc{NN}(\stanDistThresh h)}{\text{few neighbors}}(\fixedFeatureVector)$ does not happen, we know $N>\frac{1}{2}n\featureDist(\mathcal{B}_{\fixedFeatureVector,\stanDistThresh h})$. Then
\begin{align*}
 \mathbb{E}_{\{\obsVar\}}[\estObsEnd^{K}(\timeVar)]
 & =\sum_{j=1}^{N}\frac{K_{(j)}}{\sum_{\ell=1}^{N}K_{(\ell)}}\mathbb{E}_{\obsVar_{(j)}}[\ind\{\obsVar_{(j)}>t\}]\\
 & =\sum_{j=1}^{N}\frac{K_{(j)}}{\sum_{\ell=1}^{N}K_{(\ell)}}\obsEnd(\timeVar|\featureVar_{(j)}).
\end{align*}
Using H\"{o}lder's inequality and since $\obsEnd(\timeVar|\cdot)$ is H\"{o}lder continuous with parameters $(\lips_{\textsc{\tiny{\survVar}}}+\lips_{\textsc{\tiny{\censVar}}})\timeVar$ and~$\holderIndex$,
\begin{align*}
 & |\mathbb{E}_{\{\obsVar\}}[\estObsEnd^{K}(\timeVar)]-\obsEnd(\timeVar|\fixedFeatureVector)|\\
 & \quad =\bigg|\sum_{j=1}^{N}\frac{K_{(j)}}{\sum_{\ell=1}^{N}K_{(\ell)}}(\obsEnd(\timeVar|\featureVar_{(j)})-\obsEnd(\timeVar|\fixedFeatureVector))\bigg|\\
 & \quad \le\max_{j=1,\dots,N}|\obsEnd(\timeVar|\featureVar_{(j)})-\obsEnd(\timeVar|\fixedFeatureVector)| \\
 & \quad \le\max_{j=1,\dots,k}(\lips_{\textsc{\tiny{\survVar}}}+\lips_{\textsc{\tiny{\censVar}}})\timeVar\rho(\fixedFeatureVector,\featureVar_{(j)})^\holderIndex\\
 & \quad \le(\lips_{\textsc{\tiny{\survVar}}}+\lips_{\textsc{\tiny{\censVar}}})\timeVar(\stanDistThresh h)^\holderIndex\\
 & \quad \le(\lips_{\textsc{\tiny{\survVar}}}+\lips_{\textsc{\tiny{\censVar}}})\timeHorizon(\stanDistThresh h)^\holderIndex\\
 & \quad \le\frac{K(\stanDistThresh)\theta^{2}}{2K(0)}\cdot\frac{\varepsilon}{18},
\end{align*}
where the last inequality uses the assumption that $h\le\frac1{\stanDistThresh}[\frac{\varepsilon\theta^{2}K(\stanDistThresh)}{36(\lips_{\textsc{\tiny{\survVar}}}+\lips_{\textsc{\tiny{\censVar}}})\timeHorizon K(0)}]^{1/\holderIndex}$.
\end{proof}
\begin{lem}
\label{lem:kernel-W3}
Under Assumptions A1--A3 and A5, let $\fixedFeatureVector\in\text{supp}(\featureDist)$, $\timeVar\in[0,\timeHorizon]$, and $\varepsilon\in(0,1)$. If bad events $\badEvent{\textsc{NN}(\stanDistThresh h)}{\text{few neighbors}}(\fixedFeatureVector)$ and $\badEvent{\textsc{NN}(\stanDistThresh h)}{\text{bad }\timeHorizon}(\fixedFeatureVector)$ do not happen, and the number of training subjects satisfies
\begin{align*}
n & \ge\frac{144K^{2}(0)}{\varepsilon\theta^{2}\featureDist(\mathcal{B}_{\fixedFeatureVector,\stanDistThresh h})K^{2}(\stanDistThresh)},
\end{align*}
then $|W_3(\timeVar|\fixedFeatureVector)|\le{\varepsilon}/{18}$.
\end{lem}

\begin{proof}
We have $|W_3(\timeVar|\fixedFeatureVector)|=\sum_{i=1}^{N}\Xi_{(i)}$, where
\begin{align*}
\Xi_{(i)} & :=\eventVar_{(i)}\ind\{\obsVar_{(i)}\le \timeVar\}\sum_{\ell=2}^{\infty}\frac{1}{\ell(\frac{d_{K}^{+}(\obsVar_{(i)})}{K_{(i)}}+1)^{\ell}}\\
 & \le \eventVar_{(i)}\ind\{\obsVar_{(i)}\le \timeVar\}\sum_{\ell=2}^{\infty}\frac{1}{\ell(\frac{d_{K}^{+}(\timeVar)}{K_{(i)}}+1)^{\ell}}\\
 & \le\sum_{\ell=2}^{\infty}\frac{1}{\ell(\frac{d_{K}^{+}(\timeHorizon)}{K_{(i)}}+1)^{\ell}}\\
 & \le\sum_{\ell=2}^{\infty}\frac{1}{\ell(\frac{d_{K}^{+}(\timeHorizon)}{K(0)}+1)^{\ell}},
\end{align*}
using the facts that $d_{K}^{+}$ monotonically decreases and $K(\frac{\rho(\fixedFeatureVector,\featureVar_i)}{h})\le K(0)$. Since bad events $\badEvent{\textsc{NN}(\stanDistThresh h)}{\text{few neighbors}}(\fixedFeatureVector)$ and $\badEvent{\textsc{NN}(\stanDistThresh h)}{\text{bad }\timeHorizon}(\fixedFeatureVector)$ do not happen, we have
\begin{align*}
d_{K}^{+}(\timeHorizon) & \ge K(\stanDistThresh)d_{\mathcal{N}_{\textsc{NN}(\stanDistThresh h)}(\fixedFeatureVector)}^{+}(\timeHorizon) \\
 & >K(\stanDistThresh)\frac{N\theta}{2}\\
 & >K(\stanDistThresh)\frac{n\featureDist(\mathcal{B}_{\fixedFeatureVector,\stanDistThresh h})\theta}{4}.
\end{align*}
Since we assume that $n\ge\frac{144K^{2}(0)}{\varepsilon\theta^{2}\featureDist(\mathcal{B}_{\fixedFeatureVector,\stanDistThresh h})K^{2}(\stanDistThresh)}\ge\frac{1.84964K(0)}{\theta\featureDist(\mathcal{B}_{\fixedFeatureVector,\stanDistThresh h})K(\stanDistThresh)}$,
then using the above inequality, we have $\frac{d_{K}^{+}(\timeHorizon)}{K(0)}\ge0.46241$, which is needed to apply the reasoning from the proof of Lemma~\ref{lem:kNN-R3} to get
\[
\Xi_i \le\sum_{\ell=2}^{\infty}\frac{1}{\ell(\frac{d_{K}^{+}(\timeHorizon)}{K(0)}+1)^{\ell}}
 \le\frac{1}{(\frac{d_{K}^{+}(\timeHorizon)}{K(0)})^{2}}\le\frac{4K^{2}(0)}{K^{2}(\stanDistThresh)N^{2}\theta^{2}}.
\]
Hence,
\begin{align*}
|W_3(\timeVar|\fixedFeatureVector)| & =\sum_{i=1}^{N}\Xi_{(i)} \le\frac{4K^{2}(0)N}{K^{2}(\stanDistThresh)N^{2}\theta^{2}} =\frac{4K^{2}(0)}{K^{2}(\stanDistThresh)N\theta^{2}}\\
 & <\frac{8K^{2}(0)}{K^{2}(\stanDistThresh)n\featureDist(\mathcal{B}_{\fixedFeatureVector,\stanDistThresh h})\theta^{2}}  \le\frac{\varepsilon}{18},
\end{align*}
where the last inequality uses the assumption that $n\ge\frac{144K^{2}(0)}{\varepsilon\theta^{2}\featureDist(\mathcal{B}_{\fixedFeatureVector,\stanDistThresh h})K^{2}(\stanDistThresh)}$.
\end{proof}
Lastly, for the strong consistency result, the calculation is nearly the same as for the $k$-NN case in Appendix~\ref{sec:pf-kNN-strong-consistency}. To have each of the four terms in bound~\eqref{eq:kernel-ptwise-bound} be at most $\frac{1}{4n^2}$, it suffices to have
\[
n \ge \frac{11664}{p_{\min}\varepsilon^2\theta^4\kappa^4}\Big(\frac{18\Lips_K}{\varepsilon\theta}\Big)^{d/\holderIndex}\log\frac{864n^2}{\varepsilon\theta^2\kappa},
\]
where
\[
\varepsilon \le \frac{18\Lips_K (\stanDistThresh r^*)^\holderIndex}{\theta}.
\]
Then with a fair bit of algebra, one can show that the constants that show up in the theorem statement are
\begin{align*}
c_1'' &:= \frac{1}{\stanDistThresh}\Big[\frac{36(5\holderIndex+2d)}{(2\holderIndex+d)p_{\min}(\theta\Lips_K)^2\kappa^4}\Big]^{1/(2\holderIndex+d)}, \\
c_2'' &:= \Big[\frac{64(2\holderIndex+d)p_{\min}}{\holderIndex \kappa^{(d-2\holderIndex)/\holderIndex}}\Big(\frac{48}{\theta\Lips_K}\Big)^{d/\holderIndex}\Big]^{\holderIndex/(5\holderIndex+2d)}, \\
c_3'' &:= \Big[ \frac{11664(5\holderIndex+2d)(18\Lips_K)^{d/\holderIndex}}{(2\holderIndex+d)p_{\min}\theta^{(4\holderIndex+d)/\holderIndex}\kappa^4} \Big]^{\holderIndex/(2\holderIndex+d)}
\end{align*}
In particular, define $u'':=\log(\frac{c_2'' c_4''}{e})$, where
\[
c_4'' := \frac{36(5\holderIndex+2d)}{(2\holderIndex+d)p_{\min}(\theta\Lips_K)^2\kappa^4(\stanDistThresh r^*)^{2\holderIndex+d}}.
\]
Then for
\[
n\ge 
n_0'' :=
\begin{cases}
\lceil\frac{e^{1/(2d+5)}}{c_2''}\rceil
& \text{if }c_2''c_4''\le e, \\
\max\{\lceil\frac{e^{1/(2d+5)}}{c_2''}\rceil, \\ \quad c_4''(1+\sqrt{2u}+u)\}
& \text{if }c_2''c_4''>e,
\end{cases}
\]
if we choose
\[
h_n := c_1'' \Big(\frac{\log(c_2'' n)}{n}\Big)^{1/(2\holderIndex+d)},
\]
then
\[
\mathbb{P}\Big(
\sup_{\timeVar\in[0,\timeHorizon]}\!\!
    |\widehat{\survEnd}^K\!(\timeVar|\fixedFeatureVector;\!h_n)
     -\survEnd(\timeVar|\fixedFeatureVector)|
  \!\ge\! c_3''\Big(\frac{\log(c_2'' n)}{n}\Big)^{\!\frac\holderIndex{2\holderIndex+d}}
\Big)
\!\le\! \frac1{n^2}.
\]
As with the end of the proof of Corollary \ref{thm:kNN-strong-consistency} as provided in Appendix \ref{sec:pf-kNN-strong-consistency}, applying the Borel-Cantelli lemma completes the proof.

\section{Proof of Proposition~\ref{lem:weighted-edf}}
\label{sec:pf-weighted-edf}

First off, note that since each $Z_i$ is a real-valued continuous random variable, then its CDF, which we denote as $F_i$, is also continuous. Then note that $F(t)=\mathbb{E}[\widehat{F}(t)]=\sum_{i=1}^\ell \frac{w_i}{\sum_{j=1}^\ell w_j} F_i(t)$ is the finite sum of continuous functions, so $F(t)$ is also continuous.

The proof strategy is similar to that of proving the \mbox{$k$-NN} estimator guarantee in terms of how the supremum is handled. Let $a:=\sup\{\timeVar\in\realNumbers:F(\timeVar)={\varepsilon}/{3}\}$ and $b:=\inf\{\timeVar\in\realNumbers:F(\timeVar)=1-{\varepsilon}/{3}\}$; these exist due continuity of $F$. We partition interval $[a,b]$ at points $a=\eta_{1}<\eta_{2}<\cdots<\eta_{L(\varepsilon)}=b$, where:
\begin{itemize}[leftmargin=*,topsep=0pt,itemsep=0ex,partopsep=1ex,parsep=1ex]
\item $F(\eta_j)-F(\eta_{j-1})\le\varepsilon/3$ for $j=2,\dots,L(\varepsilon)$,
\item $L(\varepsilon)\le3/\varepsilon$.
\end{itemize}
We can always produce $\eta_1,\dots,\eta_{L(\varepsilon)}$ satisfying the above conditions since if we take them to be at points in which $F$ increases by exactly $\varepsilon/3$ in value starting from $a$ (except for the last point $\eta_{L(\varepsilon)}$, where the increase from $\eta_{L(\varepsilon)-1}$ could be less than $\varepsilon/3$), then the most number $L(\varepsilon)$ of interval pieces needed is $\lceil\frac{(1-{\varepsilon}/{3})-{\varepsilon}/{3}}{\varepsilon/3}\rceil+1 = \lceil3/\varepsilon\rceil - 1 \le 3/\varepsilon$. Then since $\widehat{F}$ is piecewise constant, if we can guarantee that $|\widehat{F}(\eta_j)-F(\eta_j)|\le\varepsilon/3$ for $j=1,\dots,L(\varepsilon)$, then at any point $\timeVar\in\realNumbers$, we indeed will have $|\widehat{F}(\timeVar)-F(\timeVar)|\le\varepsilon$.

Thus, the main task is in showing, for any given ${\timeVar\in\realNumbers}$, how to guarantee $|\widehat{F}(\timeVar)-F(\timeVar)|\le\varepsilon/3$ with high probability, i.e., we want to upper-bound ${\mathbb{P}(|\widehat{F}(\timeVar)-F(\timeVar)|>\varepsilon/3)}$. Once we have an upper bound for this probability, then by a union bound,
\begin{align}
 & \mathbb{P}\Big(\sup_{\timeVar\in\realNumbers}|\widehat{F}(\timeVar)-F(\timeVar)|>\varepsilon\Big)\nonumber \\
 & \quad \le\mathbb{P}\Big(\bigcup_{j=1}^{L(\varepsilon)}\{|\widehat{F}(\eta_j)-F(\eta_j)|>\varepsilon/3\}\Big)\nonumber \\
 & \quad \le\sum_{j=1}^{L(\varepsilon)}\mathbb{P}(|\widehat{F}(\eta_j)-F(\eta_j)|>\varepsilon/3). \label{eq:lem-wEDF-helper1}
\end{align}
We now upper-bound $\mathbb{P}(|\widehat{F}(\timeVar)-F(\timeVar)|>\varepsilon/3)$. Fix $\timeVar\in\realNumbers$. Note that $\widehat{F}(\timeVar)$ is the sum of $\ell$ independent variables, where the $i$-th variable is bounded in $[0,\frac{\weightVar_i}{\sum_{j=1}^{\ell}\weightVar_j}]$. Then applying Hoeffding's inequality,
\begin{align}
 \mathbb{P}(|\widehat{F}(\timeVar)-F(\timeVar)|>\varepsilon/3)
 & \le 2\exp\Big(-\frac{2\varepsilon^{2}(\sum_{j=1}^{\ell}\weightVar_j)^{2}}{9\sum_{i=1}^{\ell}\weightVar_i^{2}}\Big) \label{eq:lem-wEDF-helper2}
\end{align}
Putting together inequalities \eqref{eq:lem-wEDF-helper1} and \eqref{eq:lem-wEDF-helper2}, and noting that $L(\varepsilon)\le3/\varepsilon$,
\begin{align*}
 & \mathbb{P}\Big(\sup_{\timeVar\in\realNumbers}|\widehat{F}(\timeVar)-F(\timeVar)|>\varepsilon\Big)\\
 & \quad \le \frac6\varepsilon \exp\Big(-\frac{2\varepsilon^{2}(\sum_{j=1}^{\ell}\weightVar_j)^{2}}{9\sum_{i=1}^{\ell}\weightVar_i^{2}}\Big). \tag*{$\square$}
\end{align*}

\section{Choosing $\bm{k}$ Using a Validation Set}
\label{sec:validation}

We now present a guarantee that chooses $k$ based on a validation set of size $n$, sampled in the same manner as the training set. A similar approach can be used to select bandwidth $h$ for the fixed-radius NN and kernel estimators. We denote the validation set as $(X_{1}',Y_{1}',\delta_{1}'),\dots,(X_{n}',Y_{n}',\delta_{n}')$. For the validation data, we minimize a variant of the integrated Brier score (IBS) \citep{graf1999assessment} (also called the IPEC score by \citet{lowsky_2013}), which requires conditional survival and censoring time tail estimates $\widehat{\survEnd}$ and $\estCensEnd$ for $\survEnd$ and $\censEnd$. The integrated Brier score estimates the following mean squared error of $\widehat{\survEnd}$, which cannot be directly computed from training and validation data:
\begin{align*}
\text{MSE}(\widehat{\survEnd}) & :=\int_{0}^{\timeHorizon}\mathbb{E}[(\ind\{\survVar>\timeVar\}-\widehat{\survEnd}(\timeVar|\featureVar))^{2}]d\timeVar.
\end{align*}
Provided that estimators $\widehat{\survEnd}$ and $\estCensEnd$ are consistent, then the integrated Brier score is a consistent estimator of $\text{MSE}(\survEnd)$ \citep{gerds_2006}.

For any two estimators $\widehat{\survEnd}$ and $\estCensEnd$ of $\survEnd$ and $\censEnd$, and user-specified time horizon $\timeHorizon>0$ and lower bound $\theta_{\text{LB}}>0$ for $\theta$ in Assumption A3, our integrated Brier score variant is
\begin{align}
 & \text{IBS}(\widehat{\survEnd},\estCensEnd;\timeHorizon,\theta_{\text{LB}}) \nonumber \\
 & \quad:=\frac{1}{n}\sum_{i=1}^{n}\int_{0}^{\timeHorizon}\widehat{W}_{i}(\timeVar)(\ind(\obsVar_{i}'>\timeVar)-\widehat{\survEnd}(\timeVar|\featureVar_{i}'))^{2}d\timeVar,
\label{eq:IBS-variant}
\end{align}
where
\begin{align*}
\widehat{W}_{i}(\timeVar) & :=\begin{cases}
\frac{\eventVar_{i}'\ind\{\obsVar_{i}'\le\timeVar\}}{\estCensEnd(\obsVar_{i}'-|\featureVar_{i}')}+\frac{\ind\{\obsVar_{i}'>\timeVar\}}{\estCensEnd(\timeVar|\featureVar_{i}')} & \text{if }\estCensEnd(\timeVar|\featureVar_{i}')\ge\theta_{\text{LB}},\\
\frac{1}{\theta_{\text{LB}}} & \text{otherwise},
\end{cases}
\end{align*}
and $\estCensEnd(\timeVar-|\fixedFeatureVector)=\lim_{s\rightarrow\timeVar^{-}}\estCensEnd(s|\fixedFeatureVector)$ (for our estimators, $\estCensEnd$ is piecewise constant so $\estCensEnd(\timeVar-|\fixedFeatureVector)$ is straightforward to compute). The only difference between this score and the original integrated Brier score is that in the original score, there is no parameter $\theta_{\text{LB}}$ (put another way, $\theta_{\text{LB}}=0$). We introduce $\theta_{\text{LB}}$ to prevent division by 0 and so that in our analysis, the worst-case integrated Brier score is finite (note that $\widehat{W}_{i}(\timeVar)\le1/\theta_{\text{LB}}$, so the worst-case score is $\timeHorizon/\theta_{\text{LB}}$, assuming that estimate $\estCensEnd$ monotonically decreases and $\widehat{\survEnd}$ takes on values between 0 and 1). In practice, $\theta_{\text{LB}}$ could simply be set to an arbitrarily small but positive constant.

Due to the inherent symmetry in the problem setup, we can readily use the same $k$-NN estimator devised for estimating $\survEnd$ to instead estimate $\censEnd$. The only difference is that we replace the censoring indicator $\eventVar$ by $1-\eventVar$. In terms of the theory, the survival and censoring times swap roles. Thus, we can readily obtain estimates $\widehat{\survEnd}^{k\textsc{-NN}}$ and $\estCensEnd^{k\textsc{-NN}}$ of  $\survEnd$ and $\censEnd$.

Note that in practice, often the number of censored data can be quite small compared to $n$, which can make estimating the conditional censoring tail function $\censEnd$ difficult. There may be reason to believe that the censoring mechanism is actually independent of the feature vector, i.e., $\censEnd(\timeVar|\fixedFeatureVector)=\mathbb{P}(\censVar>\timeVar|\featureVar=\fixedFeatureVector)=\mathbb{P}(\censVar>\timeVar)$. In this case, we can estimate $\censEnd$ using, for instance, the standard Kaplan-Meier estimator (with $\eventVar$ replaced by $1-\eventVar$). Our validation guarantee will not be making this simplifying assumption; however, it can easily be modified to handle the case when the censoring time is independent of the feature vector.

The validation strategy we analyze is as follows: for a user-specified collection $\mathcal{K}$ of number of nearest neighbors to try (e.g., $\mathcal{K}=\{2^{j}:j=0,1,\dots,\lceil\log n\rceil\}$, or $\mathcal{K}=[n]$), choose $k\in\mathcal{K}$ that minimizes $\text{IBS}(\widehat{\survEnd}^{k\textsc{-NN}},\estCensEnd^{k\textsc{-NN}};\timeHorizon,\theta_{\text{LB}})$. Denote the resulting choice of $k$ as $\widehat{k}$. We have the following guarantee.
\begin{prop}
\label{prop:validation-result}
Under Assumptions A1--A4, suppose that there exists $p_{\min}>0$, $d>0$, and $r^{*}>0$ such that $\featureDist(\mathcal{B}_{\fixedFeatureVector,r})\ge p_{\min}r^{d}$ for all $\fixedFeatureVector\in\text{supp}(\featureDist)$ and $r\in[0,r^{*}]$. Let $\varepsilon\in(0,1)$ be a desired error tolerance and $\probError\in(0,1)$ be a error probability tolerance in estimating $\widehat{\survEnd}$. Define $\Lips_{\text{val}}:=\max\big\{\frac{2\timeHorizon}{\theta}(\lips_{\survSubscript}+\lips_{\censSubscript}),\lips_{\survSubscript}\timeHorizon+\frac{\survDensity^{*}\lips_{\censSubscript}\timeHorizon^{2}}{2},\lips_{\censSubscript}\timeHorizon+\frac{\censDensity^{*}\lips_{\survSubscript}\timeHorizon^{2}}{2}\big\},$ and
\begin{align*}
\mathcal{K}^{*} & :=\bigg\{ k\in[n]:\frac{648}{\varepsilon^{2}\theta^{4}}\log\bigg[\frac{4}{\probError}\Big(\frac{8}{\varepsilon}+2\Big(\frac{3}{\varepsilon}\log\frac{1}{\theta}+1\Big)\Big)\bigg]\\
 & \quad\quad\quad\le k\le\frac{1}{2}np_{\min}\Big(\frac{\varepsilon\theta}{18\Lips_{\text{val}}}\Big)^{d/\alpha}\bigg\}.
\end{align*}
Using the above procedure for selecting $\widehat{k}$, we have
\begin{align*}
 & \mathbb{E}[\text{IBS}(\widehat{\survEnd}^{k\textsc{-NN}},\estCensEnd^{k\textsc{-NN}};\timeHorizon,\theta_{\text{LB}})]\\
 & \le2e^{\varepsilon}\text{MSE}(\survEnd)+2e^{\varepsilon}\varepsilon^{2}\timeHorizon\\
 & \quad+\frac{\timeHorizon}{\theta_{\text{LB}}}\bigg[\probError+\sqrt{\frac{\log(2|\mathcal{K}|\sqrt{n})}{2n}}+\frac{1}{\sqrt{n}}+\ind\{\theta_{\text{LB}}>\theta\}\\
 & \quad\quad\;+\ind\{\mathcal{K}\cap\mathcal{K}^{*}=\emptyset\}+\ind\Big\{\varepsilon>\frac{18\Lips_{\text{val}}(r^{*})^{\alpha}}{\theta}\Big\}\bigg].
\end{align*}
\end{prop}
As with our rate of strong consistency results, the desired error tolerance $\varepsilon$ and error probability $\probError$ should be set to decrease to 0 as a function of $n$. Also, unsurprisingly the terms in the bound involve parameters in the underlying model that the user does not know in practice. Note that by choosing $\mathcal{K}=\{2^{j}:j=0,1,\dots,\lceil\log n\rceil\}$, $\varepsilon$ and $\probError$ to decrease with $n$ toward 0, and assuming that $\theta_{\text{LB}}>\theta$, then as $n\rightarrow\infty$, the bound above converges to $\text{MSE}(S)$.

In the bound, the first two terms correspond to approximation error in the integrated Brier score estimating $\text{MSE}(\survEnd)$. Next, $\timeHorizon/\theta_{\text{LB}}$ is the worst-case integrated Brier score. The terms that it is multiplied by are as follows:
\begin{itemize}[leftmargin=1.5em,topsep=0pt,itemsep=0ex,partopsep=1ex,parsep=1ex]
\item $\probError$ is the error probability in estimating $\survEnd$ and $\log\censEnd$
\item The two $\widetilde{\mathcal{O}}(n^{-1/2})$ terms both have to do with the $|\mathcal{K}|$ empirical integrated Brier scores not being close to their means (over randomness in validation data)
\item $\theta_{\text{LB}}>\theta$ happens when the user-specified $\theta_{\text{LB}}$ is not a lower bound for the true $\theta$
\item $\mathcal{K}\cap\mathcal{K}^{*}=\emptyset$ means one of two things: either the number of training data $n$ is too small, or $\mathcal{K}$ is chosen poorly so that it does not contain any members of $\mathcal{K}^{*}$, which consists of good choices for the number of nearest neighbors $k$ (e.g., if $\mathcal{K}=\{2^{j}:j=0,1,\dots,\lceil\log n\rceil\}$, then by having the number of training data~$n$ be large enough that $\mathcal{K}^*$ contains a power of 2, we can ensure $\mathcal{K}\cap\mathcal{K}^{*}$ to be nonempty)
\item $\varepsilon>\frac{18\Lips_{\text{val}}(r^{*})^{\alpha}}{\theta}$ happens when the error tolerance chosen is too large
\end{itemize}
Note that our analysis requires that we simultaneously have an additive error guarantee for $\widehat{\survEnd}^{k\textsc{-NN}}$ and a multiplicative error guarantee for $\estCensEnd^{k\textsc{-NN}}$. We use the following lemma.

\begin{lem}
\label{lem:e-IBS-ptwise-bad-bound}
Under Assumptions A1--A4, let $\varepsilon\in(0,1)$ be a user-specified error tolerance and define critical distance $h^{*}=(\frac{\varepsilon\theta}{18\Lips_{\text{val}}})^{1/\alpha}$. For any feature vector $\fixedFeatureVector\in\text{supp}(\featureDist)$ and any choice of number of nearest neighbors $k\in[\frac{72}{\varepsilon\theta^{2}},\frac{n\featureDist(\mathcal{B}_{\fixedFeatureVector,h^{*}})}{2}]$, we have, over randomness in the training data,
\begin{align}
 & \mathbb{P}\bigg(\Big\{\sup_{\timeVar\in[0,\timeHorizon]}|\widehat{\survEnd}^{k\textsc{-NN}}(\timeVar|\fixedFeatureVector)-\survEnd(\timeVar|\fixedFeatureVector)|>\varepsilon\Big\}\nonumber \\
 & \quad\cup\Big\{\sup_{\timeVar\in[0,\timeHorizon]}|\log\estCensEnd^{k\textsc{-NN}}(\timeVar|\fixedFeatureVector)-\log\censEnd(\timeVar|\fixedFeatureVector)|>\varepsilon\Big\}\Big)\nonumber \\
 & \quad\le\exp\Big(-\frac{k\theta}{8}\Big)+\exp\Big(-\frac{np_{\min}}{8}\Big(\frac{\varepsilon\theta}{18\Lips_{\text{val}}}\Big)^{d/\holderIndex}\Big)\nonumber \\
 & \quad\quad+2\exp\Big(-\frac{k\varepsilon^{2}\theta^{4}}{648}\Big)\nonumber \\
 & \quad\quad+\Big[\frac{8}{\varepsilon}+2\Big(\frac{3}{\varepsilon}\log\frac{1}{\theta}+1\Big)\Big]\exp\Big(-\frac{k\varepsilon^{2}\theta^{2}}{162}\Big).\label{eq:e-IBS-ptwise-bad-bound}
\end{align}
\end{lem}

\begin{proof}
This lemma follows readily from the proof of Theorem~\ref{thm:kNN-survival} and the remark at the end of Appendix~\ref{sec:Nelson-Aalen} for how to modify the proof of Theorem~\ref{thm:kNN-survival} to handle log. By carefully examining the proof for Theorem~\ref{thm:kNN-survival}, we see that bad events $\badEvent{k\textsc{-NN}}{\text{bad }\timeHorizon}(\fixedFeatureVector)$, $\badEvent{k\textsc{-NN}}{\text{far neighbors}}(\fixedFeatureVector)$, and $\badEvent{k\textsc{-NN}}{\text{bad EDF}}(\fixedFeatureVector)$ for the $k$-NN estimate $\widehat{\survEnd}^{k\textsc{-NN}}$ of $\survEnd$ can actually be shared with the bad events for the $k$-NN estimate $\log\estCensEnd^{k\textsc{-NN}}$ of $\log\censEnd$, with the small change that we now replace $\Lips$ with $\Lips_{\text{val}}$ within the choice of $h^{*}$ (note that $\Lips_{\text{val}}$ is now symmetric in the survival and censoring time terms, which naturally happens because we estimate tail functions for both).

With the above explanation, note that the first three RHS terms in bound~\eqref{eq:e-IBS-ptwise-bad-bound} are the same as those of Theorem~\ref{thm:kNN-survival}. However, bad event $\badEvent{k\textsc{-NN}}{\text{bad }U_{1}}(\timeVar,\featureVar)$ (which is controlled at no larger than $8/\varepsilon$ time points) has to be changed for estimating $\log\censEnd$ instead (as discussed in Appendix~\ref{sec:Nelson-Aalen}, the number of time points for controlling the log is at most $2(\frac{3}{\varepsilon}\log\frac{1}{\theta}+1)$ instead of $8/\varepsilon$). Thus, the fourth RHS term in bound~\eqref{eq:e-IBS-ptwise-bad-bound} union bounds over the final $k$-NN regression pieces of estimators $\widehat{\survEnd}^{k\textsc{-NN}}$ and $\log\estCensEnd^{k\textsc{-NN}}$.
\end{proof}

\subsection*{Proof of Proposition \ref{prop:validation-result}}

Bound~\eqref{eq:e-IBS-ptwise-bad-bound} is at most $\probError$ (by making each of the four RHS terms at most $\probError/4$) when $k$, $n$, and $\varepsilon$ satisfy
\begin{align}
 & \frac{648}{\varepsilon^{2}\theta^{4}}\log\bigg[\frac{4}{\probError}\Big(\frac{8}{\varepsilon}+2\Big(\frac{3}{\varepsilon}\log\frac{1}{\theta}+1\Big)\Big)\bigg]\nonumber \\
 & \quad\le k\le\frac{1}{2}np_{\min}\Big(\frac{\varepsilon\theta}{18\Lips_{\text{val}}}\Big)^{d/\alpha},\label{eq:val-guarantee-sufficient-k}
\end{align}
and
\begin{align}
 & \varepsilon\le\frac{18\Lips_{\text{val}}(r^{*})^{\alpha}}{\theta}.\label{eq:val-guarantee-sufficient-eps}
\end{align}
We refer to the bad event of Lemma~\ref{lem:e-IBS-ptwise-bad-bound} as $\badEvent{k\textsc{-NN}}{\text{bad est}}(\fixedFeatureVector)$. The set $\mathcal{K}^{*}$ precisely corresponds to choices for the number of nearest neighbors that satisfy sufficient condition~\eqref{eq:val-guarantee-sufficient-k}. If $\mathcal{K}\cap\mathcal{K}^{*}$ is nonempty, then the validation procedure could potentially select some $k\in\mathcal{K}\cap\mathcal{K}^{*}$. If, furthermore, $\theta_{\text{LB}}\le\theta$, and $\varepsilon$ satisfies condition~\eqref{eq:val-guarantee-sufficient-eps}, then our performance guarantee comes into effect. For the rest of the proof, we assume that these nice conditions happen; otherwise, we assume a worst-case integrated Brier score of $\timeHorizon/\theta_{\text{LB}}$.

Throughout the proof, we use the abbreviation $\text{IBS}(k):=\text{IBS}(\widehat{\survEnd}^{k\textsc{-NN}},\estCensEnd^{k\textsc{-NN}};\timeHorizon,\theta_{\text{LB}})$. We denote $\mathbb{E}_{n}$ to be the expectation over the $n$ training data, and $\mathbb{E}_{n'}$ to be the expectation over the $n$ validation data.

We introduce a bad event for when at least one of the integrated Brier scores we compute during validation is not sufficiently close to its expectation over randomness in the validation data:
\begin{align*}
 & \badEvent{}{\text{bad IBS}}:=\bigcup_{k\in\mathcal{K}}\Bigg\{\text{IBS}(k)\ge\mathbb{E}_{n'}[\text{IBS}(k)]\\
 & \qquad\qquad\qquad\qquad\quad\quad+\frac{\timeHorizon}{\theta_{\text{LB}}}\sqrt{\frac{\log(|\mathcal{K}|\sqrt{n})}{2n}}\Bigg\}.
\end{align*}
Note that, over randomness in the validation data, $\text{IBS}(k)$ is the average of $n$ independent terms each bounded in $[0,\timeHorizon/\theta_{\text{LB}}]$. Thus, by Hoeffding's inequality and a union bound over $k\in\mathcal{K}$, we have $\mathbb{P}(\badEvent{}{\text{bad IBS}})\le1/\sqrt{n}$.

Let $\widetilde{k}\in\mathcal{K}\cap\mathcal{K}^{*}$. We will show shortly that $\mathbb{E}[\text{IBS}(\widetilde{k})]$ is close to $\text{MSE}(\survEnd)$. When bad event $\badEvent{}{\text{bad IBS}}$ does not happen, then
\[
\text{IBS}(\widetilde{k})\le\mathbb{E}_{n'}[\text{IBS}(\widetilde{k})]+\frac{\timeHorizon}{\theta_{\text{LB}}}\sqrt{\frac{\log(|\mathcal{K}|\sqrt{n})}{2n}}.
\]
Moreover, by how $\widehat{k}$ is chosen, $\text{IBS}(\widehat{k})\le\text{IBS}(k)$ for all $k\in\mathcal{K}$. In particular, $\text{IBS}(\widehat{k})\le\text{IBS}(\widetilde{k})$. Therefore,
\[
\text{IBS}(\widehat{k})\le\mathbb{E}_{n'}[\text{IBS}(\widetilde{k})]+\frac{\timeHorizon}{\theta_{\text{LB}}}\sqrt{\frac{\log(2|\mathcal{K}|\sqrt{n})}{2n}}.
\]
Taking the expectation $\mathbb{E}_{n}$ of both sides above over randomness in the training data,
\begin{equation}
\mathbb{E}_{n}[\text{IBS}(\widehat{k})]\le\mathbb{E}[\text{IBS}(\widetilde{k})]+\frac{\timeHorizon}{\theta_{\text{LB}}}\sqrt{\frac{\log(2|\mathcal{K}|\sqrt{n})}{2n}}.\label{eq:validation-helper0}
\end{equation}
Much of the rest of the proof is in upper-bounding $\mathbb{E}[\text{IBS}(\widetilde{k})]$ in terms of the mean squared error achieved by $\survEnd$:
\begin{align*}
\text{MSE}(\survEnd) & =\int_{0}^{\timeHorizon}\mathbb{E}_{\featureVar}\big[\mathbb{E}_{\survVar}[(\ind\{\survVar>\timeVar\}-\survEnd(\timeVar|\featureVar))^{2}]\big]d\timeVar.
\end{align*}
As it will be helpful to know what this is equal to, we compute it now. The inner-most expectation inside the integral is
\begin{align*}
 & \mathbb{E}_{\survVar}[(\ind\{\survVar>\timeVar\}-S(\timeVar|\featureVar))^{2}]\\
 & \quad =\mathbb{E}_{\survVar}[\ind\{\survVar>\timeVar\}-2\ind\{\survVar>\timeVar\}\survEnd(\timeVar|\featureVar)+(\survEnd(\timeVar|\featureVar))^{2}]\\
 & \quad =\survEnd(\timeVar|\featureVar)-2(\survEnd(\timeVar|\featureVar))^{2}+(\survEnd(\timeVar|\featureVar))^{2}\\
 & \quad =\survEnd(\timeVar|\featureVar)(1-\survEnd(\timeVar|\featureVar)).
\end{align*}
Hence,
\begin{equation}
\text{MSE}(\survEnd)=\int_{0}^{\timeHorizon}\mathbb{E}[\survEnd(\timeVar|\featureVar)(1-\survEnd(\timeVar|\featureVar))]d\timeVar.\label{eq:MSE-S}
\end{equation}
We proceed to upper-bounding $\mathbb{E}[\text{IBS}(\widetilde{k})]$ in terms of $\text{MSE}(\survEnd)$. Note that
\begin{align*}
 & \mathbb{E}_{n'}[\text{IBS}(\widetilde{k})]\\
 & \!=\!\frac{1}{n}\!\sum_{i=1}^{n}\!\int_{0}^{\timeHorizon}\!\!\!\mathbb{E}_{\featureVar_{i}',\obsVar_{i}',\eventVar_{i}'}[\widehat{W}_{i}(\timeVar)(\ind(\obsVar_{i}'>\timeVar)-\widehat{\survEnd}^{\widetilde{k}\textsc{-NN}}(\timeVar|\featureVar_{i}'))^{2}]d\timeVar.
\end{align*}
Since the validation data are i.i.d., let $\featureVar$ denote a feature vector sampled from $\featureDist$ and denote its observed time and censoring indicator as $\obsVar$ and $\eventVar$. Then
\begin{align*}
&\mathbb{E}_{n'}[\text{IBS}(\widetilde{k})] \\
&\quad=\int_{0}^{\timeHorizon}\mathbb{E}_{\featureVar,\obsVar,\eventVar}[\widehat{W}(\timeVar)(\ind(\obsVar>\timeVar)-\widehat{\survEnd}^{\widetilde{k}\textsc{-NN}}(\timeVar|\featureVar))^{2}]d\timeVar,
\end{align*}
where
\begin{align*}
\widehat{W}(\timeVar) & :=\begin{cases}
\frac{\eventVar\ind\{\obsVar\le\timeVar\}}{\estCensEnd(\obsVar-|\featureVar)}+\frac{\ind\{\obsVar>\timeVar\}}{\estCensEnd(\timeVar|\featureVar)} & \text{if }\estCensEnd(\timeVar|\featureVar)\ge\theta_{\text{LB}},\\
\frac{1}{\theta_{\text{LB}}} & \text{otherwise}.
\end{cases}
\end{align*}
Then
\begin{align}
 & \mathbb{E}[\text{IBS}(\widetilde{k})]\nonumber \\
 & \!=\mathbb{E}_{n}\big[\mathbb{E}_{n'}[\text{IBS}(\widetilde{k})]\big]\nonumber \\
 & \!=\!\int_{0}^{\timeHorizon}\!\!\!\mathbb{E}_{\featureVar}\Big[\mathbb{E}_{n}\big[\mathbb{E}_{\obsVar,\eventVar}[\widehat{W}(\timeVar)(\ind(\obsVar>\timeVar)-\widehat{\survEnd}^{\widetilde{k}\textsc{-NN}}(\timeVar|\featureVar))^{2}]\big]\Big]d\timeVar\nonumber \\
 & \!=\!\int_{0}^{\timeHorizon}\!\!\!\mathbb{E}_{\featureVar}\big[\mathbb{E}_{n}[\Xi]\big]d\timeVar,\label{eq:E-IBS-actual}
\end{align}
where
\[
\Xi:=\mathbb{E}_{\obsVar,\eventVar}[\widehat{W}(\timeVar)(\ind(\obsVar>\timeVar)-\widehat{\survEnd}^{k\textsc{-NN}}(\timeVar|\featureVar))^{2}].
\]
Note that $\Xi$ is a function of test point $\featureVar$ and the training data, and $\Xi$ is upper-bounded by $1/\theta_{\text{LB}}$. The expectation $\mathbb{E}_{n}[\Xi]$ is a function of $\featureVar$, which we are conditioning on (so we treat it as fixed). Then, denoting $\mathbb{P}_{n}$ to be probability over the training data, and noting that bad event $\badEvent{\widetilde{k}\textsc{-NN}}{\text{bad est}}(\featureVar)$ is also a function of training data,
\begin{align}
 \mathbb{E}_{n}[\Xi]
 & =\underbrace{\mathbb{E}_{n}[\Xi\,|\,\badEvent{\widetilde{k}\textsc{-NN}}{\text{bad est}}(\featureVar)]}_{\le1/\theta_{\text{LB}}}\underbrace{\mathbb{P}_{n}(\badEvent{\widetilde{k}\textsc{-NN}}{\text{bad est}}(\featureVar))}_{\le\probError}\nonumber \\
 & \quad+\mathbb{E}_{n}\big[\Xi\,\big|\,[\badEvent{\widetilde{k}\textsc{-NN}}{\text{bad est}}(\featureVar)]^{c}\big]\underbrace{\mathbb{P}_{n}([\badEvent{\widetilde{k}\textsc{-NN}}{\text{bad est}}(\featureVar)]^{c})}_{\le1}\nonumber \\
 & \le\frac{\probError}{\theta_{\text{LB}}}+\mathbb{E}_{n}\big[\Xi\,\big|\,[\badEvent{\widetilde{k}\textsc{-NN}}{\text{bad est}}(\featureVar)]^{c}\big].\label{eq:E-IBS-helper0}
\end{align}
When bad event $\badEvent{\widetilde{k}\textsc{-NN}}{\text{bad est}}(\featureVar)$ does not happen, we simultaneously have
\begin{equation*}
\sup_{\timeVar\in[0,\timeHorizon]}|\widehat{\survEnd}^{\widetilde{k}\textsc{-NN}}(\timeVar|\featureVar)-\survEnd(\timeVar|\featureVar)|\le\varepsilon,
\end{equation*}
and
\begin{equation*}
\sup_{\timeVar\in[0,\timeHorizon]}|\log\estCensEnd^{\widetilde{k}\textsc{-NN}}(\timeVar|\featureVar)-\log\censEnd(\timeVar|\featureVar)|\le\varepsilon.
\end{equation*}
Hence,
\begin{align*}
 & (\ind(\obsVar>\timeVar)-\widehat{\survEnd}^{k\textsc{-NN}}(\timeVar|\featureVar))^{2}\nonumber \\
 & \quad\le(|\ind(\obsVar>\timeVar)-\survEnd(\timeVar|\featureVar)|+|\survEnd(\timeVar|\featureVar)-\widehat{\survEnd}^{k\textsc{-NN}}(\timeVar|\featureVar)|)^{2}\nonumber \\
 & \quad\le(|\ind(\obsVar>\timeVar)-\survEnd(\timeVar|\featureVar)|+\varepsilon)^{2}\nonumber \\
 & \quad\le2((\ind(\obsVar>\timeVar)-\survEnd(\timeVar|\featureVar))^{2}+\varepsilon^{2}),
\end{align*}
and
\begin{align*}
\widehat{W}(\timeVar) & =\frac{\eventVar\ind\{\obsVar\le\timeVar\}}{\estCensEnd^{k\textsc{-NN}}(\obsVar-|\featureVar)}+\frac{\ind\{\obsVar>\timeVar\}}{\estCensEnd^{k\textsc{-NN}}(\timeVar|\featureVar)}\nonumber \\
 & \le\frac{\eventVar\ind\{\obsVar\le\timeVar\}}{\estCensEnd^{k\textsc{-NN}}(\obsVar|\featureVar)}+\frac{\ind\{\obsVar>\timeVar\}}{\estCensEnd^{k\textsc{-NN}}(\timeVar|\featureVar)}\nonumber \\
 & \le e^{\varepsilon}\frac{\eventVar\ind\{\obsVar\le\timeVar\}}{\censEnd(\obsVar|\featureVar)}+e^{\varepsilon}\frac{\ind\{\obsVar>\timeVar\}}{\censEnd(\timeVar|\featureVar)}.
\end{align*}
Then
\begin{align*}
 & \widehat{W}(\timeVar)(\ind(\obsVar>\timeVar)-\widehat{\survEnd}^{k\textsc{-NN}}(\timeVar|\featureVar))^{2}\\
 & \le e^{\varepsilon}\frac{\eventVar\ind\{\obsVar\le\timeVar\}}{\censEnd(\obsVar|\featureVar)}2((\ind(\obsVar>\timeVar)-\survEnd(\timeVar|\featureVar))^{2}+\varepsilon^{2})\\
 & \quad+e^{\varepsilon}\frac{\ind\{\obsVar>\timeVar\}}{\censEnd(\timeVar|\featureVar)}2((\ind(\obsVar>\timeVar)-\survEnd(\timeVar|\featureVar))^{2}+\varepsilon^{2})\\
 & =2e^{\varepsilon}\frac{\eventVar\ind\{\obsVar\le\timeVar\}}{\censEnd(\obsVar|\featureVar)}((\survEnd(\timeVar|\featureVar))^{2}+\varepsilon^{2})\\
 & \quad+2e^{\varepsilon}\frac{\ind\{\obsVar>\timeVar\}}{\censEnd(\timeVar|\featureVar)}((1-\survEnd(\timeVar|\featureVar))^{2}+\varepsilon^{2}),
\end{align*}
so
\begin{align}
 & \mathbb{E}_{n}\big[\Xi\,\big|\,[\badEvent{\widetilde{k}\textsc{-NN}}{\text{bad est}}(\featureVar)]^{c}\big]\nonumber \\
 & =2e^{\varepsilon}(\survEnd(\timeVar|\featureVar))^{2}+\varepsilon^{2})\nonumber\\
 & \quad\qquad\times\mathbb{E}_{n}\bigg[\mathbb{E}_{\obsVar,\eventVar}\Big[\frac{\eventVar\ind\{\obsVar\le\timeVar\}}{\censEnd(\obsVar|\featureVar)}\,\Big|\,[\badEvent{\widetilde{k}\textsc{-NN}}{\text{bad est}}(\featureVar)]^{c}\Big]\bigg]\nonumber \\
 & \quad+\frac{2e^{\varepsilon}((1-\survEnd(\timeVar|\featureVar))^{2}+\varepsilon^{2})}{\censEnd(\timeVar|\featureVar)}\nonumber\\
 &\quad\qquad\times\mathbb{E}_{n}\Big[\mathbb{E}_{\obsVar,\eventVar}\big[\ind\{\obsVar>\timeVar\}\,\Big|\,[\badEvent{\widetilde{k}\textsc{-NN}}{\text{bad est}}(\featureVar)]^{c}\big]\Big].\label{eq:E-IBS-helper3}
\end{align}
Next, note that we are currently conditioning on $\featureVar$ and $[\badEvent{\widetilde{k}\textsc{-NN}}{\text{bad est}}(\featureVar)]^{c}$. With this conditioning, $\frac{\eventVar\ind\{\obsVar\le\timeVar\}}{\censEnd(\obsVar|\featureVar)}$ (which does not depend on training data) is independent of $[\badEvent{\widetilde{k}\textsc{-NN}}{\text{bad est}}(\featureVar)]^{c}$. Thus
\begin{align}
 & \mathbb{E}_{n}\bigg[\mathbb{E}_{\obsVar,\eventVar}\Big[\frac{\eventVar\ind\{\obsVar\le\timeVar\}}{\censEnd(\obsVar|\featureVar)}\,\Big|\,[\badEvent{\widetilde{k}\textsc{-NN}}{\text{bad est}}(\featureVar)]^{c}\Big]\bigg]\nonumber \\
 & \quad=\mathbb{E}_{\obsVar,\eventVar}\Big[\frac{\eventVar\ind\{\obsVar\le\timeVar\}}{\censEnd(\obsVar|\featureVar)}\Big]\nonumber \\
 & \quad=\int_{0}^{t}\int_{s}^{\infty}\frac{1}{\censEnd(s|\featureVar)}d\mathbb{P}_{\censVar|\featureVar}(c)d\mathbb{P}_{\survVar|\featureVar}(s)\nonumber \\
 & \quad=\int_{0}^{t}\frac{\censEnd(s|\featureVar)}{\censEnd(s|\featureVar)}d\mathbb{P}_{\survVar|\featureVar}(s)\nonumber \\
 & \quad=\int_{0}^{t}d\mathbb{P}_{\survVar|\featureVar}(s)\nonumber \\
 & \quad=1-\survEnd(t|\featureVar).\label{eq:E-IBS-helper4}
\end{align}
Similarly,
\begin{align}
 & \mathbb{E}_{n}\Big[\mathbb{E}_{\obsVar,\eventVar}\big[\ind\{\obsVar>\timeVar\}\,\Big|\,[\badEvent{\widetilde{k}\textsc{-NN}}{\text{bad est}}(\featureVar)]^{c}\big]\Big]\nonumber \\
 & \quad=\mathbb{E}_{\obsVar}[\ind\{\obsVar>\timeVar\}]
 =\obsEnd(\timeVar|\featureVar)
 =\survEnd(\timeVar|\featureVar)\censEnd(\timeVar|\featureVar).\label{eq:E-IBS-helper5}
\end{align}
Putting together inequality \eqref{eq:E-IBS-helper3} with equations
\eqref{eq:E-IBS-helper4} and \eqref{eq:E-IBS-helper5},

\begin{align}
 & \mathbb{E}_{n}\big[\Xi\,\big|\,[\badEvent{\widetilde{k}\textsc{-NN}}{\text{bad est}}(\featureVar)]^{c}\big]\nonumber \\
 & \quad\le2e^{\varepsilon}(\survEnd(\timeVar|\featureVar))^{2}+\varepsilon^{2})(1-\survEnd(\timeVar|\featureVar))\nonumber \\
 & \qquad+2e^{\varepsilon}((1-\survEnd(\timeVar|\featureVar))^{2}+\varepsilon^{2})\survEnd(\timeVar|\featureVar).\nonumber \\
 & \quad=2e^{\varepsilon}\survEnd(\timeVar|\featureVar)(1-\survEnd(\timeVar|\featureVar))+2e^{\varepsilon}\varepsilon^{2}.\label{eq:E-IBS-helper6}
\end{align}
Finally, putting together equation~\eqref{eq:E-IBS-actual} with inequalities~\eqref{eq:E-IBS-helper0} and~\eqref{eq:E-IBS-helper6} and also using equation~\eqref{eq:MSE-S},
\begin{align*}
 & \mathbb{E}[\text{IBS}(\widetilde{k})]\\
 & =\int_{0}^{\timeHorizon}\mathbb{E}_{\featureVar}\big[\mathbb{E}_{n}[\Xi]\big]d\timeVar\\
 & \le\int_{0}^{\timeHorizon}\mathbb{E}_{\featureVar}\Big[\frac{\probError}{\theta_{\text{LB}}}+\mathbb{E}_{n}\big[\Xi\,\big|\,[\badEvent{\widetilde{k}\textsc{-NN}}{\text{bad est}}(\featureVar)]^{c}\big]\Big]d\timeVar\\
 & \le\int_{0}^{\timeHorizon}\mathbb{E}_{\featureVar}\Big[\frac{\probError}{\theta_{\text{LB}}}+2e^{\varepsilon}\survEnd(\timeVar|\featureVar)(1-\survEnd(\timeVar|\featureVar))+2e^{\varepsilon}\varepsilon^{2}\Big]d\timeVar\\
 & =2e^{\varepsilon}\text{MSE}(\survEnd)+2e^{\varepsilon}\varepsilon^{2}\timeHorizon+\frac{\probError\timeHorizon}{\theta_{\text{LB}}}.
\end{align*}
Combining this with inequality~\eqref{eq:validation-helper0}, we get
\begin{align*}
 & \mathbb{E}_{n}[\text{IBS}(\widehat{k})]\\
 & \le2e^{\varepsilon}\text{MSE}(\survEnd)+2e^{\varepsilon}\varepsilon^{2}\timeHorizon+\frac{\timeHorizon}{\theta_{\text{LB}}}\bigg[\probError+\sqrt{\frac{\log(2|\mathcal{K}|\sqrt{n})}{2n}}\bigg].
\end{align*}
This holds with probability at least $1-1/\sqrt{n}$ over randomness in the validation data and provided that $\theta_{\text{LB}}\le\theta$,
$\mathcal{K}\cap\mathcal{K}^{*}\ne\emptyset$, and $\varepsilon\le\frac{18\Lips_{\text{val}}(r^{*})^{\alpha}}{\theta}$.

\section{Additional Example Distribution Satisfying Assumptions A1--A4}
\label{sec:weibull-regression}

\begin{example}[Weibull regression]
We generalize the exponential regression model of Example~\ref{ex:exp-regress}. As before, $\featureSpace=\realNumbers^d$, and $\featureDist$ is a Borel probability measure with compact, convex support. We now take the hazard function to be $\hazard_{\survSubscript}(\timeVar|\fixedFeatureVector)=q (\hazard_{\survSubscript,0})^{q} \timeVar^{q-1} \exp(\fixedFeatureVector^\top \beta_{\survSubscript})$ for parameters ${q>0}$, ${\hazard_{\survSubscript,0}>0}$, and ${\beta_{\survSubscript}\in\realNumbers^d}$ (choosing $q=1$ yields Example~\ref{ex:exp-regress}). Following a similar integral calculation as in Example~\ref{ex:exp-regress}, we have $\survEnd(\timeVar|\fixedFeatureVector)=\exp(-(\hazard_{\survSubscript,0}e^{\fixedFeatureVector^\top \beta_{\survSubscript}} \timeVar)^{q})$, so the conditional survival time distribution $\mathbb{P}_{\survVar|\featureVar=\fixedFeatureVector}$ corresponds to a Weibull distribution with shape parameter $q$ and scale parameter $[ {\hazard_{\survSubscript,0}e^{\fixedFeatureVector^\top \beta_{\survSubscript}}} ]^{-1}$. We similarly define the conditional censoring time distribution using hazard function $\hazard_{\censSubscript}(\timeVar|\fixedFeatureVector)=q (\hazard_{\censSubscript,0})^{q} \timeVar^{q-1} \exp(\fixedFeatureVector^\top \beta_{\censSubscript})$ using the same $q>0$ as for the survival time but different parameters $\hazard_{\censSubscript,0}>0$ and $\beta_{\censSubscript}\in\realNumbers^d$. In this case, the observed time $\obsVar=\min\{\survVar,\censVar\}$ conditioned on $\featureVar=\fixedFeatureVector$ has a Weibull distribution with shape parameter $q$ and scale parameter $1/\omega'(\fixedFeatureVector)$, where
\[
\omega'(\fixedFeatureVector)
:=
{\big[\big({\hazard_{\survSubscript,0}e^{\fixedFeatureVector^\top \beta_{\survSubscript}}}\big)^q
+ \big({\hazard_{\censSubscript,0}e^{\fixedFeatureVector^\top \beta_{\censSubscript}}}\big)^q\big]^{1/q}}.
\]
The median of this distribution is $[(\log2)^{1/q}]/\omega'(\fixedFeatureVector)$. Thus, Assumption A3 is satisfied with $\theta=1/2$ and $\timeHorizon = \min_{\fixedFeatureVector\in\text{supp}(\featureDist)} \{[(\log2)^{1/q}]/\omega'(\fixedFeatureVector)\}$. Lastly, for Assumption A4, we can again take the Lipschitz constant for $\survDensity(\timeVar|\cdot)$ to be $\lips_{\survSubscript}=\sup_{\fixedFeatureVector\in\text{supp}(\featureDist),\timeVar\in[0,\timeHorizon]} \|\frac{\partial \survDensity(\timeVar|\fixedFeatureVector)}{\partial \fixedFeatureVector}\|_2$. We can similarly choose the Lipschitz constant for $\censDensity(\timeVar|\cdot)$.
\end{example}

\section{Nearest Neighbor and Kernel Variants of the Nelson-Aalen Estimator}
\label{sec:Nelson-Aalen}

The Nelson-Aalen estimator estimates the marginal cumulative hazard function $\Hazard_{\text{marg}}(\timeVar)=-\log\survEnd_{\text{marg}}(\timeVar)=-\log\mathbb{P}(\survVar>\timeVar)$ \citep{nelson_1969,aalen_1978}. We first give the general form of the Nelson-Aalen estimator, restricted to training subjects $\setS\in[n]$. Recall that among training subjects $\setS$, the set of unique death times is $\setY_{\setS}$. At time $\timeVar\ge0$, the number of deaths is $d_{\setS}(\timeVar)$ and the number of subjects at risk is $n_{\setS}(\timeVar)$. Then the Nelson-Aalen estimator restricted to subjects $\setS$ is given by
\[
\widehat{\Hazard}^{\text{NA}}(\timeVar|\setS):=\sum_{t'\in\setY_{\setS}}\frac{d_{\setS}(\timeVar')\ind\{\timeVar'\le\timeVar\}}{n_{\setS}(\timeVar')}.
\]
Thus, the Nelson-Aalen-based $k$-NN and fixed-radius NN estimates for the (conditional) cumulative hazard function $\Hazard(\timeVar|\fixedFeatureVector)=-\log\survEnd(\timeVar|\fixedFeatureVector)$ are $\widehat{\Hazard}^{k\textsc{-NN}}(\timeVar|\fixedFeatureVector):=\widehat{\Hazard}^{\text{NA}}(\timeVar|\neighborsKNN(\fixedFeatureVector))$ and $\widehat{\Hazard}^{\textsc{NN}(h)}(\timeVar|\fixedFeatureVector):=\widehat{\Hazard}^{\text{NA}}(\timeVar|\neighborsNNh(\fixedFeatureVector))$.

Recalling that for kernel $K$ and bandwidth $h>0$, the kernel versions of the unique death times, number of deaths, and number of subjects at risk are denoted $\setY_{K}(\fixedFeatureVector;h)$, $d_{K}(\timeVar|\fixedFeatureVector;h)$, and $n_{K}(\timeVar|\fixedFeatureVector;h)$, then the Nelson-Aalen-based kernel estimate for $\Hazard(\timeVar|\fixedFeatureVector)$ is
\[
\widehat{\Hazard}^{K}(\timeVar|\fixedFeatureVector;h):=\sum_{t'\in\setY_{K}(x;h)}\frac{d_{K}(\timeVar'|\fixedFeatureVector;h)\ind\{\timeVar'\le\timeVar\}}{n_{K}(\timeVar'|\fixedFeatureVector;h)}.
\]
As already discussed in our analysis outline (Section~\ref{sec:analysis-overview}), the main change to our proofs to obtain nonasymptotic guarantees for these Nelson-Aalen-based estimators is quite simple: for any of the Kaplan-Meier-based estimators $\widehat{\survEnd}$ we consider, taking the first-order Taylor expansion of $\log\widehat{\survEnd}$ is exactly the negated version of the corresponding Nelson-Aalen-based estimator. This is the only high-level change. A few technical changes have to be made to arrive at a guarantee for each Nelson-Aalen-based estimator. We explain these changes only for the $k$-NN case.

We reuse notation from our analysis outline (Section~\ref{sec:analysis-overview}). When there are no ties in survival and censoring times, we have
\begin{align*}
-\widehat{\Hazard}^{k\textsc{-NN}}(\timeVar|\fixedFeatureVector) & =U_{1}(\timeVar|\fixedFeatureVector)+U_{2}(\timeVar|\fixedFeatureVector).
\end{align*}
Importantly, note that we no longer have to worry about the higher-order Taylor series terms $U_{3}(\timeVar|\fixedFeatureVector)$. Thus, rather than using inequality~\eqref{eq:kNN-6-term-bound}, we now have
\begin{align*}
 & |\widehat{\Hazard}^{k\textsc{-NN}}(\timeVar|\fixedFeatureVector)-\Hazard(\timeVar|\fixedFeatureVector)|\\
 & \quad=|U_{1}(\timeVar|\fixedFeatureVector)-\log\survEnd(\timeVar|\fixedFeatureVector)+U_{2}(\timeVar|\fixedFeatureVector)|\\
 & \quad\le|U_{1}(\timeVar|\fixedFeatureVector)-\mathbb{E}[U_{1}(\timeVar|\fixedFeatureVector)|\widetilde{\featureVar}]|\nonumber\\
 & \quad\quad+|\mathbb{E}[U_{1}(\timeVar|\fixedFeatureVector)|\widetilde{\featureVar}]-\log\survEnd(\timeVar|\fixedFeatureVector)|\nonumber\\
 & \quad\quad+\frac{2}{k\theta^{2}}+\frac{2}{\theta^{2}}\sup_{s\in[0,\timeHorizon]}|\obsEnd(s|\fixedFeatureVector)-\mathbb{E}[\estObsEnd^{k\textsc{-NN}}(s|\fixedFeatureVector)|\widetilde{\featureVar}]|\nonumber\\
 & \quad\quad+\frac{2}{\theta^{2}}\sup_{s\ge0}|\estObsEnd^{k\textsc{-NN}}(s|\fixedFeatureVector)-\mathbb{E}[\estObsEnd^{k\textsc{-NN}}(s|\fixedFeatureVector)|\widetilde{\featureVar}]|.\nonumber
\end{align*}
Thus, we have five RHS terms. As before, we want the RHS to be at most $\varepsilon/3$. For simplicity, we use our earlier bounds, which controls each of the RHS terms to be at most $\varepsilon/18$ so that the RHS above is at most $5\varepsilon/18<\varepsilon/3$.

At this point, another change is needed. Previously we showed that $|\log\widehat{\survEnd}^{k\textsc{-NN}}(\timeVar|\fixedFeatureVector)-\log\survEnd(\timeVar|\fixedFeatureVector)|\le\varepsilon/3$ implies $|\widehat{\survEnd}^{k\textsc{-NN}}(\timeVar|\fixedFeatureVector)-\survEnd(\timeVar|\fixedFeatureVector)|\le\varepsilon/3$. We then used the fact that $\survEnd(\cdot|\fixedFeatureVector)$ changes by at most a value of 1 over the interval $[0,\timeHorizon]$. Now we do not remove the logs and instead observe that $\Hazard(\cdot|\fixedFeatureVector)$ changes by at most a value of $-\log\survEnd(\timeHorizon|\fixedFeatureVector)\le-\log\theta=\log\frac{1}{\theta}$ over the interval $[0,\timeHorizon]$. Thus, when we partition the interval $[0,\timeHorizon]$ into $L(\varepsilon)$ pieces such that $0=\eta_{0}<\eta_{1}<\cdots<\eta_{L(\varepsilon)}=\timeHorizon$, as before, we ask that $|\widehat{\Hazard}^{k\textsc{-NN}}(\timeVar|\fixedFeatureVector)-\Hazard(\timeVar|\fixedFeatureVector)|\le\varepsilon/3$ for $j=1,\dots,L(\varepsilon)$. However, the bound on $L(\varepsilon)$ changes. By placing the points $\eta_{j}$'s at times when $\Hazard(\timeVar|\fixedFeatureVector)$ changes by exactly $\varepsilon/3$ (except possibly across $[\eta_{L(\varepsilon)-1},\eta_{L(\varepsilon)}]$, where $\Hazard(\timeVar|\fixedFeatureVector)$ can change by less), then
$
L(\varepsilon)=\lceil\frac{\log\frac{1}{\theta}}{\varepsilon/3}\rceil=\lceil\frac{3}{\varepsilon}\log\frac{1}{\theta}\rceil\le\frac{3}{\varepsilon}\log\frac{1}{\theta}+1.
$
The rest of the proof is the same.

We now state the resulting pointwise guarantees for the Nelson-Aalen-based $k$-NN, fixed-radius NN, and kernel estimators.
\begin{thm}[Nelson-Aalen-based $k$-NN pointwise bound]
Under Assumptions A1--A4, let $\varepsilon\in(0,1)$ be a user-specified error tolerance and define critical distance $\ensuremath{{h^{*}:=(\frac{\varepsilon\theta}{18\Lips})^{1/\holderIndex}}}$. For any feature vector $\ensuremath{\fixedFeatureVector\in\text{supp}(\featureDist)}$ and any choice of number of nearest neighbors $\ensuremath{k\in[\frac{72}{\varepsilon\theta^{2}},\!\frac{n\featureDist(\mathcal{B}_{\fixedFeatureVector,h^{*}})}{2}]}$, we have, over randomness in training data,
\begin{align*}
 & \mathbb{P}\Big(\sup_{\timeVar\in[0,\timeHorizon]}|\widehat{\Hazard}^{k\textsc{-NN}}(\timeVar|\fixedFeatureVector)-\Hazard(\timeVar|\fixedFeatureVector)|>\varepsilon\Big)\nonumber\\
 & \le\exp\!\Big(-\frac{k\theta}{8}\Big)\!+\exp\!\Big(-\frac{n\featureDist(\mathcal{B}_{\fixedFeatureVector,h^{*}})}{8}\Big)\nonumber\\
 & \;\;+2\exp\!\Big(-\frac{k\varepsilon^{2}\theta^{4}}{648}\Big)\!+2\Big(\frac{3}{\varepsilon}\log\frac{1}{\theta}+1\Big)\!\exp\Big(-\frac{k\varepsilon^{2}\theta^{2}}{162}\Big).
\end{align*}
\end{thm}

\begin{thm}[Nelson-Aalen-based fixed-radius NN pointwise bound]
Under Assumptions A1--A4, let $\varepsilon\in(0,1)$ be a user-specified error tolerance. Suppose that the threshold distance satisfies $h\in(0,h^{*}]$ with $\ensuremath{{h^{*}:=(\frac{\varepsilon\theta}{18\Lips})^{1/\holderIndex}}}$, and the number of training data satisfies $n\ge\frac{144}{\varepsilon\theta^{2}\featureDist(\mathcal{B}_{\fixedFeatureVector,h})}$. For any $\ensuremath{\fixedFeatureVector\in\text{supp}(\featureDist)}$,
\begin{align*}
 & \mathbb{P}\Big(\sup_{\timeVar\in[0,\timeHorizon]}|\widehat{\Hazard}^{\textsc{NN}(h)}(\timeVar|\fixedFeatureVector)-\Hazard(\timeVar|\fixedFeatureVector)|>\varepsilon\Big)\nonumber\\
 & \quad\le\exp\!\Big(-\frac{n\featureDist(\mathcal{B}_{\fixedFeatureVector,h})\theta}{16}\Big)+\exp\!\Big(-\frac{n\featureDist(\mathcal{B}_{\fixedFeatureVector,h})}{8}\Big)\nonumber\\
 & \quad\quad+2\exp\!\Big(-\frac{n\featureDist(\mathcal{B}_{\fixedFeatureVector,h})\varepsilon^{2}\theta^{4}}{1296}\Big)\\
 & \quad\quad+2\Big(\frac{3}{\varepsilon}\log\frac{1}{\theta}+1\Big)\exp\!\Big(-\frac{n\featureDist(\mathcal{B}_{\fixedFeatureVector,h})\varepsilon^{2}\theta^{2}}{324}\Big).
\end{align*}
\end{thm}

\begin{thm}[Nelson-Aalen-based kernel pointwise bound]
Under Assumptions A1--A5, let $\varepsilon\in(0,1)$ be a user-specified error tolerance. Suppose that the threshold distance satisfies $h\in(0,\frac{1}{\stanDistThresh}(\frac{\varepsilon\theta}{18\Lips_{K}})^{1/\holderIndex}]$, and the number of training data satisfies $n\ge\frac{144}{\varepsilon\theta^{2}\featureDist(\mathcal{B}_{\fixedFeatureVector,\stanDistThresh h})\kappa}$. For any $\ensuremath{\fixedFeatureVector\in\text{supp}(\featureDist)}$,
\begin{align*}
 & \mathbb{P}\Big(\sup_{\timeVar\in[0,\timeHorizon]}|\widehat{\Hazard}^{K}(\timeVar|\fixedFeatureVector;h)-\Hazard(\timeVar|\fixedFeatureVector)|>\varepsilon\Big)\nonumber\\
 & \quad\le\exp\!\Big(-\frac{n\featureDist(\mathcal{B}_{x,\stanDistThresh h})\theta}{16}\Big)+\exp\!\Big(-\frac{n\featureDist(\mathcal{B}_{\fixedFeatureVector,\stanDistThresh h})}{8}\Big)\nonumber\\
 & \quad\quad+\frac{216}{\varepsilon\theta^{2}\kappa}\exp\!\Big(-\frac{n\featureDist(\mathcal{B}_{x,\stanDistThresh h})\varepsilon^{2}\theta^{4}\kappa^{4}}{11664}\Big)\\
 & \quad\quad+2\Big(\frac{3}{\varepsilon}\log\frac{1}{\theta}+1\Big)\exp\!\Big(-\frac{n\featureDist(\mathcal{B}_{x,\stanDistThresh h})\varepsilon^{2}\theta^{2}\kappa^{2}}{324}\Big).
\end{align*}
\end{thm}

We remark that the slight change in the proof (regarding partitioning $[0,\timeHorizon]$ as to handle log space) can actually be applied to any of the nearest neighbor and kernel Kaplan-Meier-based estimators $\widehat{\survEnd}$ to guarantee that $\sup_{\timeVar\in[0,\timeHorizon]}|\log\widehat{\survEnd}(\timeVar|\fixedFeatureVector)-\log\survEnd(\timeVar|\fixedFeatureVector)|\le\varepsilon$.

\section{Details on Experimental Results}
\label{sec:experiment-details}

\textbf{Concordance index calculation.}
Harrell's concordance index (c-index) \citep{harrell_1982} is a pairwise-ranking-based accuracy metric for survival analysis. Roughly, it measures the fraction of pairs of subjects that are correctly ordered among pairs that can actually be ordered (not every pair can be ordered due to censoring). As such, the highest c-index is 1, and 0.5 corresponds to a random ordering. Because c-index is ranking based, it requires that a survival estimator provide some way to rank pairs of subjects in terms of who is at greater risk (ties are allowed).

C-index is computed as follows. Suppose that there are $n'$ test subjects with data $(X_1',Y_1',\delta_1'),\dots,(X_{n'}',Y_{n'}',\delta_{n'}')\in\mathcal{X}\times\mathbb{R}_+\times\{0,1\}$. Then:

\begin{enumerate}
\item Construct the set of all pairs of test subjects: \[\mathcal{P}:=\{(i,j) : i,j\in[n'] \text{ such that } i<j\}.\]
\item Remove any pair $(i,j)$ from $\mathcal{P}$ for which the earlier observed time among test subjects $i$ and $j$ is censored.
\item Remove any pair $(i,j)$ from $\mathcal{P}$ for which the observed times are tied unless at least one of test subjects $i$ and~$j$ has an event indicator value of 1.
\item For each pair $(i,j)$ that remains in $\mathcal{P}$, we compute a score $C_{(i,j)}$ for $(i,j)$ as follows:
\begin{itemize}
\item If $Y_i' \ne Y_j'$: set $C_{(i,j)}:=1$ if the subject with the shorter observed time (which is guaranteed to be a survival time due to step 2) is predicted to be at higher risk among subjects $i$ and $j$; set $C_{(i,j)}:=1/2$ if the predicted risks are tied between subjects $i$ and $j$; otherwise, set $C_{(i,j)}:=0$.
\item If $Y_i' = Y_j'$ and $\delta_i'=\delta_j'=1$: set $C_{(i,j)}:=1$ if the predicted risks are tied between $i$ and $j$; otherwise, set $C_{(i,j)}:=1/2$.
\item If $Y_i' = Y_j'$ and exactly one of $\delta_i'$ or $\delta_j'$ is 1: set $C_{(i,j)}=1$ if the predicted risk is higher for the subject with event indicator set to 1; otherwise set $C_{(i,j)}=1/2$.
\end{itemize}
\item Finally, the c-index is given by:
\[
\frac1{|\mathcal{P}|} \sum_{(i,j)\in\mathcal{P}} C_{(i,j)}.
\]
\end{enumerate}
As for how we rank any pair of test subjects in our experimental results, we use the same approach as \citet{ishwaran_2008}. Let $Y_1^*,\dots,Y_m^*$ denote the unique observed times among the test subjects. Then test subject $i$ is considered to be at higher risk than test subject $j$ if
\[
\sum_{\ell=1}^{m} \widehat{H}(Y_\ell^*|X_i')
> \sum_{\ell=1}^{m} \widehat{H}(Y_\ell^*|X_j'),
\]
where $\widehat{H}$ is an estimate of the conditional cumulative hazard function $H(t|x)=-\log S(t|x)$ (we can, for instance, use nearest neighbor and kernel variants of the Nelson-Aalen estimator). (As a remark, other ways of ranking test subjects are possible. For instance, for the $i$-th test subject, we could estimate the subject's median survival time by finding time $t\ge0$ such that $\widehat{S}(t|X_i')\approx1/2$ for some estimate $\widehat{S}$ of conditional survival function $S$, and then rank the test subjects by predicted median survival times, i.e., shorter predicted median survival time means higher risk.)

\textbf{Parameter selection grids.}
For the $k$-NN estimator, we search for $k$ over integer powers of 2, starting at 4 and up to the size of the training dataset. For the kernel estimator, we first compute the largest pairwise distance $h_{\max}$ seen in the training data. Then we search for kernel bandwidth $h$ from 0.01$h_{\max}$ to $h_{\max}$ on an evenly spaced logarithmic scale with 20 grid points. For random survival forests and the adaptive kernel variant, we search over the number of trees (50, 100, 150, 200) and over the max depth (3, 4, 5, 6, 7, 8, and lastly no restriction on max depth).

\textbf{Extended results.}
We now present extended experimental results that also include Epanechnikov and truncated Gaussian kernels for the $k$-NN, \textsc{cdf-reg}, and kernel estimators. The truncated Gaussian kernel is of the form $K(s)=\exp(-\frac{s^2}{2\sigma^2})\ind\{s\le1\}$ for standard deviation/scale parameter $\sigma>0$. We have results for $\sigma\in\{1,2,3\}$. The concordance indices are reported for the \textsc{pbc}, \textsc{gbsg2}, \textsc{recid}, and \textsc{kidney} datasets in Figures~\ref{fig:pbc}, \ref{fig:gbsg2}, \ref{fig:recid}, and \ref{fig:kidney}.

We also report our integrated Brier score variant given in equation~\eqref{eq:IBS-variant} (with $\theta_{\text{LB}}=10^{-6}$) in Figures~\ref{fig:pbc-ibs}, \ref{fig:gbsg2-ibs}, \ref{fig:recid-ibs}, and \ref{fig:kidney-ibs}. Note that this integrated Brier score requires a user-specified time horizon $\timeHorizon$. For a given dataset, we set the time horizon to be the 75th percentile of the observed times in the training data (when using other percentiles that are at least the 50th percentile, although the integrated Brier scores can be different, the relative performance between the methods remains about the same). In terms of our integrated Brier score variant, which algorithms achieve the best performance changes from what we get using the concordance index. Consistently, random survival forests achieves lower (i.e., better) integrated Brier scores than the adaptive kernel method and tends to have the lowest scores for the \textsc{gbsg2}, \textsc{recid}, and \textsc{kidney} datasets. For these three larger datasets, the adaptive kernel method tends to achieve integrated Brier scores that are second best. Similar to the case of concordance indices, for the smallest dataset \textsc{pbc}, weighted versions of $k$-NN using $\ell_2$ distance have the best performance.

\begin{figure}[t]
    \vspace{-0.5em}
    \centering
    \includegraphics[width=\linewidth]{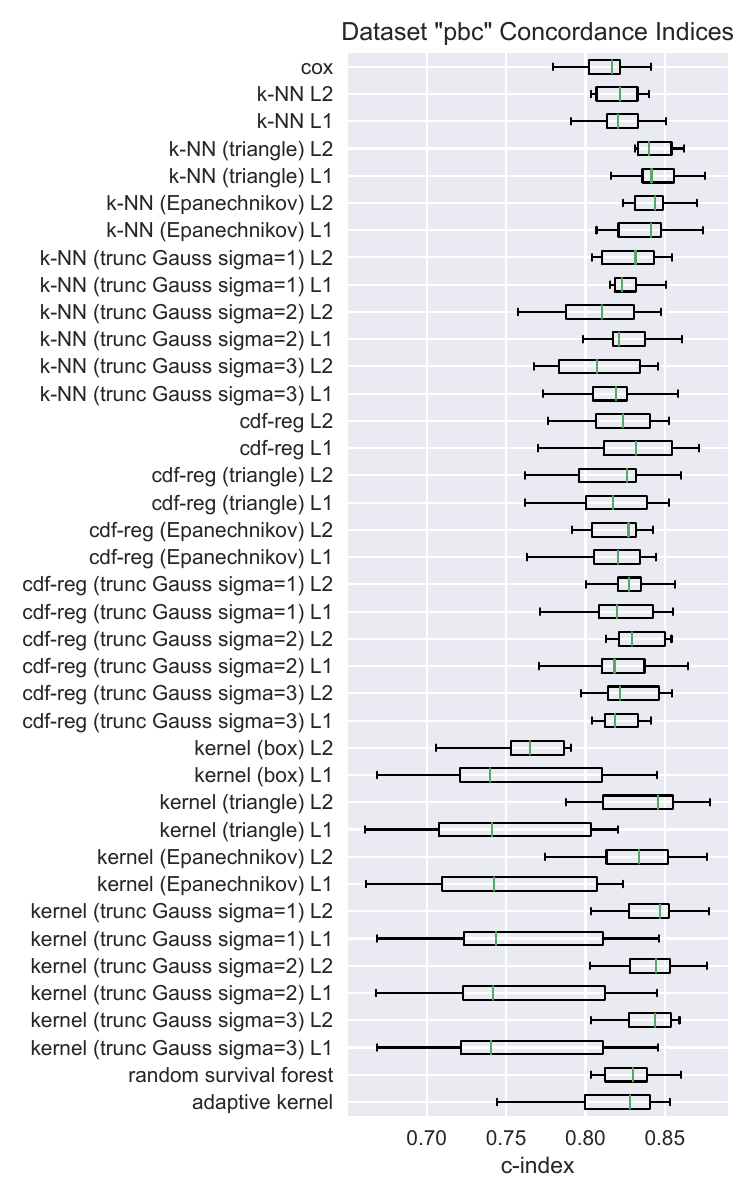}
    \vspace{-2.5em}
    \caption{Extended concordance index results for the \textsc{pbc} dataset (higher is better).}
    \label{fig:pbc}
    \vspace{-1em}
\end{figure}

\begin{figure}[t]
    \vspace{-0.5em}
    \centering
    \includegraphics[width=\linewidth]{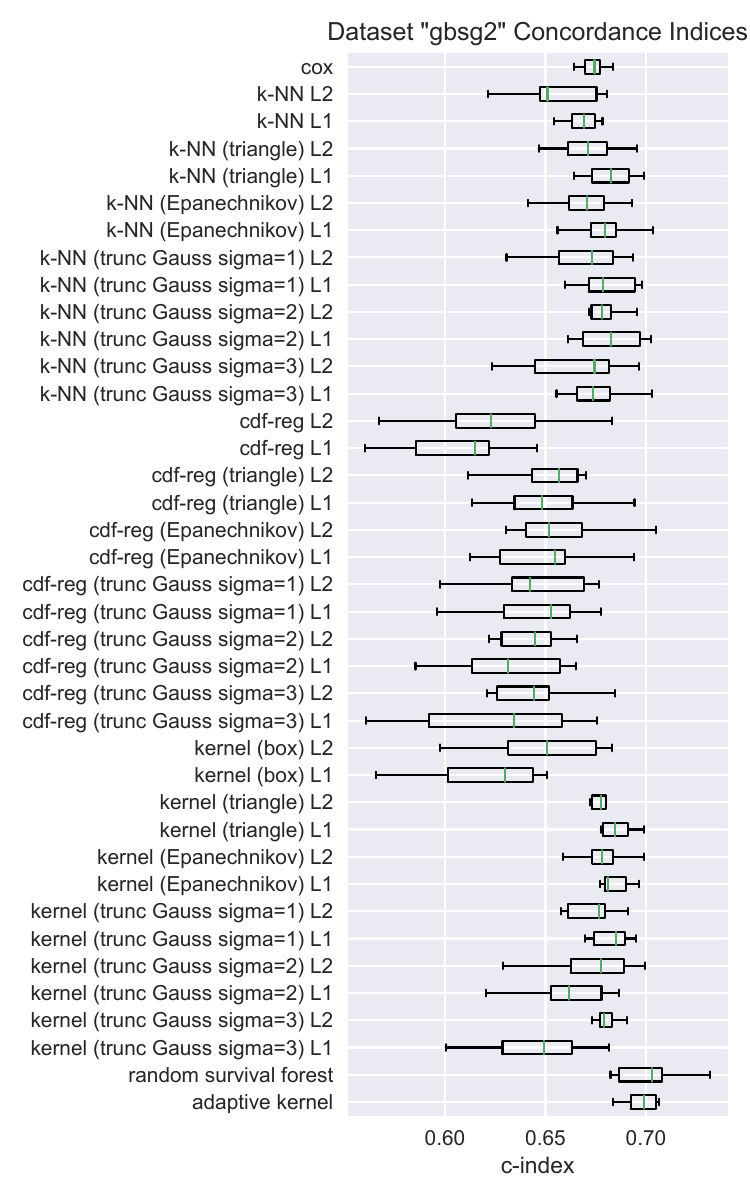}
    \vspace{-2.5em}
    \caption{Extended concordance index results for the \textsc{gbsg2} dataset (higher is better).}
    \label{fig:gbsg2}
    \vspace{-1em}
\end{figure}

\begin{figure}
    \centering
    \includegraphics[width=\linewidth]{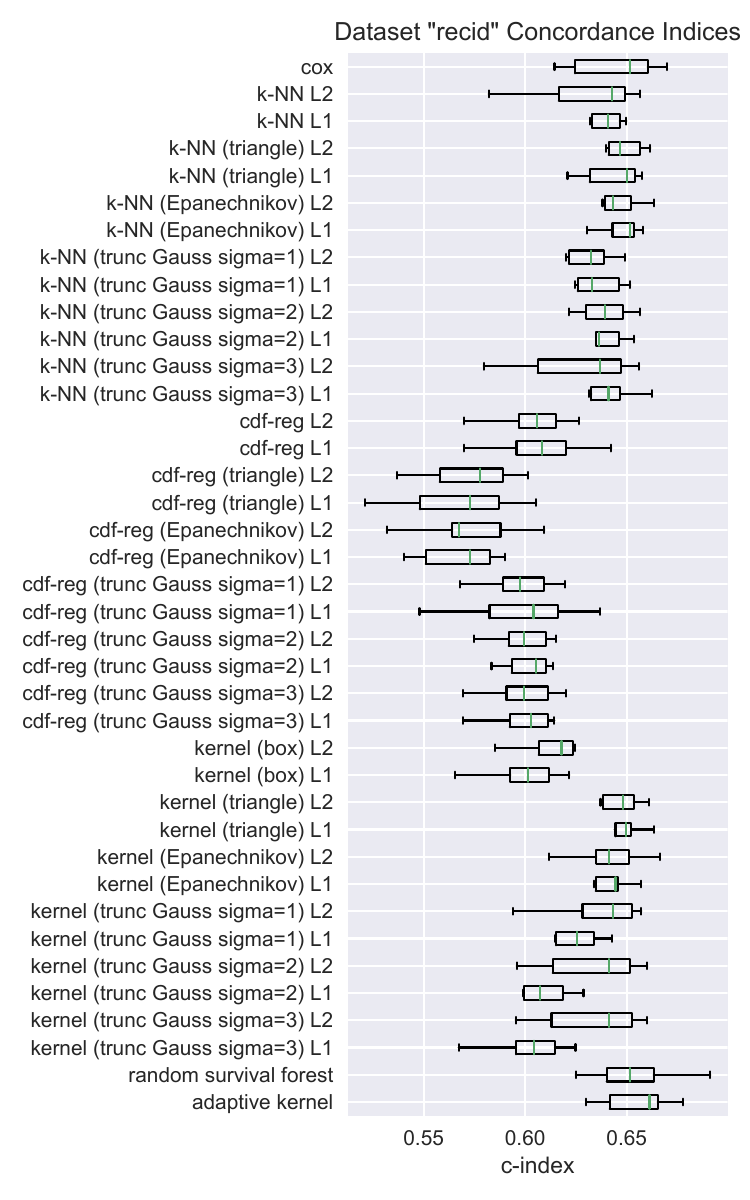}
    \vspace{-2.5em}
    \caption{Extended concordance index results for the \textsc{recid} dataset (higher is better).}
    \label{fig:recid}
\end{figure}

\begin{figure}
    \centering
    \includegraphics[width=\linewidth]{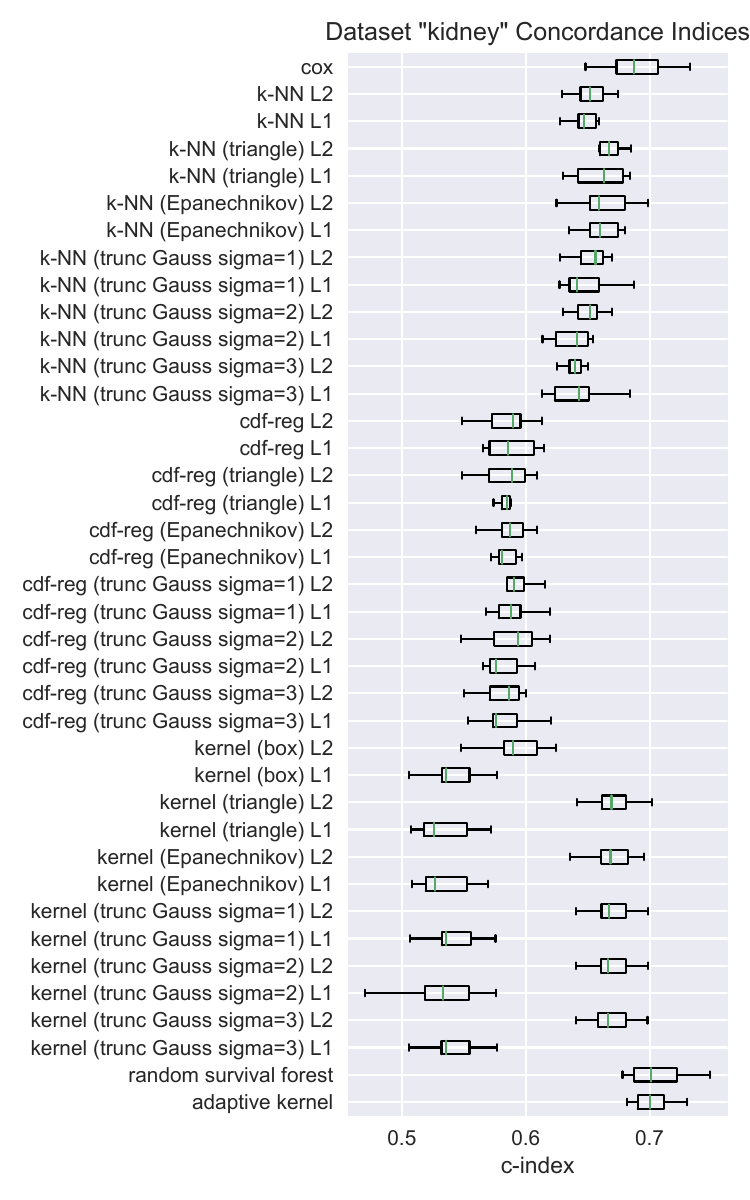}
    \vspace{-2.5em}
    \caption{Extended concordance index results for the \textsc{kidney} dataset (higher is better).}
    \label{fig:kidney}
\end{figure}

\begin{figure}[p]
    \centering
    \includegraphics[width=\linewidth]{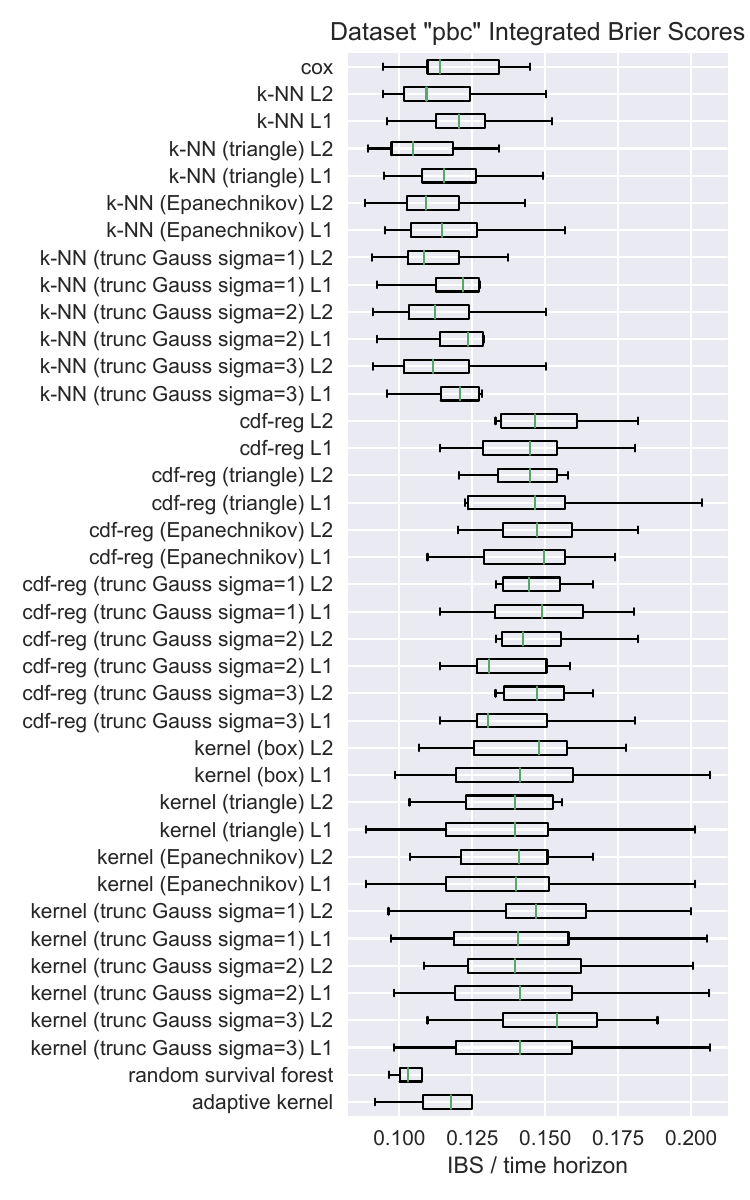}
    \vspace{-2.5em}
    \caption{Integrated Brier scores (divided by the time horizon) for the \textsc{pbc} dataset (lower is better).}
    \label{fig:pbc-ibs}
\end{figure}

\begin{figure}[p]
    \centering
    \includegraphics[width=\linewidth]{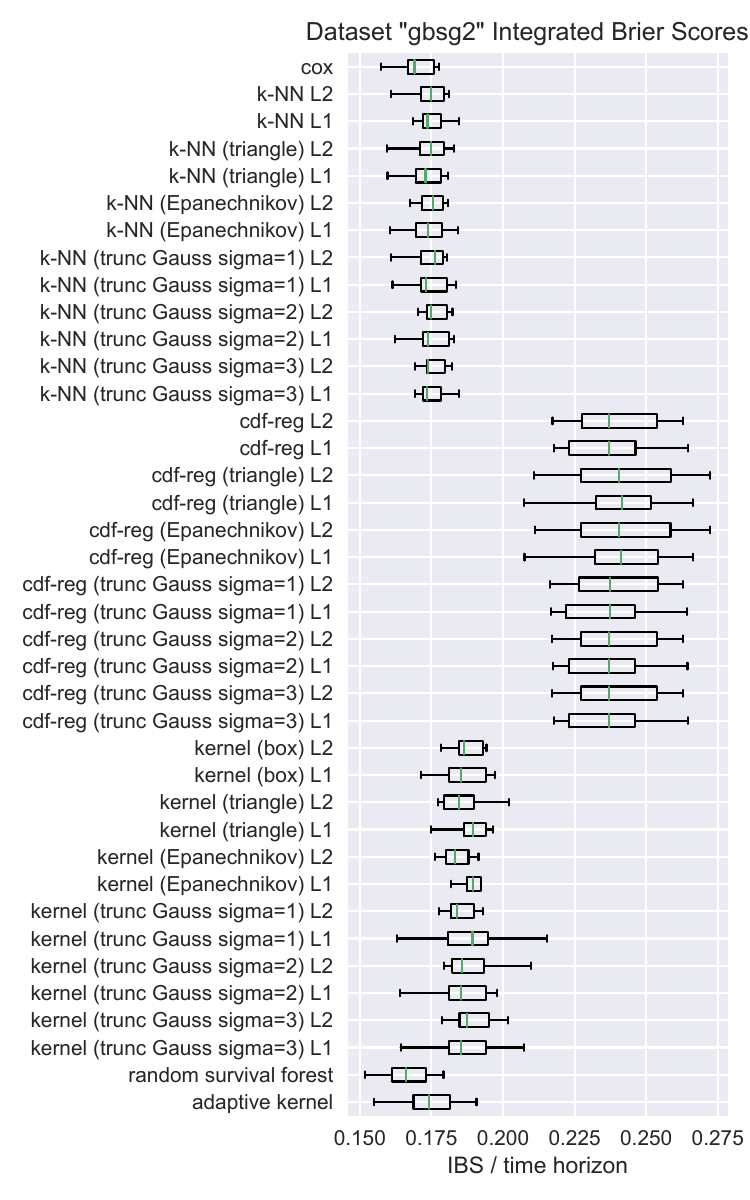}
    \vspace{-2.5em}
    \caption{Integrated Brier scores (divided by the time horizon) for the \textsc{gbsg2} dataset (lower is better).}
    \label{fig:gbsg2-ibs}
\end{figure}

\begin{figure}[p]
    \centering
    \includegraphics[width=\linewidth]{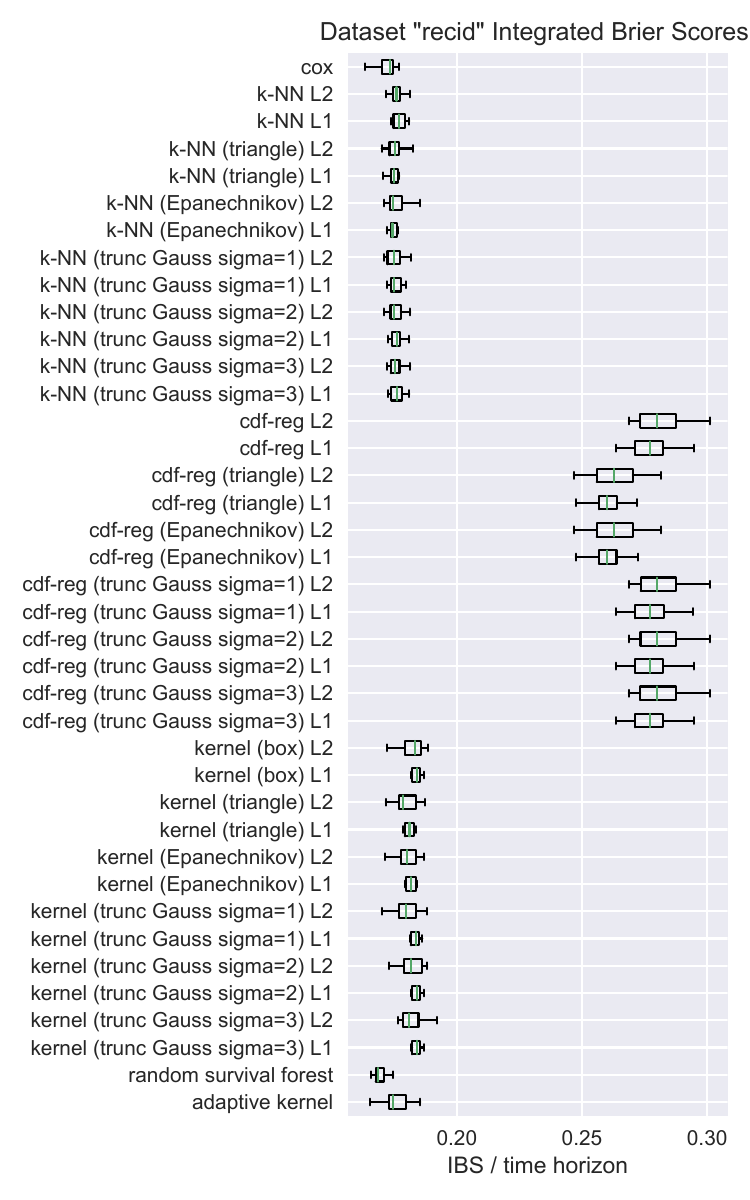}
    \vspace{-2.5em}
    \caption{Integrated Brier scores (divided by the time horizon) for the \textsc{recid} dataset (lower is better).}
    \label{fig:recid-ibs}
\end{figure}

\begin{figure}[p]
    \centering
    \includegraphics[width=\linewidth]{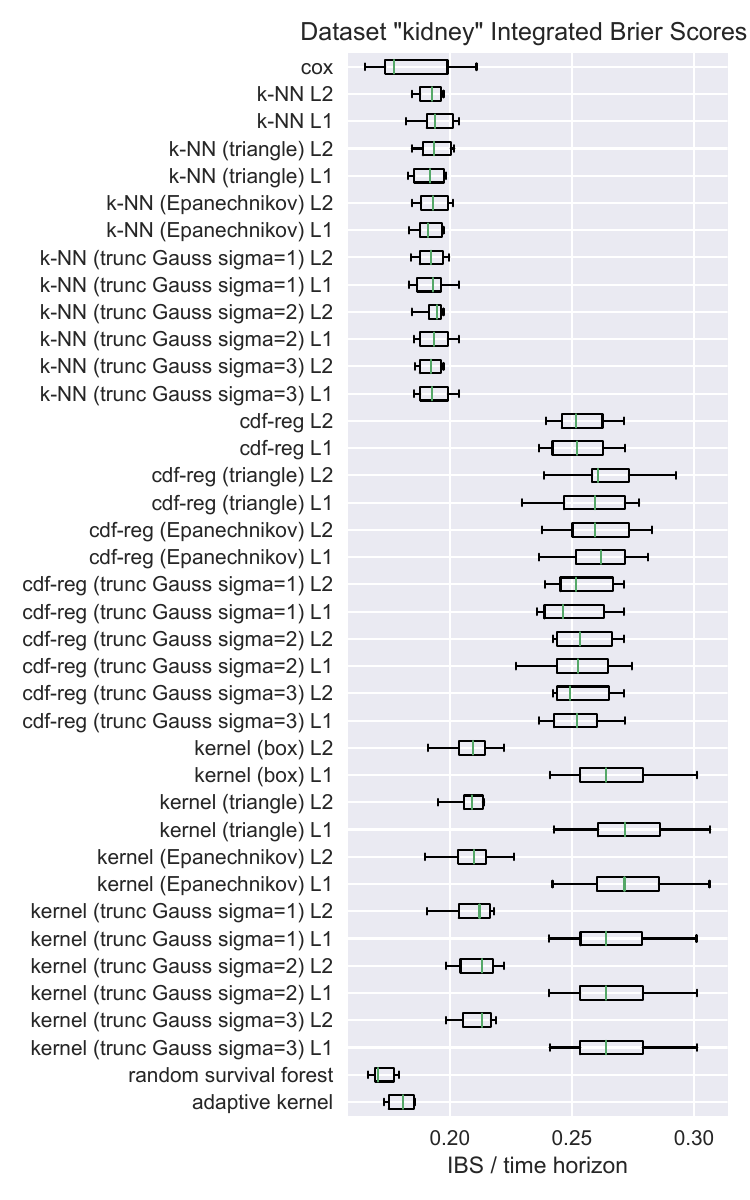}
    \vspace{-2.5em}
    \caption{Integrated Brier scores (divided by the time horizon) for the \textsc{kidney} dataset (lower is better).}
    \label{fig:kidney-ibs}
\end{figure}

\end{document}